\documentclass[11pt]{article}
\usepackage{amssymb,amsmath,amsthm,amsfonts}
\usepackage{jinshuomacros}
\usepackage{algorithm}
\usepackage[noend]{algpseudocode}
\usepackage{thmtools}
\usepackage{mathrsfs}
\usepackage[usenames,dvipsnames]{color}
\usepackage{tikz}
\usepackage{hyperref}

\usepackage{graphicx}
\usepackage{bibentry}
\usepackage{cleveref}
\usepackage{fullpage}

\newif\ifhideproofs

\ifhideproofs
\usepackage{environ}
\NewEnviron{hide}{}

\fi
\renewcommand{\N}{\mathcal{N}}
\renewcommand{\P}{\mathbb{P}}
\newcommand{\e}{\mathrm{e}}
\newcommand{\f}{\boldsymbol{f}}
\newcommand{\DCT}{{dominated convergence theorem}}
\newcommand{\Proc}{\mathrm{Proc}}
\newcommand{\Renyi}{R\'enyi}

\newcommand{\Sf}{f^{\mathrm{S}}}
		\newcommand{\TV}{\mathrm{TV}}

\nobibliography* 

\declaretheoremstyle[bodyfont=\normalfont]{normalbody}
\newtheorem{theorem}{Theorem}[section]
\newtheorem{lemma}[theorem]{Lemma}

\newtheorem{corollary}[theorem]{Corollary}
\newtheorem{proposition}[theorem]{Proposition}

\newtheorem{definition}[theorem]{Definition}
\theoremstyle{remark}
\declaretheorem{remark}

\begin{document}
\pagenumbering{roman}
\title{Gaussian Differential Privacy}


\author{Jinshuo Dong\thanks{Graduate Group in Applied Mathematics and Computational Science, University of Pennsylvania. Email: \texttt{jinshuo@sas.upenn.edu}.} \and Aaron Roth\thanks{Department of Computer and Information Sciences, University of Pennsylvania. Email: \texttt{aaroth@cis.upenn.edu}.} \and Weijie J.~Su\thanks{Department of Statistics, the Wharton School, University of Pennsylvania. Email: \texttt{suw@wharton.upenn.edu}.}}



\date{May 24, 2019}


\maketitle

\vspace{0.3cm}
\begin{abstract}
In the past decade, differential privacy has seen remarkable success as a rigorous and practical formalization of data privacy. This privacy definition and its divergence based relaxations, however, have several acknowledged weaknesses, either in handling composition of private algorithms or in analyzing important primitives like privacy amplification by subsampling. Inspired by the hypothesis testing formulation of privacy, this paper proposes a new relaxation of differential privacy, which we term ``$f$-differential privacy'' ($f$-DP). This notion of privacy has a number of appealing properties and, in particular, avoids difficulties associated with divergence based relaxations. First, $f$-DP faithfully preserves the hypothesis testing interpretation of differential privacy, thereby making the privacy
guarantees easily interpretable. In addition, $f$-DP allows for lossless reasoning about composition in an algebraic fashion. Moreover, we provide a powerful technique to import existing results proven for the original differential
privacy definition to $f$-DP and, as an application of this technique, obtain a simple and easy-to-interpret theorem of privacy amplification by subsampling for $f$-DP.

In addition to the above findings, we introduce a canonical single-parameter family of privacy notions within the $f$-DP class that is referred to as ``Gaussian differential privacy'' (GDP), defined based on hypothesis testing of two shifted Gaussian distributions. GDP is the focal privacy definition among the family of $f$-DP guarantees due to a central limit theorem for differential privacy that we prove. More precisely, the privacy guarantees of \emph{any} hypothesis testing based definition of privacy (including the original differential privacy definition) converges to GDP in the limit under composition. We also prove a Berry--Esseen style version of the central limit theorem, which gives a computationally inexpensive tool for tractably analyzing the exact composition of private algorithms.

Taken together, this collection of attractive properties render $f$-DP a mathematically coherent, analytically tractable, and versatile framework for private data analysis. Finally, we demonstrate the use of the tools we develop by giving an improved analysis of the privacy guarantees of noisy stochastic gradient descent.
\end{abstract}

\newpage
\tableofcontents

\newpage
\pagenumbering{arabic}
\setcounter{page}{1}
\section{Introduction}


Modern statistical analysis and machine learning are overwhelmingly applied to data concerning \emph{people}. Valuable datasets generated from personal devices and online behavior of billions of individuals contain data on location, web search histories, media consumption, physical activity, social networks, and more. This is on top of continuing large-scale analysis of traditionally sensitive data records, including those collected by hospitals, schools, and the Census. This reality requires the development of tools to perform large-scale data analysis in a way that still protects the \emph{privacy} of individuals represented in the data.

Unfortunately, the history of data privacy for many years consisted of ad-hoc attempts at ``anonymizing'' personal information, followed by high profile de-anonymizations. This includes the release of AOL search logs, de-anonymized by the \textit{New York Times} \cite{aol}, the Netflix Challenge dataset, de-anonymized by Narayanan and Shmatikov \cite{netflix}, the realization that participants in genome-wide association studies could be identified from aggregate statistics such as minor allele frequencies that were publicly released \cite{gwas}, and the reconstruction of individual-level census records from aggregate statistical releases \cite{census}.

Thus, we urgently needed a rigorous and principled privacy-preserving framework to prevent breaches of personal information in data analysis. In this context, \textit{differential privacy} has put private data analysis on firm theoretical foundations \cite{DMNS06,approxdp}. This definition has become tremendously successful: in addition to an enormous and growing academic literature, it has been adopted as a key privacy technology by Google \cite{rappor}, Apple \cite{apple}, Microsoft \cite{microsoft}, and the US Census Bureau \cite{census}. The definition of this new concept involves privacy parameters $\epsilon \ge 0$ and $0 \le \delta \le 1$.
\begin{definition}[\cite{DMNS06,approxdp}]\label{def:dpintro}
A randomized algorithm $M$ that takes as input a dataset consisting of individuals is $(\ep, \delta)$-differentially private (DP) if for any pair of datasets $S, S'$ that differ in the record of a single individual, and any event $E$,
\begin{equation}\label{eq:dp_ine}
\P\left[ M(S)\in E \right] \leqslant \e^\ep \P \left[ M(S')\in E \right] + \delta.
\end{equation}
When $\delta = 0$, the guarantee is simply called $\epsilon$-DP.
\end{definition}

In this definition, datasets are \textit{fixed} and the probabilities are taken \textit{only} over the randomness of the mechanism\footnote{A randomized algorithm $M$ is often referred to as a mechanism in the differential privacy literature.}. In particular, the event $E$ can take any measurable set in the range of $M$. To achieve differential privacy, a mechanism is necessarily randomized. Take as an example the problem of privately releasing the average cholesterol level of individuals in the dataset $S = (x_1, \ldots, x_n)$, each $x_i$ corresponding to an individual. A privacy-preserving mechanism may take the form\footnote{Here we identify the individual $x_i$ with his/her cholesterol level.}
\[
M(S) = \frac1n (x_1 + \cdots + x_n) + \text{noise}.
\]
The level of the noise term has to be sufficiently large to mask the \textit{characteristics} of any individual's cholesterol level,
while not being too large to distort the population average for accuracy purposes. Consequently, the probability distributions of $M(S)$ and $M(S')$ are close to each other for any datasets $S, S'$ that differ in only one individual record.



Differential privacy is most naturally defined through a hypothesis testing problem from the perspective of an attacker who aims to distinguish $S$ from $S'$ based on the output of the mechanism. This statistical viewpoint was first observed by \cite{wasserman_zhou} and then further developed by \cite{KOV}, which is a direct inspiration for our work. In short, consider the hypothesis testing problem
\begin{equation}\label{eq:hyp_into}
H_0: \text{the underlying dataset is } S \quad\text{ versus }\quad H_1: \text{the underlying dataset is } S'
\end{equation}
and call Alice the only individual that is in $S$ but not $S'$. As such, rejecting the null hypothesis corresponds to the detection of absence of Alice, whereas accepting the null hypothesis means to detect the presence of Alice in the dataset. Using the output of an $(\ep, \delta)$-DP mechanism, the power\footnote{The power is equal to 1 minus the type II error.} of any test at significance level $0 < \alpha < 1$ has an upper bound\footnote{A more precise bound is given in \Cref{thm:privacy_testing}.} of $\e^\ep\alpha+\delta$. This bound is only slightly larger than $\alpha$ provided that $\epsilon,\delta$ are small and, therefore, \textit{any} test is essentially powerless. Put differently, differential privacy with small privacy parameters protects against any inferences of the presence of Alice, or any other individual, in the dataset.

Despite its apparent success, there are good reasons to want to relax the original definition of differential privacy, which has led to a long line of proposals for such relaxations. The most important shortcoming is that $(\ep,\delta)$-DP does not tightly handle composition. Composition concerns how privacy guarantees degrade under repetition of mechanisms applied to the same dataset, rendering the design of differentially private algorithms \textit{modular}. Without compositional properties, it would be near impossible to develop complex differentially private data analysis methods. Although it has been known since the original papers defining differential privacy \cite{DMNS06,approxdp} that the composition of an $(\epsilon_1,\delta_1)$-DP mechanism and an $(\epsilon_2,\delta_2)$-DP mechanism yields an $(\epsilon_1+\epsilon_2,\delta_1+\delta_2)$-DP mechanism, the corresponding upper bound $\e^{\epsilon_1 + \epsilon_2}\alpha +  \delta_1 + \delta_2$ on the power of any test at
significance level $\alpha$ no longer tightly characterizes the trade-off between significance level and power for the testing between $S$ and $S'$. In \cite{boosting}, Dwork, Rothblum, and Vadhan gave an improved composition theorem, but it fails to capture the correct hypothesis testing trade-off. This is for a fundamental reason: $(\epsilon,\delta)$-DP is mis-parameterized in the sense that the guarantees of the composition of $(\epsilon_i,\delta_i)$-DP mechanisms cannot be characterized by any pair of parameters $(\epsilon,\delta)$. Worse, given any $\delta$, finding the parameter $\ep$ that most tightly approximates the correct trade-off between significance level and type II error for a composition of a
sequence of differentially private algorithms is  computationally hard \cite{complexity}, and so in practice, one must resort to approximations. Given that composition and modularity are first-order desiderata for a useful privacy definition, these are substantial drawbacks and often continue to push practical algorithms with meaningful privacy guarantees out of reach.

In light of this, substantial recent effort has been devoted to developing relaxations of differential privacy for which composition can be handled exactly. This line of work includes several variants of ``concentrated differential privacy'' \cite{concentrated,concentrated2}, ``R\'enyi differential privacy'' \cite{renyi}, and ``truncated concentrated differential privacy'' \cite{tcdp}. These definitions are tailored to be able to exactly and easily track the ``privacy cost'' of compositions of the most basic primitive in differential privacy, which is the perturbation of a real valued statistic with Gaussian noise. 

While this direction of privacy relaxation has been quite fruitful, there are still several places one might wish for improvement. First, these notions of differential privacy no longer have hypothesis testing interpretations, but are rather based on studying divergences that satisfy a certain information processing inequality. There are good reasons to prefer definitions based on hypothesis testing. Most immediately, hypothesis testing based definitions provide an easy way to interpret the guarantees of a privacy definition. More fundamentally, a theorem due to Blackwell (see \Cref{thm:blackwell}) provides a formal sense in which a tight understanding of the trade-off between type I and type II errors for the hypothesis testing problem of distinguishing between $M(S)$ and $M(S')$ contains only more information than any divergence between the distributions $M(S)$ and $M(S')$ (so long as the divergence satisfies the information processing inequality).

Second, certain simple and fundamental primitives associated with differential privacy---most notably, \emph{privacy amplification by subsampling} \cite{KLNRS}---either fail to apply to the existing relaxations of differential privacy, or require a substantially complex analysis \cite{wang2018subsampled}. This is especially problematic when analyzing privacy guarantees of stochastic gradient descent---arguably the most popular present-day optimization algorithm---as subsampling is inherent to this algorithm. At best, this difficulty arising from using these relaxations could be overcome by using complex technical machinery. For example, it necessitated Abadi et al.~\cite{deep} to develop the numerical \emph{moments accountant} method to sidestep the issue.

\subsection{Our Contributions}
\label{sec:our-contributions-1}
In this work, we introduce a new relaxation of differential privacy that avoids these issues and has other attractive properties. Rather than giving a ``divergence'' based relaxation of differential privacy, we start fresh from the hypothesis testing interpretation of differential privacy, and obtain a new privacy definition by allowing the \textit{full} trade-off between type I and type II errors in the simple hypothesis testing problem \eqref{eq:hyp_into} to be governed by some function $f$. The functional privacy parameter $f$ is to this new definition as $(\ep, \delta)$ is to the original definition of differential privacy. Notably, this definition that we term $f$-differential privacy ($f$-DP)---which captures $(\ep, \delta)$-DP as a special case---is accompanied by a powerful and elegant toolkit
for reasoning about composition. Here, we highlight some of our contributions:


\smallskip
\noindent {\bf An Algebra for Composition.} We show that our privacy definition is \emph{closed} and \textit{tight} under composition, which means that the trade-off between type I and type II errors that results from the composition of an $f_1$-DP mechanism with an $f_2$-DP mechanism can always be \emph{exactly} described by a certain function $f$. This function can be expressed via $f_1$ and $f_2$ in an algebraic fashion, thereby allowing for losslessly reasoning about composition. In contrast, $(\epsilon,\delta)$-DP or any other privacy definition artificially restricts itself to a small number of parameters. By allowing for a \emph{function} to keep track of the privacy guarantee of the mechanism, our new privacy definition avoids the pitfall of premature summarization\footnote{To quote Susan Holmes \cite{evil}, ``premature
  summarization is the root of all evil in statistics.''} in intermediate steps and, consequently, yields a comprehensive delineation of the overall privacy guarantee. See more details in \Cref{sec:composition-theorems}.

\smallskip

\noindent{\bf A Central Limit Phenomenon.} We define a single-parameter family of $f$-DP that uses the type I and type II error trade-off in distinguishing the standard normal distribution $\N(0,1)$ from $\N(\mu,1)$ for $\mu \ge 0$. This is referred to as Gaussian differential privacy (GDP). By relating to the hypothesis testing interpretation of differential privacy \eqref{eq:hyp_into}, the GDP guarantee can be interpreted as saying that determining whether or not Alice is in the dataset is at least as difficult as telling apart $\N(0,1)$ and $\N(\mu,1)$ based on one draw. Moreover, we show that GDP is a ``canonical'' privacy guarantee in a fundamental sense: for any privacy definition that retains a hypothesis testing interpretation, we prove that the privacy guarantee of
composition with an appropriate scaling converges to GDP in the limit. This central limit theorem type of result is remarkable not only because of its profound theoretical implication, but also for providing a computationally tractable tool for analytically approximating the privacy loss under composition. Figure~\ref{fig:contribution} demonstrates that this tool yields surprisingly accurate approximations to the exact trade-off in testing the hypotheses \eqref{eq:hyp_into} or substantially improves on the existing privacy guarantee in terms of type I and type II errors. See \Cref{sub:gaussian_differential_privacy} and \Cref{sec:composition-theorems} for a thorough discussion.


\begin{figure}[!htp]
\centering
  \includegraphics[width=0.85\linewidth]{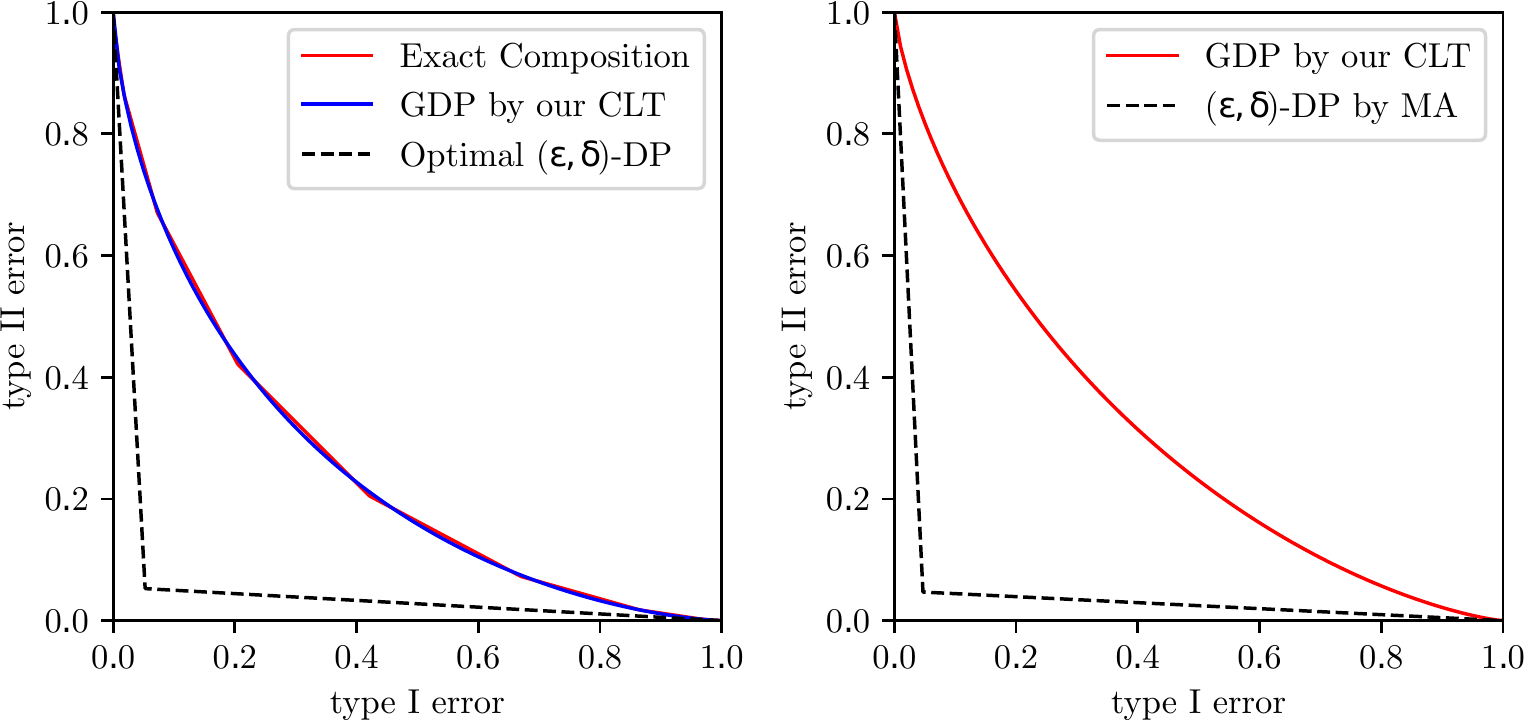}
  \captionof{figure}{Left: Our central limit theorem based approximation (in blue) is very close to the composition of just $10$ mechanisms (in red). The tightest possible approximation via an $(\ep,\delta)$-DP guarantee (in back) is substantially looser. See \Cref{fig:comp} for parameter setup. Right: Privacy analysis of stochastic gradient descent used to train a convolutional neural network on MNIST \cite{lecun-mnisthandwrittendigit-2010}. The $f$-DP framework yields a  privacy guarantee (in red) for this problem that is significantly better than the optimal $(\epsilon,\delta)$-DP guarantee (in black) that is derived from the moments accountant (MA) method \cite{deep}. Put simply, our analysis shows that stochastic gradient descent releases less sensitive information than expected in the literature. See \Cref{sec:application_in_sgd} for more plots and details.}
  \label{fig:contribution}
\end{figure}


\smallskip
\noindent {\bf A Primal-Dual Perspective.} We show a general duality between $f$-DP and infinite collections of $(\epsilon,\delta)$-DP guarantees. This duality is useful in two ways. First, it allows one to analyze an algorithm in the framework of $f$-DP, and then convert back to an $(\epsilon,\delta)$-DP guarantee at the end, if desired. More fundamentally, this duality provides an approach to import techniques developed for $(\epsilon,\delta)$-DP to the framework of $f$-DP. As an important application, we use this duality to show how to reason simply about privacy amplification by subsampling for $f$-DP, by leveraging existing results for $(\epsilon,\delta)$-DP. This is in contrast to divergence based notions of privacy, in which reasoning about amplification by subsampling is difficult.


\smallskip

Taken together, this collection of attractive properties render $f$-DP a mathematically coherent, computationally efficient, and versatile framework for privacy-preserving data analysis. To demonstrate the practical use of this hypothesis testing based framework, we give a substantially sharper analysis of the privacy guarantees of noisy stochastic gradient descent, improving on previous special-purpose analyses that reasoned about divergences rather than directly about hypothesis testing \cite{deep}. This application is presented in \Cref{sec:application_in_sgd}.

\section{$f$-Differential Privacy and Its Basic Properties}
\label{sec:fDP}


In Section~\ref{sub:trade_off_function}, we give a formal definition of $f$-DP. Section~\ref{sub:gaussian_differential_privacy} introduces Gaussian differential privacy, a special case of $f$-DP. In Section~\ref{sec:conn-with-blackw}, we highlight some appealing properties of this new privacy notation from an information-theoretic perspective. Next, Section~\ref{sub:a_primal_dual_connection_with_} offers a profound connection between $f$-DP and $(\epsilon, \delta)$-DP. Finally, we discuss the group privacy properties of $f$-DP.

Before moving on, we first establish several key pieces of notation from the differential privacy literature.
\begin{itemize}
\item  
{\bf Dataset.} A dataset $S$ is a collection of $n$ records, each corresponding to an individual. Formally, we write the dataset as $S = (x_1,\ldots, x_n)$, and an individual $x_i \in X$ for some abstract space $X$. Two datasets $S' = (x'_1,\ldots, x'_n)$ and $S$ are said to be \textit{neighbors} if they differ in exactly one record, that is, there exists an index $j$ such that $x_i = x_i'$ for all $i \ne j$ and $x_j \ne x_j'$.



\item {\bf Mechanism.}  A mechanism $M$ refers to a randomized algorithm that takes as input a dataset $S$ and releases some (randomized) statistics $M(S)$ of the dataset in some abstract space $Y$. For example, a mechanism can release the average salary of individuals in the dataset plus some random noise.


\end{itemize}


\subsection{Trade-off Functions and $f$-DP}
\label{sub:trade_off_function}


All variants of differential privacy informally require that it be hard to \textit{distinguish} any pairs of neighboring datasets based on the information released by a private a mechanism $M$. From an attacker's perspective, it is natural to formalize this notion of  ``indistinguishability'' as a hypothesis testing problem for two neighboring datasets $S$ and $S'$:
\[
H_0: \text{the underlying dataset is } S \quad\text{ versus }\quad H_1: \text{the underlying dataset is } S'.
\]
The output of the mechanism $M$ serves as the basis for performing the hypothesis testing problem. Denote by $P$ and $Q$ the probability distributions of the mechanism applied to the two datasets, namely $M(S)$ and $M(S')$, respectively. The fundamental difficulty in distinguishing the two hypotheses is best delineated by the \textit{optimal} trade-off between the achievable type I and type II errors. More precisely, consider a rejection rule $0 \le \phi \le 1$, with type I and type II error rates defined as\footnote{A rejection rule takes as input the released results of the mechanism. We flip a coin and reject the null hypothesis with probability $\phi$.}
\[
\alpha_\phi = \E_{P}[\phi], \quad \beta_\phi = 1 - \E_{Q}[\phi],
\]
respectively. The two errors satisfy, for example, the constraint is well known to satisfy
\begin{equation}\label{eq:tv_norm}
\alpha_{\phi} + \beta_{\phi} \ge 1 - \TV(P, Q),
\end{equation}
where the total variation distance $\TV(P, Q)$ is the supremum of $|P(A) - Q(A)|$ over all measurable sets $A$. Instead of this rough constraint, we seek to characterize the fine-grained trade-off between the two errors. Explicitly, fixing the type I error at \textit{any} level, we consider the minimal achievable type II error. This motivates the following definition.
\begin{definition}[trade-off function] \label{def:trade-off}
For any two probability distributions $P$ and $Q$ on the same space, define the trade-off function $\F(P, Q): [0, 1] \rightarrow [0, 1]$ as
\[
\F(P, Q)(\alpha) = \inf \left\{ \beta_\phi: \alpha_\phi \leqslant \alpha \right\},
\]
where the infimum is taken over all (measurable) rejection rules.
\end{definition}
The trade-off function serves as a clear-cut boundary of the achievable and unachievable regions of type I and type II errors, rendering itself the \textit{complete} characterization of the fundamental difficulty in testing between the two hypotheses. In particular, the greater this function is, the harder it is to distinguish the two distributions. For completeness, we remark that the minimal $\beta_{\phi}$ can be achieved by the likelihood ratio test---a fundamental result known as the Neyman--Pearson lemma, which we state in the appendix as \Cref{thm:NPlemma}.

A function is called a trade-off function if it is equal to $T(P, Q)$ for some distributions $P$ and $Q$. Below we give a necessary and sufficient condition for $f$ to be a trade-off function. This characterization reveals, for example, that $\max\{f, g\}$ is a trade-off function if both $f$ and $g$ are trade-off functions.
\begin{restatable}{proposition}{tradeoffthm}\label{prop:trade-off}
A function $f: [0, 1] \rightarrow [0, 1]$ is a trade-off function if and only if $f$ is convex, continuous\footnote{Convexity itself implies continuity in $(0,1)$ for $f$. In addition, $f(\alpha) \ge 0$ and $f(\alpha) \le 1-\alpha$ implies continuity at 1. Hence, the continuity condition only matters at $x = 0$.}, non-increasing, and $f(x) \leqslant 1-x$ for $x \in [0,1]$.
\end{restatable}

Now, we propose a new generalization of differential privacy built on top of trade-off functions. Below, we write $g \ge f$ for two functions defined on $[0, 1]$ if $g(x) \ge f(x)$ for all $0 \le x \le 1$, and we abuse notation by identifying $M(S)$ and $M(S')$ with their corresponding probability distributions. Note that if $T(P,Q) \ge T(\widetilde P, \widetilde Q)$, then in a very strong sense, $P$ and $Q$ are harder to distinguish than $\widetilde P$ and $\widetilde Q$ at \textit{any} level of type I error.

\begin{definition}[$f$-differential privacy] \label{def:kstable}
Let $f$ be a trade-off function. A mechanism $M$ is said to be $f$-differentially private if
\[
T\big(M(S), M(S')\big) \ge f
\]
for all neighboring datasets $S$ and $S'$.
\end{definition}

A graphical illustration of this definition is shown in Figure~\ref{fig:intro_def}. Letting $P$ and $Q$ be the distributions such that $f = \F(P, Q)$, this privacy definition amounts to saying that a mechanism is $f$-DP if distinguishing any two neighboring datasets based on the released information is at least as difficult as distinguishing $P$ and $Q$ based on a single draw. In contrast to existing definitions of differential privacy, our new definition is parameterized by a function, as opposed to several real valued parameters (e.g.~$\epsilon$ and $\delta$). This functional perspective offers a complete characterization of ``privacy'', thereby avoiding the pitfall of summarizing statistical information too early. This fact is crucial to the development of a composition theorem for $f$-DP in \Cref{sec:composition-theorems}. Although this completeness comes at the cost of increased complexity, as we will see in Section~\ref{sub:gaussian_differential_privacy}, a simple family of trade-off functions can often closely capture privacy loss in many scenarios.

\begin{figure}[!htp]
\centering
\includegraphics[width=.6\linewidth]{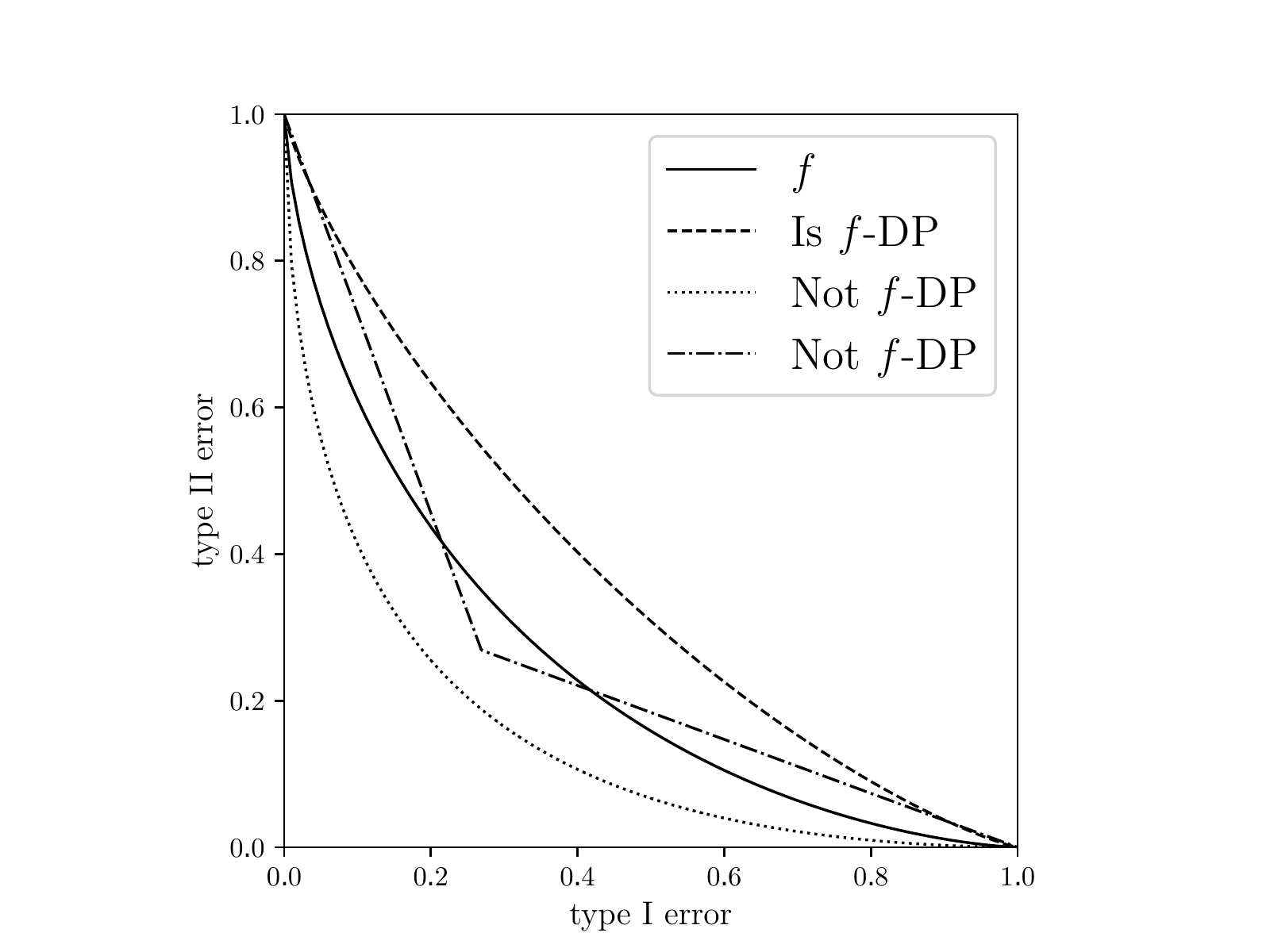}
\caption{Three different examples of $T\big(M(S),M(S')\big)$. Only the dashed line corresponds to a trade-off function satisfying $f$-DP.}
\label{fig:intro_def}
\end{figure}

Naturally, the definition of $f$-DP is symmetric in the same sense as the neighboring relationship, which by definition is symmetric. Observe that this privacy notion also requires
\[
T\big(M(S'), M(S)\big) \ge f
\]
for any neighboring pair $S, S'$. Therefore, it is desirable to restrict our attention to ``symmetric'' trade-off functions. Proposition~\ref{prop:symmetry} shows that this restriction does not lead to any loss of generality.
\begin{restatable}{proposition}{symmrep}\label{prop:symmetry}
Let a mechanism $M$ be $f$-DP. Then, $M$ is $\Sf$-DP with $\Sf = \max\{f, f^{-1}\}$, where the inverse function is defined as\footnote{
\Cref{eq:inver_f} is the standard definition of the left-continuous inverse of a decreasing function. When $f$ is strictly decreasing and $f(0)=1$ and hence bijective as a mapping, \eqref{eq:inver_f} corresponds to the inverse function in the ordinary sense, i.e. $f(f^{-1}(x)) = f^{-1}(f(x)) = x$. However, this is not true in general.}
\begin{equation}\label{eq:inver_f}
f^{-1}(\alpha):=\inf\{t \in [0,1]:f(t) \leqslant \alpha\}
\end{equation}
for $\alpha \in[0,1]$.
\end{restatable}

Writing $f = \F(P,Q)$, we can express the inverse as $f^{-1} = \F(Q, P)$, which therefore is also a trade-off function. As a consequence of this, $\Sf$ continues to be a trade-off function by making use of \Cref{prop:trade-off} and, moreover, is \textit{symmetric} in the sense that 
\[
\Sf = (\Sf)^{-1}.
\] 
Importantly, this symmetrization gives a tighter bound in the privacy definition since $\Sf \geqslant f$. In the remainder of the paper, therefore, trade-off functions will always be assumed to be symmetric unless otherwise specified. We prove \Cref{{prop:symmetry}} in \Cref{app:fDP}.



We conclude this subsection by showing that $f$-DP is a generalization of $(\ep,\delta)$-DP. This foreshadows a deeper connection between $f$-DP and $(\epsilon, \delta)$-DP that will be discussed in Section~\ref{sub:a_primal_dual_connection_with_}. Denote
\begin{equation}\label{eq:fed}
f_{\ep, \delta}(\alpha) = \max\left\{ 0,1 - \delta - \e^\ep \alpha, \e^{-\ep}(1-\delta-\alpha) \right\}
\end{equation}
for $0 \le \alpha \le 1$, which is a trade-off function. \Cref{fig:DPvsGDP} shows the graph of this function and its evident symmetry. The following result is adapted from \cite{wasserman_zhou}.
\begin{proposition}[\cite{wasserman_zhou}] \label{thm:privacy_testing}
A mechanism $M$ is $(\ep, \delta)$-DP if and only if $M$ is $f_{\ep, \delta}$-DP.
\end{proposition}

\begin{figure}[!htp]
\centering
  \includegraphics[width=.75\linewidth]
{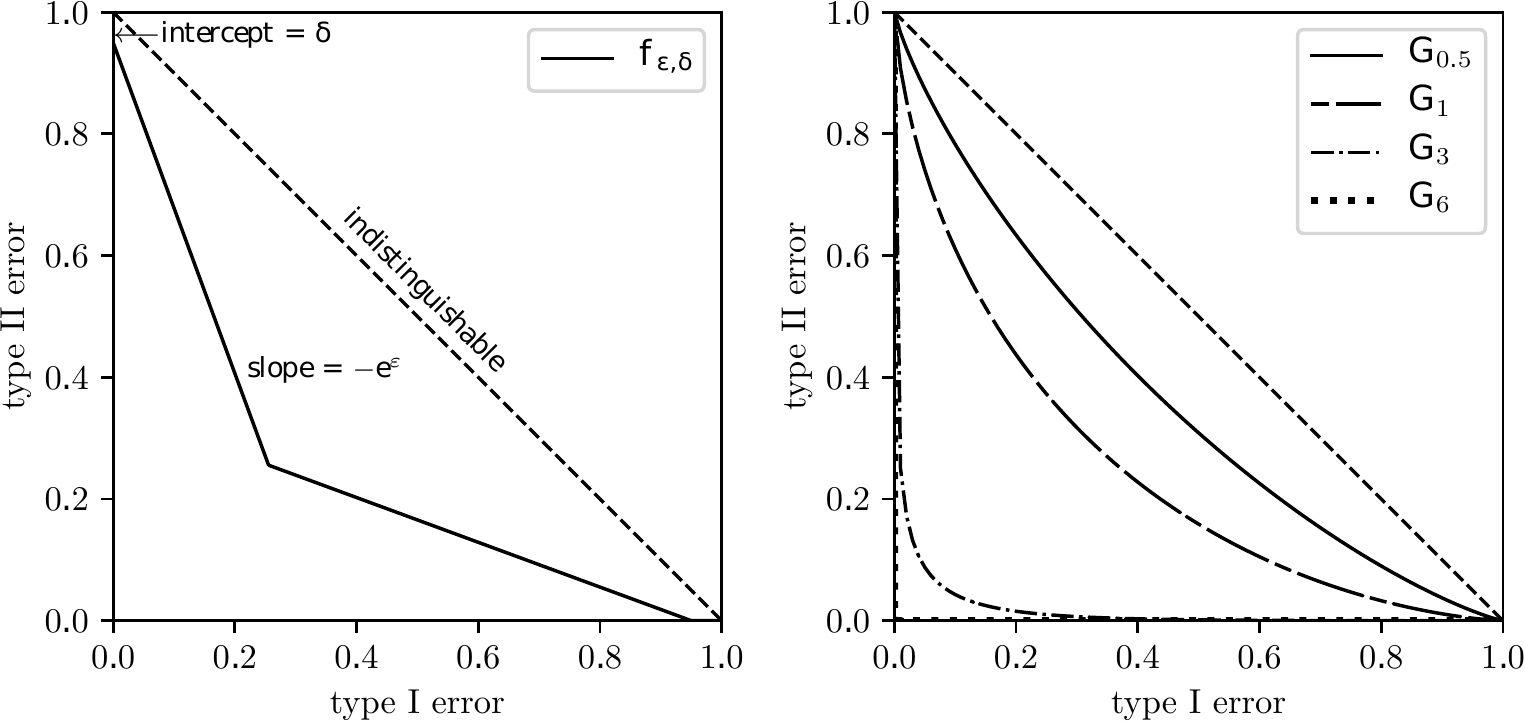}
  \captionof{figure}{Left: $f_{\ep,\delta}$ is a piecewise linear function and is symmetric with respect to the line $y = x$. It has (nontrivial) slopes $-\e^{\pm\ep}$ and intercepts $1-\delta$. Right: Trade-off functions of  unit-variance Gaussian distributions with different means. The case of $\mu=0.5$ is reasonably private, $\mu=1$ is borderline private, and $\mu=3$ is basically non-private: an adversary can control type I and type II errors simultaneously at only 0.07. In the case of $\mu=6$ (almost coincides with the axes), the two errors both can be as small as 0.001.}
  \label{fig:DPvsGDP}
\end{figure}

\subsection{Gaussian Differential Privacy} 
\label{sub:gaussian_differential_privacy}


This subsection introduces a parametric family of $f$-DP guarantees, where $f$ is the trade-off function of two normal distributions. We refer to this specialization as Gaussian differential privacy (GDP). GDP enjoys many desirable properties that lead to its central role in this paper. Among others, we can now precisely define the trade-off function with a single parameter. To define this notion, let
\[
G_\mu:=\F\big(\N(0,1), \N(\mu,1)\big)
\]
for $\mu \ge 0$. An explicit expression for the trade-off function $G_\mu$ reads
\begin{equation}\label{eq:Gmu}
	G_\mu(\alpha) = \Phi\big(\Phi^{-1}(1-\alpha)-\mu\big),
\end{equation}
where $\Phi$ denotes the standard normal CDF. For completeness, we provide a proof of \eqref{eq:Gmu} in \Cref{app:fDP}. This trade-off function is decreasing in $\mu$ in the sense that $G_\mu \leqslant G_{\mu'}$ if $\mu \ge \mu'$. We now define GDP:
\begin{definition} \label{def:GDP}
A mechanism $M$ is said to satisfy $\mu$-Gaussian Differential Privacy ($\mu$-GDP) if it is $G_\mu$-DP. That is,
\[
T\big(M(S),M(S')\big) \ge G_\mu
\]
for all neighboring datasets $S$ and $S'$.
\end{definition}
GDP has several attractive properties. First, this privacy definition is fully described by the single mean parameter of a unit-variance Gaussian distribution, which makes it easy to describe and interpret the privacy guarantees. For instance, one can see from the right panel of \Cref{fig:DPvsGDP} that $\mu \le 0.5$ guarantees a reasonable amount of privacy, whereas if $\mu \geqslant 6$, almost nothing is being promised. Second, loosely speaking, GDP occupies a role among all hypothesis testing based notions of privacy that is similar to the role that the Gaussian distribution has among general probability distributions. We formalize this important point by proving central limit theorems for $f$-DP in \Cref{sec:composition-theorems}, which, roughly speaking, says that $f$-DP converges to GDP under composition in the limit. Lastly, as shown in the remainder of this subsection, GDP \emph{precisely} characterizes the Gaussian mechanism, one of the most fundamental building blocks of differential privacy.

Consider the problem of privately releasing a univariate statistic $\theta(S)$ of the dataset $S$. Define the sensitivity of $\theta$ as
\[
\mathrm{sens}(\theta) = \sup_{S, S'} | \theta(S) - \theta(S') |,
\]
where the supremum is over all neighboring datasets. The Gaussian mechanism adds Gaussian noise to the statistic $\theta$ in order to obscure whether $\theta$ is computed on $S$ or $S'$. The following result shows that the Gaussian mechanism with noise properly scaled to the sensitivity of the statistic satisfies GDP.
\begin{theorem}\label{thm:g_mech}
Define the Gaussian mechanism that operates on a statistic $\theta$ as $M(S) = \theta(S) + \xi$, where $\xi \sim \N(0, \mathrm{sens}(\theta)^2/\mu^2)$. Then, $M$ is $\mu$-GDP.
\end{theorem}
\begin{proof}[Proof of Theorem~\ref{thm:g_mech}]
Recognizing that $M(S), M(S')$ are normally distributed with means $\theta(S), \theta(S')$, respectively, and common variance $\sigma^2 = \mathrm{sens}(\theta)^2/\mu^2$, we get
\begin{align*}
\F\big(M(S),M(S')\big) = \F\big( \N(\theta(S), \sigma^2), \N(\theta(S'), \sigma^2)\big) = G_{|\theta(S)-\theta(S')|/\sigma}.
\end{align*}
By the definition of sensitivity, $|\theta(S)-\theta(S')|/\sigma \le \mathrm{sens}(\theta)/\sigma = \mu$. Therefore, we get
\[
\F\big(M(S),M(S')\big) = G_{|\theta(S)-\theta(S')|/\sigma} \geqslant G_\mu.
\]
This completes the proof.
\end{proof}


As implied by the proof above, GDP offers the tightest possible privacy bound of the Gaussian mechanism. More precisely, the Gaussian mechanism in \Cref{thm:g_mech} satisfies
\begin{equation}\label{eq:f_inf}
G_{\mu}(\alpha) = \inf_{\text{neighboring }S, S'} \, \F \big(M(S),M(S') \big)(\alpha),
\end{equation}
where the infimum is (asymptotically) achieved at the two neighboring datasets such that $|\theta(S) - \theta(S')| = \mathrm{sens}(\theta)$ \textit{irrespective} of the type I error $\alpha$. As such, the characterization by GDP is precise in the pointwise sense. In contrast, the right-hand side of \eqref{eq:f_inf} in general is not necessarily a convex function of $\alpha$ and, in such case, is not a trade-off function according to \Cref{prop:trade-off}. This nice property of Gaussian mechanism is related to the log-concavity of Gaussian distributions. See \Cref{prop:logconcave} for a detailed treatment of log-concave distributions.

\subsection{Post-Processing and the Informativeness of $f$-DP}
\label{sec:conn-with-blackw}

Intuitively, a data analyst cannot make a statistical analysis more disclosive only by processing the output of the mechanism $M$. This is called the post-processing property, a natural requirement that any notion of privacy, including our definition of $f$-DP, should satisfy.



To formalize this point for $f$-DP, denote by $\Proc: Y \to Z$ a (randomized) algorithm that maps the input $M(S) \in Y$ to some space $Z$, yielding a new mechanism that we denote by $\Proc \circ M$. The following result confirms the post-processing property of $f$-DP.
\begin{proposition}\label{prop:post_process}
If a mechanism $M$ is $f$-DP, then its post-processing $\mathrm{Proc}\circ M$ is also $f$-DP.
\end{proposition}
Proposition~\ref{prop:post_process} is a consequence of the following lemma. Let $\Proc(P)$ be the probability distribution of $\Proc(\zeta)$ with $\zeta$ drawn from $P$. Define $\Proc(Q)$ likewise.
\begin{restatable}{lemma}{postrep} \label{lem:post}
For any two distributions $P$ and $Q$, we have
$$\F\big(\mathrm{Proc}(P),\mathrm{Proc}(Q)\big)\geqslant \F(P,Q).$$
\end{restatable}

This lemma means that post-processed distributions can only become more difficult to tell apart than the original distributions from the perspective of trade-off functions. While the same property holds for many divergence based measures of indistinguishability such as the {\Renyi} divergences\footnote{See \Cref{app:relation} for its definition and relation with trade-off functions.} used by the concentrated differential privacy family of definitions \cite{concentrated,concentrated2,renyi,tcdp}, a consequence of the following theorem is that trade-off functions offer the most informative measure among all. This remarkable inverse of \Cref{lem:post} is due to Blackwell (see also Theorem 2.5 in \cite{KOV}).
\begin{theorem}[\cite{blackwell1950comparison}, Theorem 10]\label{thm:blackwell}
Let $P,Q$ be probability distributions on $Y$ and $P',Q'$ be probability distributions on $Z$. The following two statements are equivalent:
\begin{enumerate}
\item[(a)] $\F(P,Q)\leqslant \F(P',Q')$.
\item[(b)] There exists a randomized algorithm $\Proc: Y \to Z$ such that $\Proc(P)=P',\Proc(Q)=Q'$.
\end{enumerate}
\end{theorem}
\newcommand{\Ineq}{\mathrm{Ineq}}
To appreciate the implication of this theorem, we begin by observing that post-processing induces an order\footnote{This is in general not a partial order.} on pairs of distributions, which is called the Blackwell order (see, e.g., \cite{raginsky2011shannon}). Specifically, if the above condition (b) holds, then we write $(P,Q)\preceq_{\mathrm{Blackwell}}(P',Q')$ and interpret this as ``$(P,Q)$ is easier to distinguish than $(P',Q')$ in the Blackwell sense''. Similarly, when $\F(P,Q)\leqslant \F(P',Q')$, we write $(P,Q)\preceq_{\mathrm{tradeoff}}(P',Q')$ and interpret this as ``$(P,Q)$ is easier to distinguish than $(P',Q')$ in the testing sense''. In general, any privacy measure used in defining a privacy notion induces an order $\preceq$ on pairs of distributions. Assuming the post-processing property for the privacy notion, the induced order $\preceq$ must be consistent with $\preceq_{\mathrm{Blackwell}}$. Concretely, we
denote by $\Ineq(\preceq) = \{(P,Q; P', Q'): (P, Q) \preceq (P', Q')\}$ the set of all comparable pairs of the order $\preceq$. As is clear, a privacy notion satisfies the post-processing property if and only if the induced order $\preceq$ satisfies $\Ineq({\preceq})\supseteq \Ineq({\preceq_{\mathrm{Blackwell}}})$. 


Therefore, for any reasonable privacy notion, the set $\Ineq({\preceq})$ must be large enough to contain $\Ineq({\preceq_{\mathrm{Blackwell}}})$. However, it is also desirable to have a not too large $\Ineq({\preceq})$. For example, consider the privacy notion based on a trivial divergence $D_0$ with $D_0(P\|Q) \equiv 0$ for any $P,Q$. Note that $\Ineq({\preceq_{D_0}})$ is the largest possible and, meanwhile, it is not informative at all in terms of measuring the indistinguishability of two distributions.


The argument above suggests that going from the ``minimal'' order $\Ineq({\preceq_{\mathrm{Blackwell}}})$ to the ``maximal'' order $\Ineq({\preceq_{D_0}})$ would lead to information loss. Remarkably, $f$-DP is the most informative differential privacy notion from this perspective because its induced order $\preceq_{\mathrm{tradeoff}}$ satisfies $\Ineq({\preceq_{\mathrm{tradeoff}}}) = \Ineq({\preceq_{\mathrm{Blackwell}}})$. In stark contrast, this is not true for the order induced by other popular privacy notions such as R\'enyi differential privacy and $(\epsilon,\delta)$-DP. We prove this claim in \Cref{app:relation} and further justify the informativeness of $f$-DP by providing general tools that can losslessly convert $f$-DP guarantees into divergence based privacy guarantees.

\newcommand{\Symm}{\mathrm{Symm}}
\subsection{A Primal-Dual Perspective}
\label{sub:a_primal_dual_connection_with_}

In this subsection, we show that $f$-DP is equivalent to an infinite \textit{collection} of $(\ep,\delta)$-DP guarantees via the convex conjugate of the trade-off function. As a consequence of this, we can view $f$-DP as the \textit{primal} privacy representation and, accordingly, its \textit{dual} representation is the collection of $(\ep,\delta)$-DP guarantees. Taking this powerful viewpoint, many results from the large body of $(\epsilon, \delta)$-DP work can be carried over to $f$-DP in a seamless fashion. In particular, this primal-dual perspective is crucial to our analysis of ``privacy amplification by subsampling'' in \Cref{sec:subsampling}. All proofs are deferred to \Cref{app:fDP}.


First, we present the result that converts a collection of $(\epsilon, \delta)$-DP guarantees into an $f$-DP guarantee.
\begin{proposition}[Dual to Primal] \label{prop:DPtof}
Let $I$ be an arbitrary index set such that each $i\in I$ is associated with $\ep_i\in[0, \infty)$ and $\delta_i\in[0,1]$. A mechanism is $(\ep_i,\delta_i)$-DP for all $i\in I$ if and only if it is $f$-DP with
$$f = \sup_{i\in I} f_{\ep_i,\delta_i}.$$
\end{proposition}
This proposition follows easily from the equivalence of $(\ep,\delta)$-DP and $f_{\ep,\delta}$-DP. We remark that the function $f$ constructed above remains a symmetric trade-off function.

The more interesting direction is to convert $f$-DP into a collection of $(\ep,\delta)$-DP guarantees. Recall that the convex conjugate of a function $g$ defined on $(-\infty, \infty)$ is defined as
\begin{equation}\label{eq:conjugate}
g^*(y) = \sup_{-\infty < x < \infty} y x - g(x).
\end{equation}
To define the conjugate of a trade-off function $f$, we extend its domain by setting $f(x) = \infty$ for $x < 0$ and $x > 1$. With this adjustment, the supremum is effectively taken over $0 \le x \le 1$.

\begin{restatable}[Primal to Dual]{proposition}{ftoDPrep} \label{prop:ftoDP}
	For a symmetric trade-off function $f$, a mechanism is $f$-DP if and only if it is $\big(\ep,\delta(\ep)\big)$-DP for all $\ep\geqslant 0$ with $\delta(\ep)=1+f^*(-\e^{\ep})$.
\end{restatable}

\begin{figure}[!ht]
\centering
  \includegraphics[width=0.6\linewidth]
{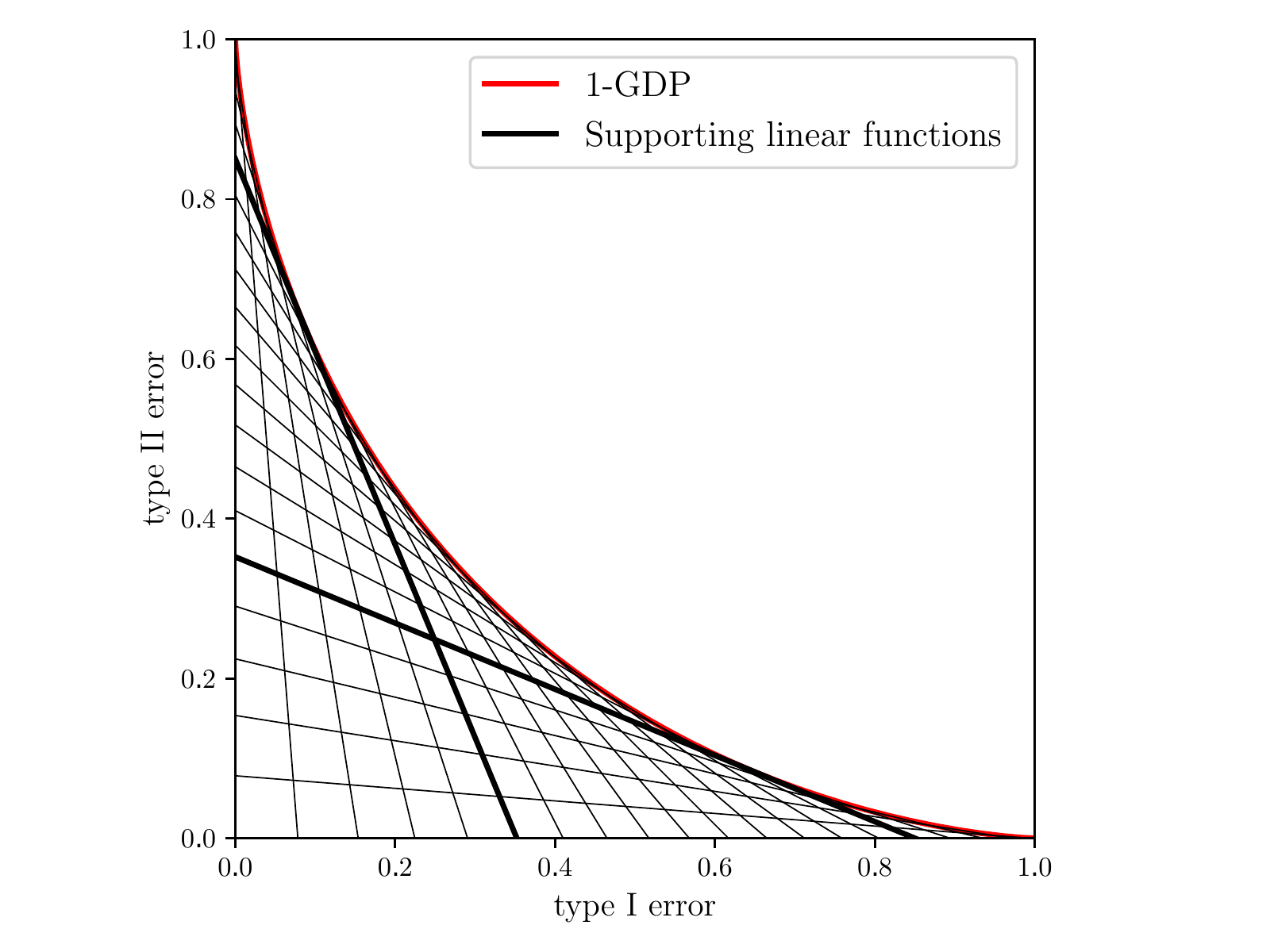}
  \captionof{figure}{Each $(\ep,\delta(\ep))$-DP guarantee corresponds to two supporting linear functions (symmetric to each other) to the trade-off function describing the complete $f$-DP guarantee. In general, characterizing a privacy guarantee using only a subset of $(\ep,\delta)$-DP guarantees (for example, only those with small $\delta$) would result in information loss.}
  \label{fig:envelope}
\end{figure}

For example, taking $f=G_\mu$, the following corollary provides a lossless conversion from GDP to a collection of $(\ep,\delta)$-DP guarantees. This conversion is exact and, therefore, any other $(\ep,\delta)$-DP guarantee derived for the Gaussian mechanism is implied by this corollary. See \Cref{fig:envelope} for an illustration of this result.
\begin{restatable}{corollary}{GDPtoDPrep}\label{corr:GDPtoDP}
    A mechanism is $\mu$-GDP if and only if it is $\big(\ep,\delta(\ep)\big)$-DP for all $\ep\geqslant0$, where
    \[
    \delta(\ep)= \Phi\Big( -\frac{\varepsilon}{\mu} +\frac{\mu}{2} \Big)-
    \e^{\varepsilon}\Phi\Big(- \frac{\varepsilon}{\mu} - \frac{\mu}{2} \Big).
    \]
\end{restatable}
This corollary has appeared earlier in \cite{balle2018improving}. Along this direction, \cite{balle2018privacy} further proposed ``privacy profile'', which in essence corresponds to an infinite collection of $(\ep,\delta)$. The notion of privacy profile mainly serves as an analytical tool in \cite{balle2018privacy}.

The primal-dual perspective provides a useful tool through which we can bridge the two privacy definitions. In some cases, it is easier to work with $f$-DP by leveraging the interpretation and informativeness of trade-off functions, as seen from the development of composition theorems for $f$-DP in \Cref{sec:composition-theorems}. Meanwhile, $(\ep,\delta)$-DP is more convenient to work with in the cases where the lower complexity of two parameters $\epsilon,\delta$ is helpful, for example, in the proof of the privacy amplification by subsampling theorem for $f$-DP. In short, our approach in \Cref{sec:subsampling} is to first work in the dual world and use existing subsampling theorems for $(\ep,\delta)$-DP, and then convert the results back to $f$-DP using a slightly more advanced version of \Cref{prop:ftoDP}.

\subsection{Group Privacy}\label{sub:group_privacy}
\newcommand{\group}{\mathbin{\hat{\circ}}}

The notion of $f$-DP can be extended to address privacy of a \textit{group} of individuals, and a question of interest is to quantify how privacy degrades as the group size grows. To set up the notation, we say that two datasets $S, S'$ are $k$-neighbors (where $k \ge 2$ is an integer) if there exist datasets $S = S_0, S_1, \ldots, S_k = S'$ such that $S_i$ and $S_{i+1}$ are neighboring or identical for all $i = 0, \ldots, k-1$. Equivalently, $S,S'$ are $k$-neighbors if they differ by at most $k$ individuals. Accordingly, a mechanism $M$ is said to be $f$-DP for \textit{groups of size $k$} if
\[ T\big(M(S),M(S')\big)\geqslant f\]
for all $k$-neighbors $S$ and $S'$.

In the following theorem, we use $h^{\circ k}$ to denote the $k$-fold iterative composition of a function $h$. For example, $h^{\circ 1} = h$ and $h^{\circ 2}(x) = h(h(x))$.
\begin{restatable}{theorem}{groupthm} \label{thm:group}
If a mechanism is $f$-DP, then it is $\left[1-(1-f)^{\circ k}\right]$-DP for groups of size $k$. In particular, if a mechanism is $\mu$-GDP, then it is $k\mu$-GDP for groups of size $k$. 
\end{restatable}
For completeness, $1-(1-f)^{\circ k}$ is a trade-off function and, moreover, remains symmetric if $f$ is symmetric. These two facts and \Cref{thm:group} are proved in \Cref{app:fDP}. As revealed in the proof, the privacy bound $1-(1-f)^{\circ k}$ in general cannot be improved, thereby showing that the group operation in the $f$-DP framework is \textit{closed} and \textit{tight}. In addition, it is easy to see that $1 - (1-f)^{\circ k} \leqslant 1 - (1-f)^{\circ (k-1)}$ by recognizing that the trade-off function $f$ satisfies $1-f(x)\geqslant x$. This is consistent with the intuition that detecting changes in groups of $k$ individuals becomes easier as the group size increases.

As an interesting consequence of \Cref{thm:group}, the group privacy of $\epsilon$-DP in the limit corresponds to the trade-off function of two Laplace distributions. Recall that the density of $\mathrm{Lap}(\mu,b)$ is $\frac1{2b}\e^{-|x-\mu|/b}$.
\begin{restatable}{proposition}{grouplimit}
\label{prop:group_limit}
Fix $\mu \ge 0$ and set $\epsilon = \mu/k$. As $k \to \infty$, we have
\[
1-(1-f_{\ep,0})^{\circ k} \to \F\big(\mathrm{Lap}(0,1),\mathrm{Lap}(\mu,1)\big).
\]
The convergence is uniform over $[0,1]$.
\end{restatable}

Two remarks are in order. First, $\F\big(\mathrm{Lap}(0,1),\mathrm{Lap}(\mu,1)\big)$ is not equal to $f_{\ep,\delta}$ for any $\ep,\delta$ and, therefore, $(\ep,\delta)$-DP is not expressive enough to measure privacy under the group operation. Second, the approximation in this theorem is very accurate even for small $k$. For example, for $\mu=1,k=4$, the function $1-(1-f_{\ep,0})^{\circ k}$ is within $0.005$ of $\F\big(\mathrm{Lap}(0,1),\mathrm{Lap}(\mu,1)\big)$ uniformly over $[0, 1]$. The proof of \Cref{prop:group_limit} is deferred to \Cref{app:fDP}.

\section{Composition and Limit Theorems}
\label{sec:composition-theorems}

\newcommand{\kl}{\mathrm{kl}}
\newcommand{\lk}{\mathrm{lk}}
\renewcommand{\Proc}{\mathrm{Proc}}

Imagine that an analyst performs a sequence of analyses on a private dataset, in which each analysis is informed by prior analyses on the same dataset. Provided that every analysis alone is private, the question is whether all analyses collectively are private, and if so, how the privacy degrades as the number of analyses increases, namely under composition. It is essential for a notion of privacy to gracefully handle composition, without which the privacy analysis of complex algorithms would be almost impossible.


Now, we describe the composition of two mechanisms. For simplicity, this section writes $X$ for the space of datasets and abuse notation by using $n$ to refer to the number of mechanisms in composition\footnote{As will be clear later, the use of $n$ is consistent with the literature on central limit theorems.}. Let $M_1: X \to Y_1$ be the first mechanism and $M_2: X \times Y_1 \to Y_2$ be the second mechanism. In brief, $M_2$ takes as input the output of the first mechanism $M_1$ in addition to the dataset. With the two mechanisms in place, the joint mechanism $M: X \to Y_1 \times Y_2$ is defined as
\begin{equation}\label{eq:M_2_comp}
M(S) = (y_1, M_2(S, y_1)),
\end{equation}
where $y_1 = M_1(S)$.\footnote{Alternatively, we can write $M(S) = (M_1(S), M_2(S, M_1(S)))$, in which case it is necessary to specify that $M_1$ should be run only once in this expression.} Roughly speaking, the distribution of $M(S)$ is constructed from the marginal distribution of $M_1(S)$ on $Y_1$ and the conditional distribution of $M_2(S, y_1)$ on $Y_2$ given $M_1(S) = y_1$. The composition of more than two mechanisms follows recursively. In general, given a sequence of mechanisms $M_i: X \times Y_1 \times \cdots \times Y_{i-1} \to Y_i$ for $i=1,2,\ldots, n$, we can recursively define the joint mechanism as their composition:
\[
M: X \to Y_1 \times \cdots \times Y_n.
\]
Put differently, $M(S)$ can be interpreted as the trajectory of a Markov chain whose initial distribution is given by $M_1(S)$ and the transition kernel $M_i(S, \cdots)$ at each step.


Using the language above, the goal of this section is to relate the privacy loss of $M$ to that of the $n$ mechanisms $M_1, \ldots, M_n$ in the $f$-DP framework. In short, Section~\ref{sub:composition_theorem} develops a general composition theorem for $f$-DP. In Sections~\ref{sub:a_berry_esseen_type_of_clt}, we identify a central limit theorem phenomenon of composition in the $f$-DP framework, which can be used as an approximation tool, just like we use the central limit theorem for random variables. This approximation is extended to and improved for $(\epsilon, \delta)$-DP in Section~\ref{sub:_ep_0_dp_}.

\subsection{A General Composition Theorem} 
\label{sub:composition_theorem}


The main thrust of this subsection is to demonstrate that the composition of private mechanisms is closed and tight\footnote{\Cref{sub:group_privacy} shows that $f$-DP is ``closed and tight'' in a similar sense, in terms of the guarantees of group privacy.} in the $f$-DP framework. This result is formally stated in Theorem~\ref{thm:n_steps}, which shows that the composed mechanism remains $f$-DP with the trade-off function taking the form of a certain product.  To define the product, consider two trade-off functions $f$ and $g$ that are given as $f = \F(P, Q)$ and $g = \F(P', Q')$ for some probability distributions $P, P', Q, Q'$.



\begin{definition} \label{def:product}
The tensor product of two trade-off functions $f = \F(P, Q)$ and $g = \F(P', Q')$ is defined as
$$f\otimes g := \F(P\times P', Q\times Q').$$
\end{definition}
Throughout the paper, write $f \otimes g (\alpha)$ for $(f \otimes g) (\alpha)$, and denote by $f^{\otimes n}$ the $n$-fold tensor product of $f$. The well-definedness of $f^{\otimes n}$ rests on the associativity of the tensor product, which we will soon illustrate.

By definition, $f\otimes g$ is also a trade-off function. Nevertheless, it remains to be shown that the tensor product is well-defined: that is, the definition is independent of the choice of distributions used to represent a trade-off function. More precisely, assuming $f = \F(P, Q) = \F(\tilde P, \tilde Q)$ for some distributions $\tilde P, \tilde Q$, we need to ensure that
\begin{equation*}\label{eq:tensor_welldef}
T(P\times P', Q \times Q') = T(\tilde{P} \times P', \tilde{Q} \times Q').
\end{equation*}
We defer the proof of this intuitive fact to \Cref{app:self}.  Below we list some other useful properties\footnote{These properties make the class of trade-off functions a \textit{commutative monoid}. Informally, a monoid is a group without the inverse operator.} of the tensor product of trade-off functions, whose proofs are placed in \Cref{app:CLT}.
\begin{enumerate}
\setlength\itemsep{0.15em}
\item The product $\otimes$ is commutative and associative.
\item If $g_1\geqslant g_2$, then $f\otimes g_1 \geqslant f\otimes g_2$.
\item $f \otimes \Id = \Id \otimes f = f$, where the identity trade-off function $\Id(x)=1-x$ for $0 \le x \le 1$.
\item $(f\otimes g)^{-1} = f^{-1}\otimes g^{-1}$. See the definition of inverse in \eqref{eq:inver_f}.
\end{enumerate}
Note that $\Id$ is the trade-off function of two identical distributions. Property 4 implies that when $f,g$ are symmetric trade-off functions, their tensor product $f\otimes g$ is also symmetric.

Now we state the main theorem of this subsection. Its proof is given in \Cref{app:self}.
\begin{theorem} \label{thm:n_steps}
Let $M_i(\cdot,y_1,\cdots,y_{i-1})$ be $f_i$-DP for all $y_1 \in Y_1, \ldots, y_{i-1}\in Y_{i-1}$. Then the $n$-fold composed mechanism $M:X\to Y_1\times\cdots\times Y_n$ is $f_1\otimes\cdots\otimes f_n$-DP.
\end{theorem}

This theorem shows that the composition of mechanisms remains $f$-DP or, put differently, composition is closed in the $f$-DP framework. Moreover, the privacy bound $f_1\otimes\cdots\otimes f_n$ in Theorem~\ref{thm:n_steps} is \textit{tight} in the sense that it cannot be improved in general. To see this point, consider the case where the second mechanism completely ignores the output of the first mechanism. In that case, the composition obeys
$$
\begin{aligned}
T\big(M(S),M(S')\big) &= T\big(M_1(S)\times M_2(S),M_1(S')\times M_2(S')\big) \\
&= T\big(M_1(S),M_1(S')\big)\otimes T\big(M_2(S),M_2(S')\big).
\end{aligned}
$$
Next, taking neighboring datasets such that $T\big(M_1(S),M_1(S')\big) = f_1$ and $T\big(M_2(S),M_2(S')\big) = f_2$, one concludes that $f_1 \otimes f_2$ is the tightest possible bound on the two-fold composition. For comparison, the advanced composition theorem for $(\ep, \delta)$-DP does not admit a single pair of optimal parameters $\epsilon, \delta$ \cite{boosting}. In particular, no pair of $\ep,\delta$ can exactly capture the privacy of the composition of $(\ep,\delta)$-DP mechanisms. See \Cref{sub:_ep_0_dp_} and \Cref{fig:comp} for more elaboration.

In the case of GDP,  composition enjoys a simple and convenient formulation due to the identity
\[
G_{\mu_1} \otimes G_{\mu_2} \otimes \cdots \otimes G_{\mu_n} = G_{\mu},
\]
where $\mu = \sqrt{\mu_1^2+\cdots+\mu_n^2}$. This formula is due to the rotational invariance of Gaussian distributions with identity covariance. We provide the proof in \Cref{app:CLT}. The following corollary formally summarizes this finding.


\begin{corollary}\label{cor:gdp_comp}
The $n$-fold composition of $\mu_i$-GDP mechanisms is $\sqrt{\mu_1^2+\cdots+\mu_n^2}$-GDP.
\end{corollary}

On a related note, the pioneering work \cite{KOV} is the first to take the hypothesis testing viewpoint in the study of privacy composition and to use Blackwell's theorem as an analytic tool therein. In particular, the authors offered a composition theorem for $(\ep,\delta)$-DP that improves on the advanced composition theorem \cite{boosting}. Following this work, \cite{complexity} provided a self-contained proof by essentially proving the ``$(\ep,\delta)$ special case'' of Blackwell's theorem. In contrast, our novel proof of \Cref{thm:n_steps} only makes use of the Neyman--Pearson lemma, thereby circumventing the heavy machinery of Blackwell's theorem. This simple proof better illuminates the essence of the composition theorem.

\subsection{Central Limit Theorems for Composition} 
\label{sub:a_berry_esseen_type_of_clt}

\newcommand{\bkl}{\boldsymbol{\kl}}
\newcommand{\bk}{\boldsymbol{\kappa_2}}
\newcommand{\bkk}{\boldsymbol{\kappa_3}}
\newcommand{\bkbar}{\boldsymbol{\bar{\kappa}_3}}


In this subsection, we identify a central limit theorem type phenomenon of composition in the $f$-DP framework. Our main results (\Cref{thm:Berry} and \Cref{thm:CLT}), roughly speaking, show that trade-off functions corresponding to small privacy leakage accumulate to $G_\mu$ for some $\mu$ under composition. Equivalently, the privacy of the composition of many ``very private'' mechanisms is best measured by GDP in the limit. This identifies GDP as the focal privacy definition among the family of $f$-DP privacy guarantees, including $(\epsilon,\delta)$-DP. More precisely, \emph{all} privacy definitions that are based on a hypothesis testing formulation of ``indistinguishability'' converge to the guarantees of GDP in the limit of composition. We remark that \cite{sommer2018privacy} proved a conceptually related central limit theorem for random variables corresponding to the privacy loss. This theorem is used to reason about the non-adaptive composition for $(\epsilon,\delta)$-DP. In
contrast, our central limit theorem is concerned with the optimal hypothesis testing trade-off functions for the composition theorem. Moreover, our theorem is applicable in the setting of composition, where each mechanism is informed by prior interactions with the same database.

From a computational viewpoint, these limit theorems yield an efficient method of approximating the composition of general $f$-DP mechanisms. This is very appealing for analyzing the privacy properties of algorithms that are comprised of many building blocks in a sequence. For comparison, the exact computation of privacy guarantees under composition can be computationally hard \cite{complexity} and, thus, tractable approximations are important. Using our central limit theorems, the computation of the exact overall privacy guarantee $f_1\otimes\cdots\otimes f_n$ in \Cref{thm:n_steps} can be reduced to the evaluation of a single mean parameter $\mu$ in a GDP guarantee. We give an exemplary application of this powerful technique in \Cref{sec:application_in_sgd}.



Explicitly, the mean parameter $\mu$ in the approximation depends on certain functionals of the trade-off functions\footnote{Although the trade-off function satisfies $f'(x) \le 0$ almost everywhere on $[0, 1]$, we prefer to use $|f'(x)|$ instead of $-f'(x)$ for aesthetic reasons.}:
\begin{align*}
	\kl(f) &:= -\int_0^1\log |f'(x)|\diff x\\
	\kappa_2(f)&:=\int_0^1\log^2 |f'(x)| \diff x\\
	\kappa_3(f)&:=\int_0^1\big|\log |f'(x)|\big|^3\diff x\\
	\bar{\kappa}_3(f)&:=\int_0^1\big|\log |f'(x)|+\kl(f)\big|^3\diff x.
\end{align*}
All of these functionals take values in $[0,+\infty]$, and the last is defined for $f$ such that $\kl(f) < \infty$. In essence, these functionals are calculating moments of the log-likelihood ratio of $P$ and $Q$ such that $f=T(P,Q)$. In particular, all of these functionals are 0 if $f(x) = \Id(x) = 1 - x$, which corresponds to zero privacy leakage. As its name suggests, $\kl(f)$ is the Kullback--Leibler (KL) divergence of $P$ and $Q$ and, therefore, $\kl(f) \ge 0$. Detailed elaboration on these functionals is deferred to \Cref{app:CLT}.



In the following theorem, $\boldsymbol{\kl}$ denotes the vector $\big(\kl(f_{1}),\ldots, \kl(f_{n})\big)$ and $\bk, \bkk,\bkbar$ are defined similarly; in addition, $\|\cdot\|_1$ and $\|\cdot\|_2$ are the $\ell_1$ and $\ell_2$ norms, respectively.
\begin{restatable}{theorem}{berryrep} \label{thm:Berry}
Let $f_1,\ldots, f_n$ be symmetric trade-off functions such that $\kappa_3(f_i) < \infty$ for all $1 \le i \le n$. Denote
\[
\mu:= \frac{2\|\bkl\|_1}{\sqrt{\|\bk\|_1 - \|\bkl\|_2^2}}\,\, ~ \text{and} \,\, ~ \gamma:=\frac{0.56\|\bkbar\|_1}{\big(\|\bk\|_1 - \|\bkl\|_2^2\big)^{3/2}}
\]
and assume $\gamma < \frac12$. Then, for all $\alpha\in[\gamma,1-\gamma]$, we have\footnote{We can extend $G_{{\mu}}$ to be 1 in $(-\infty,0)$ and 0 in $(1,+\infty)$ so that the assumption that $\alpha\in[\gamma,1-\gamma]$ can be removed.}
\begin{equation}\label{eq:lower_upper}
G_\mu(\alpha+\gamma)-\gamma\leqslant f_{1}\otimes f_{2} \otimes \cdots \otimes f_{n}(\alpha)\leqslant G_\mu(\alpha-\gamma)+\gamma.
\end{equation}

\end{restatable}


Loosely speaking, the lower bound in \eqref{eq:lower_upper} shows that the composition of $f_i$-DP mechanisms for $i = 1, \ldots, n$ is approximately $\mu$-GDP and, in addition, the upper bound demonstrates that the tightness of this approximation is specified by $\gamma$. In the case where all $f_i$ are equal to some $f\neq \Id$, the theorem reveals that the composition becomes blatantly non-private as $n \to \infty$ because $\mu \asymp \sqrt{n} \to \infty$. More interesting applications of the theorem, however, are cases where each $f_i$ is close to the ``perfect privacy'' trade-off function $\Id$ such that collectively $\mu$ is convergent and $\gamma$ vanishes as $n \to \infty$ (see the example in \Cref{sec:application_in_sgd}). For completeness, the condition $\kappa_3(f_i) < \infty$ (which implies that the other three functionals are also finite) for the use of this theorem excludes the case where $f_i(0) < 1$, in particular, $f_{\epsilon,\delta}$ in $(\epsilon, \delta)$-DP with
$\delta > 0$. We introduce an easy and general technique in \Cref{sub:_ep_0_dp_} to deal with this issue.


From a technical viewpoint, Theorem~\ref{thm:Berry} can be thought of as a Berry--Esseen type central limit theorem.
The detailed proof, as well as that of \Cref{thm:CLT}, is provided in \Cref{app:CLT}.

Next, we present an asymptotic version of \Cref{thm:Berry} for composition of $f$-DP mechanisms. In analogue to classical central limit theorems, below we consider a triangular array of mechanisms  $\{M_{n1},\ldots, M_{nn}\}_{n=1}^{\infty}$, where $M_{ni}$ is $f_{ni}$-DP for $1 \le i \le n$.



\begin{restatable}{theorem}{asymprep} \label{thm:CLT}
Let $\{f_{ni}: 1\leqslant i \leqslant n\}_{n=1}^{\infty}$ be a triangular array of symmetric trade-off functions and assume the following limits for some constants $K \ge 0$ and $s > 0$ as $n \to \infty$:
	\begin{enumerate}
		\item[\textup{1.}] $\sum_{i=1}^n \kl(f_{ni})\to K;$
		\item[\textup{2.}] $\max_{1\leqslant i\leqslant n} \kl(f_{ni}) \to 0;$
		\item[\textup{3.}] $\sum_{i=1}^n \kappa_2(f_{ni})\to s^2;$
		\item[\textup{4.}] $\sum_{i=1}^n \kappa_3(f_{ni})\to 0$.
	\end{enumerate}
Then, we have
	$$\lim_{n\to \infty} f_{n1}\otimes f_{n2} \otimes \cdots \otimes f_{nn} (\alpha) = G_{2K/s}(\alpha)$$
uniformly for all $\alpha \in [0,1]$.
\end{restatable}

Taken together, this theorem and \Cref{thm:n_steps} amount to saying that the composition $M_{n1} \otimes \ldots \otimes M_{nn}$ is asymptotically ${2K/s}$-GDP. In fact, this asymptotic version is a consequence of Theorem~\ref{thm:Berry} as one can show $\mu \to 2K/s$ and $\gamma \to 0$ for the triangular array of symmetric trade-off functions. This central limit theorem implies that GDP is the \emph{only} parameterized family of trade-off functions that can faithfully represent the effects of composition. In contrast, neither $\epsilon$- nor $(\epsilon,\delta)$-DP can losslessly be tracked under composition---the parameterized family of functions $f_{\epsilon,\delta}$ cannot represent the trade-off function that results from the limit under composition.



The conditions for use of this theorem are reminiscent of Lindeberg's condition in the central limit theorem for independent random variables. The proper scaling of the trade-off functions is that both $\kl(f_{ni})$ and $\kappa_2(f_{ni})$ are of order $O(1/n)$ for most $1 \le i \le n$. As a consequence, the cumulative effects of the moment functionals are bounded. Furthermore, as with Lindeberg's condition, the second condition in Theorem~\ref{thm:CLT} requires that no single mechanism has a significant contribution to the composition in the limit.


In passing, we remark that $K$ and $s$ satisfy the relationship $s = \sqrt{2K}$ in all examples of the application of \Cref{thm:CLT} in this paper, including \Cref{thm:DPCLT} and \Cref{thm:mixtureSGD} as well as their corollaries. As such, the composition is asymptotically $s$-GDP. A proof of this interesting observation or the construction of a counterexample is left for future work.

\subsection{Composition of $(\ep,\delta)$-DP: Beating Berry--Esseen} 
\label{sub:_ep_0_dp_}

\newcommand{\Bern}{\mathrm{Bern}}


Now, we extend central limit theorems to $(\ep,\delta)$-DP. As shown by \Cref{thm:privacy_testing}, $(\ep, \delta)$-DP is equivalent to $f_{\epsilon, \delta}$-DP and, therefore, it suffices to approximate the trade-off function $f_{\ep_{1},\delta_{1}}\otimes \cdots \otimes f_{\ep_{n},\delta_{n}}$ by making use of the composition theorem for $f$-DP mechanisms. As pointed out in Section~\ref{sub:a_berry_esseen_type_of_clt}, however, the moment conditions required in the two central limit theorems (Theorems~\ref{thm:Berry} and \ref{thm:CLT}) exclude the case where $\delta_i > 0$.

To overcome the difficulty caused by a nonzero $\delta$, we start by observing the useful fact that
\begin{equation}\label{eq:decomposition}
f_{\ep,\delta} = f_{\ep,0}\otimes f_{0,\delta}.
\end{equation}
This decomposition, along with the commutative and associative properties of the tensor product, shows
$$f_{\ep_{1},\delta_{1}}\otimes \cdots \otimes f_{\ep_{n},\delta_{n}} = \big(f_{\ep_{1},0}\otimes \cdots \otimes f_{\ep_{n},0}\big)\otimes \big(f_{0,\delta_{1}}\otimes \cdots \otimes f_{0,\delta_{n}}\big).$$
This identity allows us to work on the $\epsilon$ part and $\delta$ part separately. In short, the $\epsilon$ part $f_{\ep_{1},0}\otimes \cdots \otimes f_{\ep_{n},0}$ now can be approximated by $G_{\sqrt{\ep_1^2+\cdots+\ep_n^2}}$ by invoking Theorem~\ref{thm:CLT}. For the $\delta$ part, we can iteratively apply the rule
\begin{equation}\label{eq:delta}
f_{0,\delta_1}\otimes f_{0,\delta_2} = f_{0,1-(1-\delta_1)(1-\delta_2)}
\end{equation}
to obtain $f_{0,\delta_{1}}\otimes \cdots \otimes f_{0,\delta_{n}} = f_{0, 1 - (1 - \delta_1)(1 - \delta_2) \cdots (1 - \delta_n)}$. This rule is best seen via the interesting fact that $f_{0,\delta}$ is the trade-off function of shifted uniform distributions $T\big(U[0,1],U[\delta,1+\delta]\big)$.

Now, a central limit theorem for $(\epsilon, \delta)$-DP is just a stone's throw away. In what follows, the privacy parameters $\ep$ and $\delta$ are arranged in a triangular array $\{(\ep_{ni},\delta_{ni}):1\leqslant i\leqslant n\}_{n=1}^{\infty}$.
\begin{restatable}{theorem}{DPCLTrep}\label{thm:DPCLT}
Assume
$$\sum_{i=1}^n \ep_{ni}^2 \to \mu^2, \quad \max_{1\leqslant i\leqslant n} \ep_{ni}\to 0, \quad \sum_{i=1}^n \delta_{ni} \to \delta, \quad \max_{1\leqslant i\leqslant n} \delta_{ni}\to 0$$
for some nonnegative constants $\mu,\delta$ as $n \to \infty$. Then, we have
$$f_{\ep_{n1},\delta_{n1}}\otimes \cdots \otimes f_{\ep_{nn},\delta_{nn}} \to G_{\mu}\otimes f_{0, 1 - \e^{-\delta}}$$
uniformly over $[0, 1]$ as $n \to \infty$.
\end{restatable}
\begin{remark}
A formal proof is provided in \Cref{app:CLT}. The assumptions concerning $\{\delta_{ni}\}$ give rise to $1 - (1 - \delta_{n1})(1 - \delta_{n2}) \cdots (1 - \delta_{nn}) \to 1 - \e^{-\delta}$. In general, tensoring with $f_{0,\delta}$ is equivalent to scaling the graph of the trade-off function $f$ toward the origin by a factor of $1-\delta$. This property is specified by the following formula, and we leave its proof to \Cref{app:CLT}:
	\begin{equation}\label{prop:ruibbit}
	f\otimes f_{0,\delta}(\alpha) =
		\left\{
		\begin{array}{ll}
		(1-\delta)\cdot f(\frac{\alpha}{1-\delta}), 		& 0\leqslant \alpha \leqslant 1-\delta \\
		0, & 1-\delta\leqslant \alpha\leqslant 1.
		\end{array}
		\right.
	\end{equation}
In particular, $f\otimes f_{0,\delta}$ is symmetric if $f$ is symmetric. Note that \eqref{eq:decomposition} and \eqref{eq:delta} can be deduced by the formula above.
\end{remark}

This theorem interprets the privacy level of the composition using Gaussian and uniform distributions. Explicitly, the theorem demonstrates that, based on the released information of the composed mechanism, distinguishing between any neighboring datasets is at least as hard as distinguishing between the following two bivariate distributions:
\[
\N(0, 1)\times U[0, 1]  \text{ versus }  \N(\mu, 1)\times U[1 - \e^{-\delta}, 2 - \e^{-\delta}].
\]
We note that for small $\delta$, $\e^{-\delta}\approx 1-\delta$. So $U[1 - \e^{-\delta}, 2 - \e^{-\delta}]\approx U[\delta,1+\delta]$.

This approximation of the tensor product $f_{\ep_{n1},\delta_{n1}}\otimes \cdots \otimes f_{\ep_{nn},\delta_{nn}}$ using simple distributions is important from the viewpoint of computational complexity. Murtagh and Vadhan \cite{complexity} showed that, given a collection of $\{(\ep_i,\delta_i)\}_{i=1}^n$, finding the smallest $\ep$ such that $f_{\ep,\delta}\leqslant f_{\ep_{1},\delta_{1}}\otimes \cdots \otimes f_{\ep_{n},\delta_{n}}$ is \#P-hard\footnote{\#P is a complexity class that is ``even harder than'' NP (i.e.~a polynomial time algorithm for any \#P-hard problem would imply P=NP).  See, e.g.,  Ch. 9.~of \cite{arora2009computational}.} for any $\delta$. From the dual perspective (see \Cref{sub:a_primal_dual_connection_with_}), this negative result is equivalent to the
\#P-hardness of evaluating the convex conjugate $\big(f_{\ep_{1},\delta_{1}}\otimes \cdots \otimes f_{\ep_{n},\delta_{n}}\big)^*$ at any point. For completeness, we remark that \cite{complexity} provided an FPTAS\footnote{An approximation algorithm is called a fully polynomial-time approximation scheme (FPTAS) if its running time is polynomial in both the input size and the inverse of the relative approximation error. See, e.g., Ch.~8.~of \cite{vazirani2013approximation}.} to approximately find the smallest $\epsilon$ in $O(n^3)$ time for a \textit{single} $\delta$. In comparison, Theorem~\ref{thm:DPCLT} offers a \textit{global} approximation of the tensor product in $O(n)$ time using a closed-form expression, subsequently enabling an analytical approximation of the smallest
$\epsilon$ for each $\delta$.


\begin{figure}[!htp]
\centering
  \includegraphics[width=0.75\linewidth]{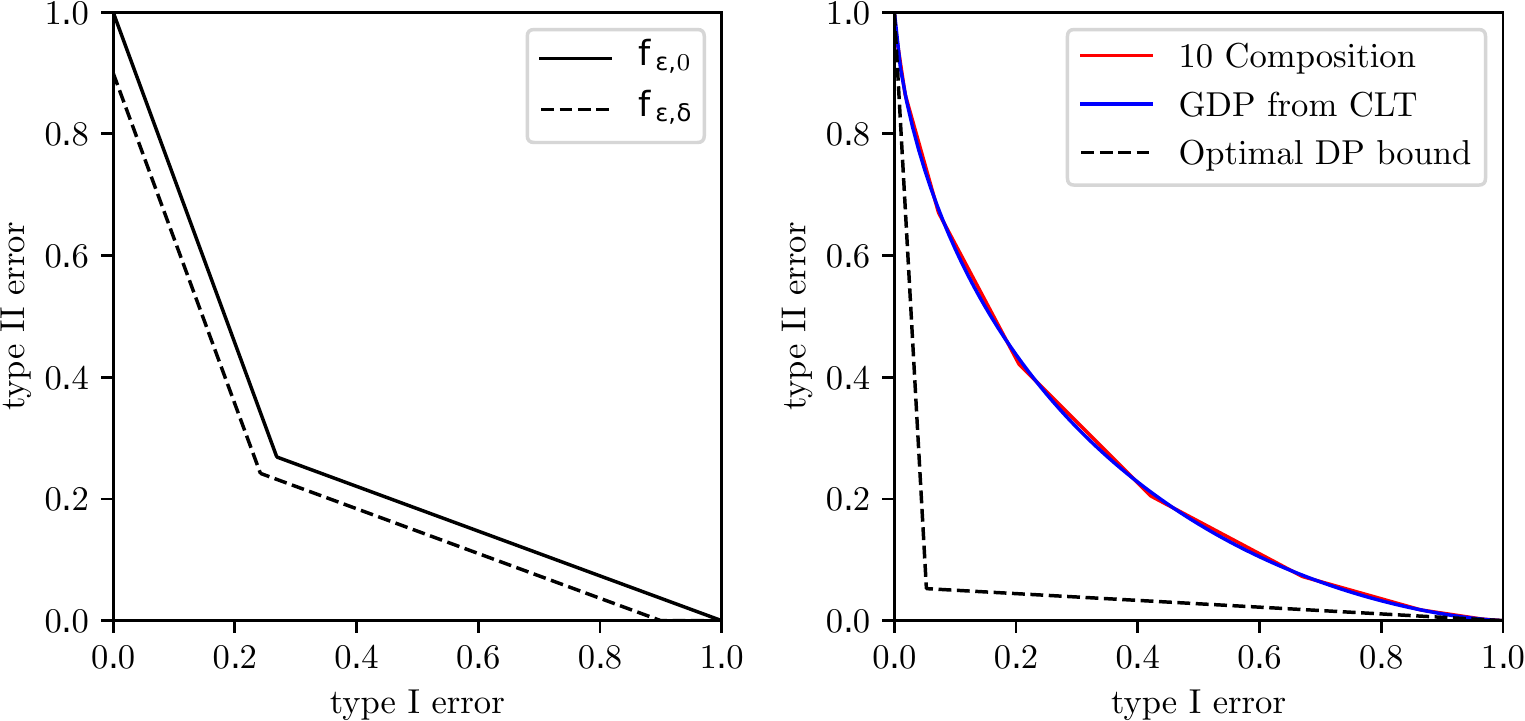}
  \captionof{figure}{Left: Tensoring with $f_{0,\delta}$ scales the graph towards the origin by a factor of $1-\delta$. Right: 10-fold composition of $(1/\sqrt{10},0)$-DP mechanisms, that is, $f_{\ep,0}^{\otimes n}$ with $n=10, \ep=1/\sqrt{n}.$ The dashed curve corresponds to $\ep=2.89,\delta = 0.001$. These values are obtained by first setting $\delta = 0.001$ and finding the smallest $\ep$ such that the composition is $(\ep,\delta)$-DP. Note that the central limit theorem approximation to the true trade-off curve is almost perfect, whereas the tightest possible approximation via $(\ep,\delta)$-DP is substantially looser.}
  \label{fig:comp}
\end{figure}

That being said, \Cref{thm:DPCLT} remains silent on the approximation error in applications with a moderately large number of $(\epsilon, \delta)$-DP mechanisms. Alternatively, we can apply \Cref{thm:Berry} to obtain a non-asymptotic normal approximation to $f_{\ep_{1},0}\otimes \cdots \otimes f_{\ep_{n},0}$ and use $\gamma$ to specify the approximation error. It can be shown that $\gamma = O(1/\sqrt{n})$ under mild conditions (\Cref{cor:e0Berry}). This bound, however, is not sharp enough for tight privacy guarantees if $n$ is not too large (note that $1/\sqrt{n} \approx 0.14$ if $n = 50$, for which exact computation is already challenging, if possible at all). Surprisingly, the following theorem establishes a $O(1/n)$ bound, thereby ``beating'' the classical Berry--Esseen bound.
\begin{theorem} \label{thm:fast}
Fix $\mu > 0$ and let $\ep = \mu/\sqrt{n}$. There is a constant $c > 0$ that only depends on $\mu$ satisfying
\[
G_{\mu}\left( \alpha+\tfrac{c}{n} \right) - \tfrac{c}{n} \leqslant f_{\ep,0}^{\otimes n}(\alpha)\leqslant G_{\mu}\left( \alpha-\tfrac{c}{n} \right) + \tfrac{c}{n}
\]
for all $n\geqslant 1$ and $c/n \le \alpha \le 1 - c/n$.
\end{theorem}


As with \Cref{thm:DPCLT}, this theorem can be extended to approximate DP ($\delta \ne 0$) by making use of the decomposition \eqref{eq:decomposition}. Our simulation studies suggest that $c \approx 0.1$ for $\mu = 1$, which is best illustrated in the right panel of \Cref{fig:comp}. Despite a fairly small $n = 10$, the difference between $G_1$ and its target $f_{\ep,0}^{\otimes n}$ is less than 0.013 in the pointwise sense. Interestingly, numerical evidence suggests the same $O(1/n)$ rate in the inhomogeneous composition provided that $\ep_1, \ldots, \ep_n$ are roughly the same size. A formal proof, or even a quantitative statement of this observation, constitutes an interesting problem for future investigation.


In closing this section, we highlight some novelties in the proof of \Cref{thm:fast}. Denoting $p_{\epsilon} = \frac{1}{1+\e^\ep}$ and $q_{\epsilon} = \frac{\e^\ep}{1+\e^\ep}$, \cite{KOV} presented a very useful expression (rephrased in our framework):
\[
f_{\ep,0}^{\otimes n} = \F\big(B(n,p_{\ep}),B(n, q_{\ep})\big),
\]
where $B(n, p)$ denotes the binomial distribution with $n$ trials and success probability $p$. However, directly approximating $f_{\ep,0}^{\otimes n}$ through these two binomial distributions is unlikely to yield 
an $O(1/n)$ bound because the Berry--Esseen bound is rate-optimal for binomial distributions. Our analysis, instead, rests crucially on a certain smoothing effect that comes for free in testing between the two distributions. It is analogous to the continuity correction for normal approximations to binomial probabilities. See the technical details in \Cref{app:CLT}.






\section{Amplifying Privacy by Subsampling}
\label{sec:subsampling}

\newcommand{\Sample}{\mathtt{Sample}}
\newcommand{\tm}{\widetilde{M}}
\newcommand{\Sn}{S\cup\{n\}}
\newcommand{\xv}{S}
%

Subsampling is often used prior to a private mechanism $M$ as a way to \textit{amplify} privacy guarantees. Specifically, we can construct a smaller dataset $\tilde S$ by flipping a fair coin for each individual in the original dataset $S$ to decide whether the individual is included in $\tilde{S}$. This subsampling scheme roughly shrinks the dataset by half and, therefore, we would expect that the induced mechanism applied to $\tilde S$ is about twice as private as the original mechanism $M$. Intuitively speaking, this privacy amplification is due to the fact that every individual enjoys perfect privacy if the individual is not included in the resulting dataset $\tilde S$, which happens with probability 50\%.

The claim above was first formalized in \cite{KLNRS} for $(\ep,\delta)$-DP. Such a privacy amplification property is, unfortunately, no longer true for the most natural previous relaxations of differential privacy aimed at recovering precise compositions (like concentrated differential privacy (CDP) \cite{concentrated,concentrated2}). Further modifications such as truncated CDP \cite{tcdp} have been introduced primarily to remedy this deficiency of CDP---but at the cost of extra complexity in the definition. Other relaxations like {\Renyi} differential privacy \cite{renyi} can be shown to satisfy a form of privacy amplification by subsampling, but both the analysis and the statement are complex \cite{wang2018subsampled}.


In this section, we show that these obstacles can be overcome by our hypothesis testing based relaxation of differential privacy. Explicitly, our main result is a simple, general, and easy-to-interpret subsampling theorem for $f$-DP. Somewhat surprisingly, our theorem significantly improves on the classical subsampling theorem for privacy amplification in the $(\ep,\delta)$-DP framework \cite{jon}. Note that this classical theorem continues to use $(\epsilon,\delta)$-DP to characterize the subsampled mechanism. However, $(\epsilon,\delta)$-DP is simply not expressive enough to capture the amplification of privacy.


\subsection{A Subsampling Theorem} 
\label{sec:subsample_theorems}

Given an integer $1 \le m \le n$ and a dataset $S$ of $n$ individuals, let $\Sample_m(\xv)$ be a subset of $S$ that is chosen uniformly at random among all the $m$-sized subsets of $\xv$. For a mechanism $M$ defined on $X^m$, we call $M\big(\Sample_m(\xv)\big)$ the subsampled mechanism, which takes as input an $n$-sized dataset. Formally, we use $M \circ \Sample_m$ to denote this subsampled mechanism. To clear up any confusion, note that intermediate result $\Sample_m(\xv)$ is not released and, in particular, this is different from the composition in \Cref{sec:composition-theorems}.

In brief, our main theorem shows that the privacy bound of the subsampled mechanism in the $f$-DP framework is given by an operator acting on trade-off functions. To introduce this operator, write the convex combination $f_p:= pf + (1-p)\Id$ for $0 \le p \le 1$, where $\Id(x) = 1-x$. Note that the trade-off function $f_p$ is asymmetric in general.
\begin{definition} \label{def:Cp}
For any $0 \le p \le 1$, define the operator $C_p$ acting on trade-off functions as
\[C_p(f) := \min\{f_p, f_p^{-1}\}^{**}.\]
We call $C_p$ the $p$-sampling operator.
\end{definition}
Above, the inverse $f_p^{-1}$ is defined in \eqref{eq:inver_f}. The biconjugate $\min\{f_p, f_p^{-1}\}^{**}$ is derived by applying the conjugate as defined in \eqref{eq:conjugate} twice to $\min\{f_p, f_p^{-1}\}$. For the moment, take for granted the fact that $C_p(f)$ is a symmetric trade-off function.


Now, we present the main theorem of this section.

\begin{theorem}\label{thm:subsample}
If $M$ is $f$-DP on $X^m$, then the subsampled mechanism $M\circ\Sample_m$ is $C_p(f)$-DP on $X^n$, where the sampling ratio $p=\frac{m}{n}$.
\end{theorem}


Appreciating this theorem calls for a better understanding of the operator $C_p$. In effect, $C_p$ performs a two-step transformation: symmetrization (taking the minimum of $f_p$ and its inverse $f_p^{-1}$) and convexification (taking the largest convex lower envelope of $\min\{f_p, f_p^{-1}\}$). The convexification step is seen from convex analysis that the biconjugate $h^{**}$ of any function $h$ is the greatest convex lower bound of $h$. As such, $C_p(f)$ is convex and, with a bit more analysis, \Cref{prop:trade-off} ensures that $C_p(f)$ is indeed a trade-off function. As an aside, $C_p(f) \le \min\{f_p, f_p^{-1}\} \le f_p$. See \Cref{fig:subsample} for a graphical illustration.


Next, the following facts concerning the $p$-sampling operator qualitatively illustrate this privacy amplification phenomenon.
\begin{enumerate}
\item If $0\leqslant p\leqslant q\leqslant 1$ and $f$ is symmetric, we have $f=C_1(f)\leqslant C_q(f)\leqslant C_p(f)\leqslant C_0(f)= \Id$. That is, as the sampling ratio declines from 1 to 0, the privacy guarantee interpolates monotonically between the original $f$ and the perfect privacy guarantee $\Id$. This monotonicity follows from the fact that $g \ge h$ is equivalent to $g^{-1} \ge h^{-1}$ for any trade-off functions $g$ and $h$.

\item
If two trade-off functions $f$ and $g$ satisfy $f \ge g$, then $C_p(f) \ge C_p(g)$. This means that if a mechanism is more private than the other, using the same sampling ratio, the subsampled mechanism of the former remains more private than that of the latter, at least in terms of lower bounds.

\item For any $0 \le p \le 1$, $C_p(\Id) = \Id$. That is, perfect privacy remains perfect privacy with subsampling.
\end{enumerate}

Explicitly, we provide a formula to calculate $C_p(f)$ for a symmetric trade-off function $f$. Letting $x^*$ be the unique fixed point of $f$, that is $f(x^*)= x^*$, we have
\begin{equation}\label{eq:Cp_expression}
		C_p(f)(x)=
			\left\{
			\begin{array}{ll}
			f_p(x),&x\in[0,x^*] \\
			x^*+f_p(x^*)-x, & x\in[x^*,f_p(x^*)]\\
			f_p^{-1}(x), &x\in[f_p(x^*),1].
			\end{array}
			\right.
	\end{equation}
This expression is almost self-evident from the left panel of \Cref{fig:subsample}. Nevertheless, a proof of this formula is given in \Cref{app:property}. This formula, together with \Cref{thm:subsample}, allows us to get a closed-form characterization of the privacy amplification for $(\epsilon, \delta)$-DP.
\begin{corollary}\label{cor:ep_d}
If $M$ is $(\epsilon, \delta)$-DP on $X^m$, then the subsampled mechanism $M\circ\Sample_m$ is $C_p(f_{\epsilon, \delta})$-DP on $X^n$, where
\begin{equation}\label{eq:ep_d_sump_for}
C_p(f_{\epsilon, \delta})(\alpha) = \max\left\{f_{\epsilon', \delta'}(\alpha), 1 - p\delta - p \, \frac{\e^{\ep} - 1}{\e^{\ep} + 1} - \alpha \right\}.
\end{equation}
Above, $\epsilon'= \log(1-p + p \e^\ep), \delta' = p\delta$, and $p = \frac{m}{n}$.
\end{corollary}


\begin{figure}[!htp]
\centering
  \includegraphics[width=.7\linewidth]
{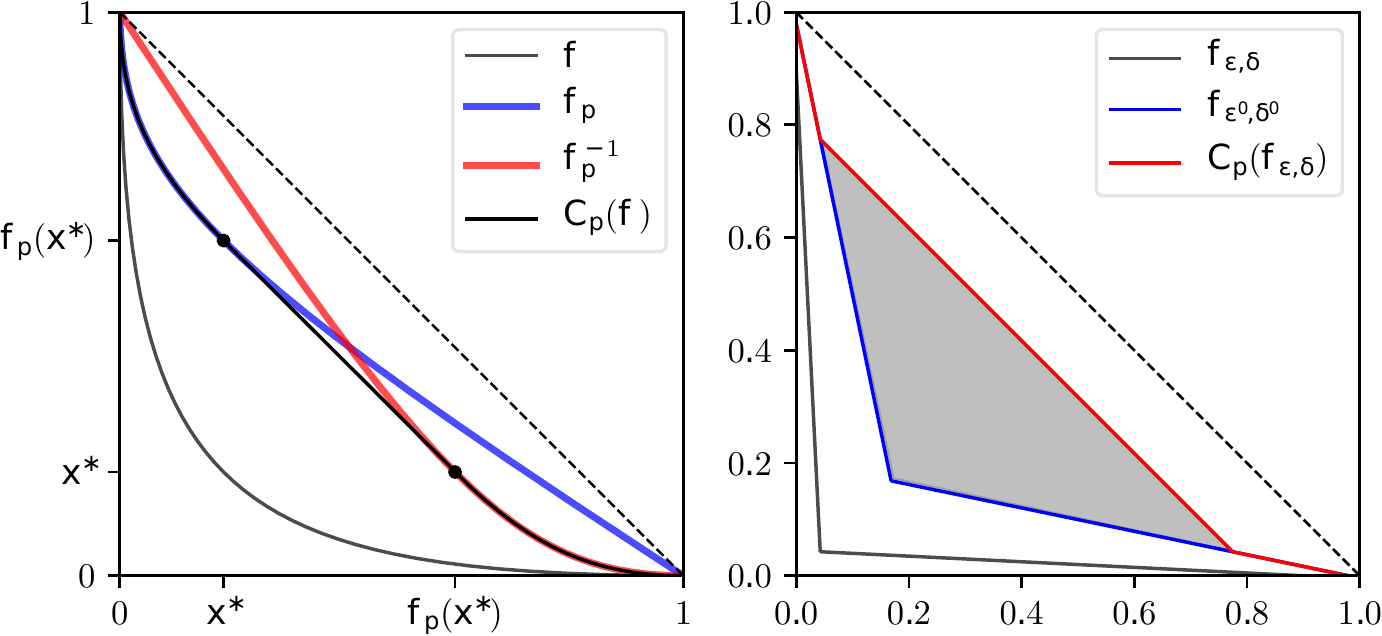}
  \captionof{figure}{The action of $C_p$. Left panel: $f=G_{1.8}, p=0.35$. Right panel: $\ep = 3,\delta=0.1,p=0.2$. The subsampling theorem \ref{thm:subsample} results in a significantly tighter trade-off function compared to the classical theorem for $(\ep,\delta)$-DP.}
  \label{fig:subsample}
\end{figure}


For comparison, we now present the existing bound on the privacy amplification by subsampling for $(\ep,\delta)$-DP. To be self-contained, \Cref{app:property} gives a proof of this result, which primarily follows \cite{jon} .

\begin{restatable}[\cite{jon}]{lemma}{DPsubsamplerep}\label{lem:DPsubsample}
If $M$ is $(\ep,\delta)$-DP, then $M\circ\Sample_m$ is $(\ep',\delta')$-DP with $\ep'$ and $\delta'$ defined in \Cref{cor:ep_d}.
\end{restatable}

Using the language of the $f$-DP framework, \Cref{lem:DPsubsample} states that $M\circ\Sample_m$ is $f_{\epsilon',\delta'}$-DP. \Cref{cor:ep_d} improves on \Cref{lem:DPsubsample} because, as is clear from \eqref{eq:ep_d_sump_for},
\[
C_p(f_{\epsilon, \delta}) \ge f_{\epsilon',\delta'}.
\]
The right panel of \Cref{fig:subsample} illustrates \Cref{lem:DPsubsample} and our \Cref{cor:ep_d} for $\epsilon = 3, \delta = 0.1$, and $p = 0.2$. In effect, the improvement is captured by the shaded triangle enclosed by $C_p(f_{\epsilon, \delta})$ and $f_{\epsilon',\delta'}$, revealing that the minimal sum of type I and type II errors in distinguishing two neighboring datasets with subsampling can be significantly lower than the prediction of \Cref{lem:DPsubsample}. This gain is only made possible by the flexibility of trade-off functions in the sense that $C_p(f_{\epsilon, \delta})$ \textit{cannot} be expressed within the $(\epsilon, \delta)$-DP framework. The unavoidable loss in the $(\epsilon,\delta)$-DP representation of the subsampled mechanism is compounded when analyzing the composition of many private mechanisms. 


In the next subsection, we prove \Cref{thm:subsample} by making use of \Cref{lem:DPsubsample}. Its proof implies that \Cref{thm:subsample} holds for any subsampling scheme for which \Cref{lem:DPsubsample} is true. In particular, it holds for the subsampling scheme described at the beginning of this section, that is, independent coin flips for every data item. 



\subsection{Proof of the Subsampling Theorem} 
\label{sub:proof_of_subsample_theorems}
The proof strategy is as follows. First, we convert the $f$-DP guarantee into an infinite collection of $(\ep,\delta)$-DP guarantees by taking a dual perspective that is enabled by \Cref{prop:ftoDP}. Next, by applying the classical subsampling theorem (that is, \Cref{lem:DPsubsample}) to these $(\ep,\delta)$-DP guarantees, we conclude that the subsampled mechanism satisfies a new infinite collection of $(\ep,\delta)$-DP guarantees. Finally, \Cref{prop:DPtof} allows us to convert these new privacy guarantees back into an $\tilde{f}$-DP guarantee, where $\tilde f$ can be shown to coincide with $C_p(f)$.


\begin{proof}[Proof of \Cref{thm:subsample}]
Provided that $M$ is $f$-DP, from \Cref{prop:ftoDP} it follows that $M$ is $\big(\ep,\delta(\ep)\big)$-DP with $\delta(\ep) = 1+f^*(- \e^{\ep})$ for all $\ep\geqslant 0$. Making use of \Cref{lem:DPsubsample}, the subsampled mechanism $M \circ \Sample_m$ satisfies the following collection of $(\ep',\delta')$-DP guarantees for all $\ep\geqslant 0$:
\[
\ep' =\log(1 - p + p\e^\ep),\quad \delta' =p\big(1+f^*(-\e^{\ep})\big).
\]
Eliminating the variable $\epsilon$ from the two parametric equations above, we can relate $\ep'$ to $\delta'$ using
\begin{equation}\label{eq:fp}
\delta' = 1+f_p^*(-\e^{\ep'}),
\end{equation}
which is proved in \Cref{app:property}. The remainder of the proof is devoted to showing that $(\epsilon', \delta')$-DP guarantees for all $\ep' \ge 0$ is equivalent to the $C_p(f)$-DP guarantee.


At first glance, \eqref{eq:fp} seems to enable the use of \Cref{prop:ftoDP}. Unfortunately, that would be invalid because $f_p$ is asymmetric. To this end, we need to extend \Cref{prop:ftoDP} to general trade-off functions. To avoid conflicting notation, let $g$ be a generic trade-off function, not necessarily symmetric. Denote by $\bar{x}$ be the smallest point such that $g'(x)=-1$, that is, $\bar{x} = \inf\{x\in[0,1]:g'(x)=-1\}$.\footnote{For simplicity, the proof assumes differentiable trade-off functions. If $g$ is not differentiable, use the definition $\bar{x} = \inf\{x\in[0,1]: -1 \in \partial g(x) \}$ instead. This adjustment applies to other parts of the proof.} As a special instance of \Cref{prop:asymm_env} in the appendix, the following result serves our purpose.
\begin{proposition}\label{prop:asymm_for_proof}
If $g(\bar{x})\geqslant\bar{x}$ and a mechanism $M$ is $(\ep,1+g^*(-\e^{\ep}))$-DP for all $\ep\geqslant 0$, then $M$ is $\min\{g,g^{-1}\}^{**}$-DP.
\end{proposition}

The proof of the present theorem would be complete if \Cref{prop:asymm_for_proof} can be applied to the collection of privacy guarantees in \eqref{eq:fp}for $f_p$. To use \Cref{prop:asymm_for_proof}, it suffices to verify the condition $f_p(\bar{x})\geqslant \bar{x}$ where $\bar{x}$ is the smallest point such that $f_p'(x)=-1$. Let $x^*$ be the (unique) fixed point of $f$. To this end, we collect a few simple facts:
	\begin{itemize}
		\item First, $f'(x^*) = -1$. This is because the graph of $f$ is symmetric with respect to the $45^\circ$ line passing through the origin.
		\item Second, $\bar{x}\leqslant x^*$. This is because $f_p'(x^*)=pf'(x^*)+(1-p)\Id'(x^*)=-1$ and, by definition, $\bar{x}$ can only be smaller.
	\end{itemize}
With these facts in place, we get
\[f_p(\bar{x})\geqslant f_p(x^*) \geqslant f(x^*) = x^* \geqslant \bar{x}\]
by recognizing that $f_p$ is decreasing and $f_p \ge f$. Hence, the proof is complete.

\end{proof}

\section{Application: Privacy Analysis of Stochastic Gradient Descent} 
\label{sec:application_in_sgd}
One of the most important algorithms in machine learning and optimization is stochastic gradient descent (SGD). This is an iterative optimization method used to train a wide variety of models, for example, deep neural networks. SGD has also served as an important benchmark in the development of private optimization: as an iterative algorithm, the tightness of its privacy analysis crucially depends on the tightness with which composition can be accounted for. The analysis also crucially requires a privacy amplification by subsampling argument. 

The first asymptotically optimal analysis of differentially private SGD was given by \cite{bassily2014private}. Because of the inherent limits of $(\epsilon,\delta)$-DP, however, this original analysis did not give meaningful privacy bounds for realistically sized datasets. This is in part what motivated the development of divergence based relaxations of differential privacy. Unfortunately, these relaxations cannot be directly applied to the analysis of SGD due to the lack of a privacy amplification by subsampling theorem. In response, Abadi et al.~\cite{deep} circumvented this challenge by developing the moments accountant---a numeric technique tailored specifically to repeated application of subsampling, followed by a Gaussian mechanism---to give privacy bounds for SGD that are strong enough to give non-trivial guarantees when training deep neural networks on real datasets. But this analysis is ad-hoc in the sense that it uses a tool designed specifically for the analysis of SGD. 

In this section, we use the general tools we have developed so far to give a simple and improved analysis of the privacy of SGD. In particular, the analysis rests crucially on the compositional and subsampling properties of $f$-DP.

\subsection{Stochastic Gradient Descent and Its Privacy Analysis} 
\label{sub:noisy_sgd}
The private variant of the SGD algorithm is described in \Cref{alg:dpsgd}. As we will see, from the perspective of its privacy analysis, it can simply be viewed as a repeated composition of Gaussian mechanisms operating on subsampled datasets.
\begin{algorithm}[htb]
	\caption{\texttt{NoisySGD}}\label{alg:dpsgd}
	\begin{algorithmic}[1]
	\State{\bf Input:} Dataset $S = (x_1,\ldots,x_n)$, loss function $L(\theta, x)$.
	\Statex \hspace{1.35cm}Parameters: initial state $\theta_0$, learning rate $\eta_t$, batch size $m$, time horizon $T$,
	\Statex \hspace{3.55cm}noise scale $\sigma$, gradient norm bound $C$.
		\For{$t = 1, \ldots, T$}
		\State {\bf Subsampling:}
		\Statex {\hspace{0.55cm}Take a uniformly random subsample $I_t\subseteq \{1, \ldots, n\}$ with batch size $m$}
		\Comment{$\Sample_m$ in \Cref{sec:subsampling}}
		\For{$i\in I_t$}
			\State {\bf Compute gradient:}
			\Statex \hspace{1.1cm} {$v_t^{(i)} \gets \nabla_{\theta} L(\theta_t, x_i)$}		
			\State {\bf Clip gradient:}
			\Statex \hspace{1.1cm} {$\bar{v}_t^{(i)} \gets v_t^{(i)} / \max\big\{1, \|v_t^{(i)}\|_2/C\big\}$}
		\EndFor
		\State {\bf Average, perturb, and descend:}
		\Statex \hspace{0.55cm} {$\theta_{t+1} \gets \theta_{t} - \eta_t \Big(\frac{1}{m} \sum_i\bar{v}_t^{(i)}+\N(0, \frac{4\sigma^2 C^2}{m^2} I)\Big)$}
                 \Comment{$I$ is an identity matrix}
		\EndFor
		\State {\bf Output} $\theta_T$
	\end{algorithmic}
\end{algorithm}


To analyze the privacy of \texttt{NoisySGD}, we start by building up the privacy properties from the inner loop. Let $V$ be the vector space where parameter $\theta$ lives in and $M:X^m\times V\to V$ be the mechanism that executes lines 4-7 in \Cref{alg:dpsgd}. Here $m$ denotes the batch size. In effect, what $M$ does in iteration $t$ can be expressed as
$$M(S_{I_t},\theta_t) = \theta_{t+1},$$
where $S_{I_t}$ is the subset of the dataset $S$ indexed by $I_t$. Next, we turn to the analysis of the subsampling step (line 3) and use $\tm$ to denote its composition with $M$, that is, $\tm = M\circ \Sample_m $. Taken together, $\tm$ executes lines 3-7 and maps from $X^n\times V$ to $V$.

The mechanism we are ultimately interested in
\begin{align*}
	\mathrm{\texttt{NoisySGD}}:X^n&\to V\times V\times\cdots\times V\\
	S&\mapsto(\theta_1,\theta_2,\ldots, \theta_T)
\end{align*}
is simply the composition of $T$ copies of $\tm$. To see this fact, note that the trajectory $(\theta_1,\theta_2,\ldots, \theta_T)$ is obtained by iteratively running
\[
\theta_{j+1} = \tm(S,\theta_j)
\]
for $j = 0, \ldots, T-1$. Let $M$ be $f$-DP. Straightforwardly, $\tm$ is $C_{m/n}(f)$-DP by \Cref{thm:subsample}. Then, from the composition theorem (\Cref{thm:n_steps}), we can immediately prove that \texttt{NoisySGD} is $C_{m/n}(f)^{\otimes T}$-DP.

Hence, it suffices to give a bound on the privacy of $M$. For simplicity, we now focus on a single step and drop the subscript $t$. Recognizing that changing one of the $m$ data points only affects one $v^{(i)}$, the sensitivity of $\frac{1}{m} \sum_i\bar{v}_t^{(i)}$ is at most $\frac{2C}{m}$ due to the clipping operation. Making use of \Cref{thm:g_mech}, adding Gaussian noise $N(0, \sigma^2 \cdot\frac{4C^2}{m^2} I)$ to the average gradient renders this step $\frac1{\sigma}$-GDP. Since that the gradient update following the gradient averaging step is deterministic, we conclude that $M$ satisfies $\frac1{\sigma}$-GDP.


In summary, the discussion above has proved the following theorem:
\begin{theorem} \label{thm:sgdcompo}
\Cref{alg:dpsgd} is $C_{m/n}(G_{\sigma^{-1}})^{\otimes T}$-DP.
\end{theorem}
To clear up any confusion, we remark that this $C_{m/n}(G_{\sigma^{-1}})^{\otimes T}$-DP mechanism does not release the subsampled indices.

The use of \Cref{thm:sgdcompo} requires the evaluation of $C_{m/n}(G_{\sigma^{-1}})^{\otimes T}$. However, numerical computation of this tensor product is computationally cumbersome. As a matter of fact, the moment's accountant technique applied to the present problem is basically equivalent to direct computing $C_{m/n}(G_{\sigma^{-1}})^{\otimes T}$. In contrast, our central limit theorems provide an entirely different tool by analytically approximating $C_{m/n}(G_{\sigma^{-1}})^{\otimes T}$ in a way that becomes nearly exact as $T$ grows. The next two subsections presents two such results, corresponding to our two central limit theorems (\Cref{thm:Berry} and \Cref{thm:CLT}), respectively. An asymptotic privacy analysis of \texttt{NoisySGD} is given in \Cref{sub:asymptotic_analysis_of_noisy_sgd} by developing a general limit theorem for composition of subsampled mechanisms. A Berry--Esseen type analysis is shown in \Cref{sub:berry_esseen_for_privacy_of_sgd}.

\begin{figure}[!tp]
\centering
  \includegraphics[width=\linewidth]
{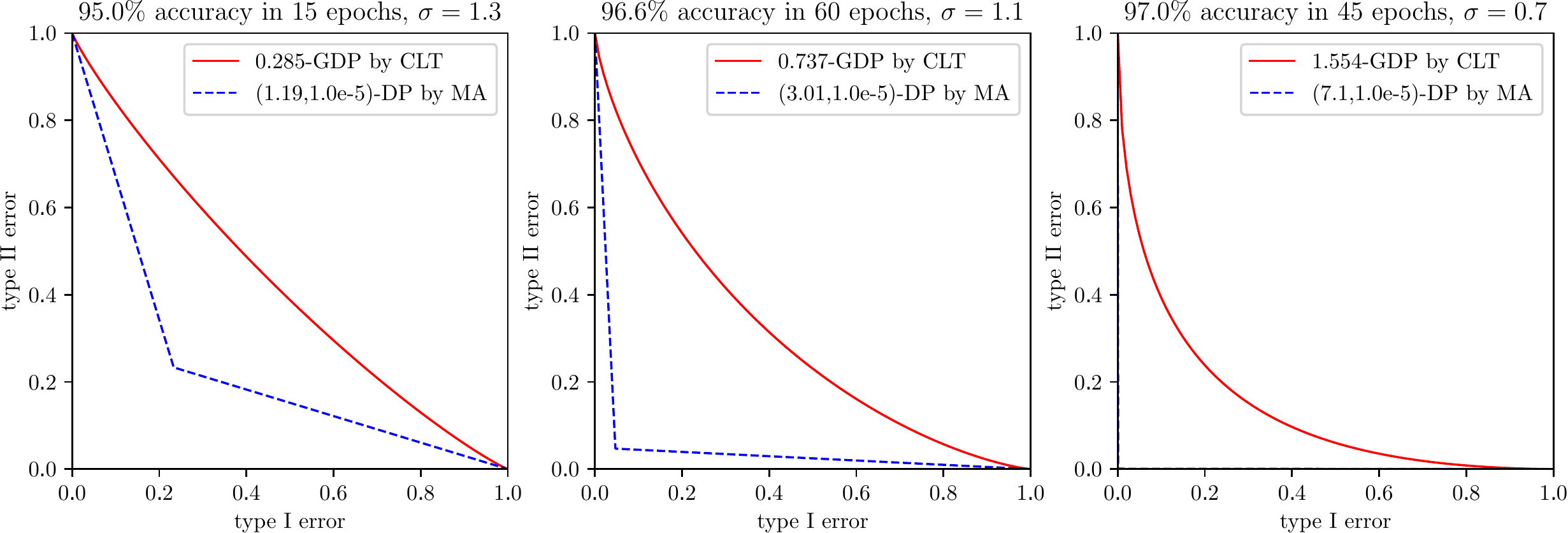}
  \captionof{figure}{Comparison of the GDP bounds derived from our method, and the $(\ep,\delta)$-DP bounds derived using the moments accountant \cite{deep}. All three experiments run \Cref{alg:dpsgd} on the entire MNIST dataset with $n=60,000$ data points, batch size $m=256$, learning rates $\eta_t$ set to 0.25, 0.15, and 0.25, respectively, and clipping thresholds $C$ set to 1.5, 1.0, 1.5, respectively. The red lines are obtained via \Cref{thm:SGDlimit}, while the blue dashed lines are produced by the tensorflow/privacy library. See \url{https://github.com/tensorflow/privacy} for the detail of the experiments.}
  \label{fig:compareSGD}
\end{figure}


\subsection{Asymptotic Privacy Analysis} 
\label{sub:asymptotic_analysis_of_noisy_sgd}
In this subsection, we first consider the limit of $C_p(f)^{\otimes T}$ for a general trade-off function $f$, then plug in $f = G_{\sigma^{-1}}$ for the analysis of \texttt{NoisySGD}. The more general approach is useful for analyzing other iterative algorithms. 

Recall from \Cref{sec:subsampling} that a $p$-subsampled $f$-DP mechanism is $C_p(f)$-DP, where $C_p(f)$ is defined as
\[
	C_p(f)(x)=
		\left\{
		\begin{array}{ll}
		f_p(x),&x\in[0,x^*] \\
		x^*+f_p(x^*)-x, & x\in[x^*,f_p(x^*)]\\
		f_p^{-1}(x), &x\in[f_p(x^*),1],
		\end{array}
		\right.
\]
where $x^*$ is the unique fixed point of $f$. We will let the sampling fraction $p$ tend to 0 as $T$ approaches infinity.
In the following, $a_+^2$ is a short-hand for $(\max\{a,0\})^2$.
\begin{restatable}{theorem}{mixturerep} \label{thm:mixtureSGD}
	Suppose $f$ is a symmetric trade-off function such that $f(0)=1$ and $\int_0^1(f'(x)+1)^4\diff x<+\infty$. Furthermore, assume $p\sqrt{T} \to p_0$ as $T\to\infty$ for some constant $p_0 > 0$. Then we have the uniform convergence
	$$C_p(f)^{\otimes T}\to G_{p_0\sqrt{2\chi^2_+(f)}}$$
as $T \to \infty$, where
	\[\chi^2_+(f)=\int_0^1\big(|f'(x)|-1\big)_+^2\diff x.\]
\end{restatable}
This theorem has implications for the design of iterative private mechanisms involving subsampling as a subroutine. One way to bound the privacy of such a mechanism is to let the sampling ratio $p$ go to zero as the total number of iterations $T$ goes to infinity. The theorem says that the correct scaling between the two values is $p\sim 1/\sqrt{T}$ and, furthermore, gives an explicit form of the limit.



In order to analyze \texttt{NoisySGD}, we need to compute the quantity $\chi^2_+(G_\mu)$. This can be done by directly working with its definition. In \Cref{app:SGD}, we provide a different approach by relating $\chi^2_+(f)$ to $\chi^2$-divergence.
\begin{restatable}{lemma}{chirep} \label{lem:chi2GDP}
We have	\[\chi^2_+(G_\mu) = \e^{\mu^2}\cdot\Phi(3\mu/2)+3\Phi(-\mu/2)-2.\]
\end{restatable}
When using SGD to train large models, we typically perform a very large number of iterations, so it is reasonable to consider the parameter regime in which $n\to\infty,T\to\infty$. The batch size can also vary with these quantities. The following theorem is a direct consequence of \Cref{thm:sgdcompo,thm:mixtureSGD} and \Cref{lem:chi2GDP}.
\begin{restatable}{corollary}{SGDlimitrep} \label{thm:SGDlimit}
	If $m\sqrt{T}/n \to c$, then \textup{\texttt{NoisySGD}} is asymptotically $\mu$-GDP with
	\[\mu = \sqrt{2}c\cdot\sqrt{\e^{\sigma^{-2}}\cdot\Phi(1.5\sigma^{-1})+3\Phi(-0.5\sigma^{-1})-2}.\]
\end{restatable}
First, we remark that the condition in the theorem is consistent with the analysis of private SGD in \cite{bassily2014private}, which considers $m = 1$ and $T = O(n^2)$. We also note that in the deep learning literature, the quantity $\frac{m}{n}\cdot\sqrt{T}$ is generally quite small. The convention in this literature is to reparameterize the number of gradient steps $T$ by the number of ``epochs'' $E$, which is the number of sweeps of the entire dataset. The relationship between these parameters is that  $E = Tm/n$. In this reparameterization, our assumption is that $Em/n\to c^2$. Concretely, the AlexNet \cite{krizhevsky2012imagenet} sets the parameters as $m=128, E\approx90$ on the ILSVRC-2010 dataset with $n\approx 1.2\times 10^6$, leading to $Em/n < 0.01$. Many other prominent implementations\footnote{See the webpage of Gluon CV Toolkit \cite{he2018bag,zhang2019bag} for a collection of such hyperparameters on computer vision
  tasks.} also lead to a small value of $Em/n$. 


\subsection{A Berry--Esseen Privacy Bound} 
\label{sub:berry_esseen_for_privacy_of_sgd}
Now, we apply the Berry--Esseen style central limit theorem (\Cref{thm:Berry}) for the privacy analysis of \texttt{NoisySGD}, with the advantage of giving sharp privacy guarantees. However, the disadvantage is that the expressions it yields are more unwieldy: they are computer evaluable, so usable in implementations, but do not admit simple closed forms. 

The individual components in \Cref{thm:Berry} have the form $C_p(G_\mu)$ with $p = m/n, \mu = \sigma^{-1}$. It suffices to evaluate the moment functionals on $C_p(G_\mu)$. This is done in the following lemma.
\begin{restatable}{lemma}{functionalGmurep}\label{lem:functionalGmu}
	Let $Z(x) = \log(p\cdot \e^{\mu x-\mu^2/2}+1-p)$ and $\varphi(x) = \frac{1}{\sqrt{2\pi}}\e^{-x^2/2}$ be the density of the standard normal distribution. Then
	\begin{align*}
		\kl\big(C_p(G_\mu)\big) &= p\int_{\mu/2}^{+\infty} Z(x)\cdot\big(\varphi(x-\mu)-\varphi(x)\big)\diff x\\
		\kappa_2\big(C_p(G_\mu)\big)&= \int_{\mu/2}^{+\infty} Z^2(x)\cdot\big(p\varphi(x-\mu)+(2-p)\varphi(x)\big)\diff x\\
		\bar{\kappa}_3\big(C_p(G_\mu)\big)&=\int_{\mu/2}^{+\infty} \big|Z(x)-\kl\big(C_p(G_\mu)\big)\big|^3\cdot(p\varphi(x-\mu)+(1-p)\varphi(x))\diff x\\
		&+\int_{\mu/2}^{+\infty} \big|Z(x)+\kl\big(C_p(G_\mu)\big)\big|^3\cdot\varphi(x)\diff x.
	\end{align*}
\end{restatable}
We can plug these expressions into \Cref{thm:Berry} and get
\begin{restatable}{corollary}{SGDBerryrep} \label{thm:SGDBerry}
Let $p = m/n, \mu = \sigma^{-1}$ and
	\begin{align*}
		\tilde{\mu}=\frac{2\sqrt{T}\cdot\kl\big(C_p(G_\mu)\big)}{\sqrt{\kappa_2\big(C_p(G_\mu)\big) - \kl^2\big(C_p(G_\mu)\big)}}, \quad		\gamma=\frac{0.56}{\sqrt{T}}\cdot \frac{\bar{\kappa}_3\big(C_p(G_\mu)\big)}{\big(\kappa_2\big(C_p(G_\mu)\big) - \kl^2\big(C_p(G_\mu)\big)\big)^{\frac{3}{2}}}.
	\end{align*}
\textup{\texttt{NoisySGD}} is $f$-DP with
	\begin{equation*}
		f(\alpha) = \max\{G_{\tilde{\mu}}(\alpha+\gamma)-\gamma,0\}.
	\end{equation*}
\end{restatable}
We remark that $G_{\tilde{\mu}}$ can be set to 0 in $(1,+\infty)$ so that $f$ is well-defined when $\alpha>1-\gamma$.



\section{Discussion}
\label{sec:discussion}

In this paper, we have introduced a new framework for private data analysis that we refer to as $f$-differential privacy, which generalizes $(\ep, \delta)$-DP and has a number of attractive properties that escape the difficulties of prior work. This new privacy definition uses trade-off functions of hypothesis testing as a measure of indistinguishability of two neighboring datasets rather than a few parameters as in prior differential privacy relaxations. Our $f$-DP retains an interpretable hypothesis testing semantics and is expressive enough to losslessly reason about composition, post-processing, and group privacy by virtue of the informativeness of trade-off functions. Moreover, $f$-DP admits a central limit theorem that identifies a simple and single-parameter family of privacy
definitions as focal: Gaussian differential privacy. Precisely, all hypothesis testing based definitions of privacy converge to Gaussian differential privacy in the limit under composition, which implies that Gaussian differential privacy is the unique such definition that can tightly handle composition. The central limit theorem and its Berry--Esseen variant give a tractable analytical approach to tightly analyzing the privacy cost of iterative methods such as SGD. Notably, $f$-DP is \emph{dual} to $(\ep, \delta)$-DP in a constructive sense, which gives the ability to import results proven for $(\epsilon,\delta)$-DP. This powerful perspective allows us to obtain an easy-to-use privacy amplification by subsampling theorem for $f$-DP, which in particular significantly improves on the
state-of-the-art counterpart in the $(\ep, \delta)$-DP setting.

We see several promising directions for future work using and extending the $f$-DP framework. First, \Cref{thm:fast} can possibly be extended to the inhomogeneous case where trade-off functions are different from each other in the composition. Such an extension would allow us to apply the central limit theorem for privacy approximation with strong finite-sample guarantees to a broader range of problems. Second, it would be of interest to investigate whether the privacy guarantee of the subsampled mechanism in \Cref{thm:subsample} can be improved for some trade-off functions. Notably, we have shown in \Cref{app:property} that this bound is tight if the trade-off function $f = 0$, that is, the original mechanism is blatantly non-private. Third, the notion of $f$-DP naturally has a \textit{local} realization where the obfuscation of the sensitive information is applied at the individual record level. In this setting, what are the fundamental limits of estimation with local $f$-DP
guarantees \cite{duchi2018minimax}? In light of \cite{duchi2018right}, what is the correct complexity measure in local $f$-DP estimation? If it is not the Fisher information, can we identify an alternative to the Fisher information for some class of trade-off functions? Moreover, we recognize that an adversary in differentially private learning may set different pairs of target type I and type II errors. For example, an adversary that attempts to control type I and II errors at 10\% and 10\%, respectively, can behave very differently from one who aims to control the two errors at 0.1\% and 99\%, respectively. An important question is to address the trade-offs between resources such as privacy and statistical efficiency and target type I and type II errors in the framework of $f$-DP. 

Finally, we wish to remark that $f$-DP can possibly offer a mathematically tractable and flexible framework for minimax estimation under privacy constraints
(see, for example, \cite{cai2019cost,bun2018fingerprinting,dwork2015robust}). Concretely, given a candidate estimator satisfying $(\epsilon, \delta)$-DP appearing in the upper bound and a possibly loose lower bound under the $(\ep, \delta)$-DP constraint, we can replace the $(\epsilon, \delta)$-DP constraint by the $f$-DP constraint where $f$ is the tightest trade-off function characterizing the estimation procedure. As is clear, the $f$-DP constraint is more stringent than the $(\ep, \delta)$-DP constraint by recognizing the primal-dual conversion (see \Cref{prop:ftoDP}). While the upper bound remains the same as the estimator continues to satisfy the new privacy constraint, the lower bound can be possibly improved due to a more stringent constraint. It would be of great interest to investigate to what extent this $f$-DP based approach can reduce the gap between upper and lower bounds minimax estimation under privacy constraints.

Ultimately, the test of a privacy definition lies not just in its power and semantics, but also in its ability to usefully analyze diverse algorithms. In this paper, we have given convincing evidence that $f$-DP is up to the task. We leave the practical evaluation of this new privacy definition to future work.




\subsection*{Acknowledgements}
\label{sec:acknowledgements}

This work was supported in part by NSF grants CCF-1763314 and CNS-1253345, and by the Sloan Foundation, and by a sub-contract for the DARPA Brandeis program.


{\small
\bibliographystyle{alpha}
\bibliography{privacy}

\newcommand{\etalchar}[1]{$^{#1}$}
\begin{thebibliography}{DKM{\etalchar{+}}06}

\bibitem[AB09]{arora2009computational}
Sanjeev Arora and Boaz Barak.
\newblock {\em Computational complexity: a modern approach}.
\newblock Cambridge University Press, 2009.

\bibitem[Abo18]{census}
John~M Abowd.
\newblock The {US} {C}ensus {B}ureau adopts differential privacy.
\newblock In {\em Proceedings of the 24th ACM SIGKDD International Conference
  on Knowledge Discovery \& Data Mining}, pages 2867--2867. ACM, 2018.

\bibitem[ACG{\etalchar{+}}16]{deep}
Martin Abadi, Andy Chu, Ian Goodfellow, H~Brendan McMahan, Ilya Mironov, Kunal
  Talwar, and Li~Zhang.
\newblock Deep learning with differential privacy.
\newblock In {\em Proceedings of the 2016 ACM SIGSAC Conference on Computer and
  Communications Security}, pages 308--318. ACM, 2016.

\bibitem[App17]{apple}
Differential Privacy~Team Apple.
\newblock Learning with privacy at scale.
\newblock Technical report, Apple, 2017.

\bibitem[BBG18]{balle2018privacy}
Borja Balle, Gilles Barthe, and Marco Gaboardi.
\newblock Privacy amplification by subsampling: Tight analyses via couplings
  and divergences.
\newblock In {\em Advances in Neural Information Processing Systems}, pages
  6280--6290, 2018.

\bibitem[BDRS18]{tcdp}
Mark Bun, Cynthia Dwork, Guy~N Rothblum, and Thomas Steinke.
\newblock Composable and versatile privacy via truncated cdp.
\newblock In {\em Proceedings of the 50th Annual ACM SIGACT Symposium on Theory
  of Computing}, pages 74--86. ACM, 2018.

\bibitem[Bla50]{blackwell1950comparison}
David Blackwell.
\newblock Comparison of experiments.
\newblock Technical report, HOWARD UNIVERSITY Washington United States, 1950.

\bibitem[BS16]{concentrated2}
Mark Bun and Thomas Steinke.
\newblock Concentrated differential privacy: Simplifications, extensions, and
  lower bounds.
\newblock In {\em Theory of Cryptography Conference}, pages 635--658. Springer,
  2016.

\bibitem[BST14]{bassily2014private}
Raef Bassily, Adam Smith, and Abhradeep Thakurta.
\newblock Private empirical risk minimization: Efficient algorithms and tight
  error bounds.
\newblock In {\em 2014 IEEE 55th Annual Symposium on Foundations of Computer
  Science}, pages 464--473. IEEE, 2014.

\bibitem[BUV18]{bun2018fingerprinting}
Mark Bun, Jonathan Ullman, and Salil Vadhan.
\newblock Fingerprinting codes and the price of approximate differential
  privacy.
\newblock {\em SIAM Journal on Computing}, 47(5):1888--1938, 2018.

\bibitem[BW18]{balle2018improving}
Borja Balle and Yu-Xiang Wang.
\newblock Improving the {G}aussian mechanism for differential privacy:
  Analytical calibration and optimal denoising.
\newblock {\em arXiv preprint arXiv:1805.06530}, 2018.

\bibitem[BZ06]{aol}
Michael Barbaro and Tom Zeller.
\newblock A face is exposed for {AOL} searcher no. 4417749.
\newblock {\em The New York Times}, August 2006.
\newblock
  \url{http://select.nytimes.com/gst/abstract.html?res=F10612FC345B0C7A8CDDA10894DE404482}.

\bibitem[CWZ19]{cai2019cost}
T~Tony Cai, Yichen Wang, and Linjun Zhang.
\newblock The cost of privacy: Optimal rates of convergence for parameter
  estimation with differential privacy.
\newblock {\em arXiv preprint arXiv:1902.04495}, 2019.

\bibitem[DJW18]{duchi2018minimax}
John~C Duchi, Michael~I Jordan, and Martin~J Wainwright.
\newblock Minimax optimal procedures for locally private estimation.
\newblock {\em Journal of the American Statistical Association},
  113(521):182--201, 2018.

\bibitem[DKM{\etalchar{+}}06]{approxdp}
Cynthia Dwork, Krishnaram Kenthapadi, Frank McSherry, Ilya Mironov, and Moni
  Naor.
\newblock Our data, ourselves: Privacy via distributed noise generation.
\newblock In {\em Annual International Conference on the Theory and
  Applications of Cryptographic Techniques}, pages 486--503. Springer, 2006.

\bibitem[DKY17]{microsoft}
Bolin Ding, Janardhan Kulkarni, and Sergey Yekhanin.
\newblock Collecting telemetry data privately.
\newblock In {\em Proceedings of Advances in Neural Information Processing
  Systems 30 (NIPS 2017)}, 2017.

\bibitem[DMNS06]{DMNS06}
Cynthia Dwork, Frank McSherry, Kobbi Nissim, and Adam Smith.
\newblock Calibrating noise to sensitivity in private data analysis.
\newblock In {\em Theory of cryptography conference}, pages 265--284. Springer,
  2006.

\bibitem[DR16]{concentrated}
Cynthia Dwork and Guy~N Rothblum.
\newblock Concentrated differential privacy.
\newblock {\em arXiv preprint arXiv:1603.01887}, 2016.

\bibitem[DR18]{duchi2018right}
John~C Duchi and Feng Ruan.
\newblock The right complexity measure in locally private estimation: It is not
  the fisher information.
\newblock {\em arXiv preprint arXiv:1806.05756}, 2018.

\bibitem[DRV10]{boosting}
Cynthia Dwork, Guy~N Rothblum, and Salil Vadhan.
\newblock Boosting and differential privacy.
\newblock In {\em Foundations of Computer Science (FOCS), 2010 51st Annual IEEE
  Symposium on}, pages 51--60. IEEE, 2010.

\bibitem[DSS{\etalchar{+}}15]{dwork2015robust}
Cynthia Dwork, Adam Smith, Thomas Steinke, Jonathan Ullman, and Salil Vadhan.
\newblock Robust traceability from trace amounts.
\newblock In {\em 2015 IEEE 56th Annual Symposium on Foundations of Computer
  Science}, pages 650--669. IEEE, 2015.

\bibitem[Dur19]{durrett2019probability}
Rick Durrett.
\newblock {\em Probability: theory and examples}, volume~49.
\newblock Cambridge university press, 2019.

\bibitem[E{\etalchar{+}}85]{eguchi1985differential}
Shinto Eguchi et~al.
\newblock A differential geometric approach to statistical inference on the
  basis of contrast functionals.
\newblock {\em Hiroshima mathematical journal}, 15(2):341--391, 1985.

\bibitem[EPK14]{rappor}
{\'U}lfar Erlingsson, Vasyl Pihur, and Aleksandra Korolova.
\newblock Rappor: Randomized aggregatable privacy-preserving ordinal response.
\newblock In {\em Proceedings of the 2014 ACM SIGSAC conference on computer and
  communications security}, pages 1054--1067. ACM, 2014.

\bibitem[Hol19]{evil}
Susan Holmes.
\newblock Challenges in the analyses of multidomain longitudinal data: Layered
  solutions, 2019.
\newblock 3rd Workshop on Statistical and Algorithmic Challenges in Microbiome
  Data Analysis.

\bibitem[HSR{\etalchar{+}}08]{gwas}
N.~Homer, S.~Szelinger, M.~Redman, D.~Duggan, W.~Tembe, J.~Muehling, J.V.
  Pearson, D.A. Stephan, S.F. Nelson, and D.W. Craig.
\newblock Resolving individuals contributing trace amounts of dna to highly
  complex mixtures using high-density snp genotyping microarrays.
\newblock {\em PLoS Genetics}, 4(8):e1000167, 2008.

\bibitem[HV11]{harremoes2011pairs}
Peter Harremo{\"e}s and Igor Vajda.
\newblock On pairs of $ f $-divergences and their joint range.
\newblock {\em IEEE Transactions on Information Theory}, 57(6):3230--3235,
  2011.

\bibitem[HZZ{\etalchar{+}}18]{he2018bag}
Tong He, Zhi Zhang, Hang Zhang, Zhongyue Zhang, Junyuan Xie, and Mu~Li.
\newblock Bag of tricks for image classification with convolutional neural
  networks.
\newblock {\em arXiv preprint arXiv:1812.01187}, 2018.

\bibitem[KLN{\etalchar{+}}11]{KLNRS}
Shiva~Prasad Kasiviswanathan, Homin~K Lee, Kobbi Nissim, Sofya Raskhodnikova,
  and Adam Smith.
\newblock What can we learn privately?
\newblock {\em SIAM Journal on Computing}, 40(3):793--826, 2011.

\bibitem[KOV17]{KOV}
Peter Kairouz, Sewoong Oh, and Pramod Viswanath.
\newblock The composition theorem for differential privacy.
\newblock {\em IEEE Transactions on Information Theory}, 63(6):4037--4049,
  2017.

\bibitem[KSH12]{krizhevsky2012imagenet}
Alex Krizhevsky, Ilya Sutskever, and Geoffrey~E Hinton.
\newblock Imagenet classification with deep convolutional neural networks.
\newblock In {\em Advances in neural information processing systems}, pages
  1097--1105, 2012.

\bibitem[LC10]{lecun-mnisthandwrittendigit-2010}
Yann LeCun and Corinna Cortes.
\newblock {MNIST} handwritten digit database.
\newblock 2010.

\bibitem[Leh04]{lehmann2004elements}
Erich~Leo Lehmann.
\newblock {\em Elements of large-sample theory}.
\newblock Springer Science \& Business Media, 2004.

\bibitem[LR06]{lehmann2006testing}
Erich~L Lehmann and Joseph~P Romano.
\newblock {\em Testing statistical hypotheses}.
\newblock Springer Science \& Business Media, 2006.

\bibitem[LV06]{liese2006divergences}
Friedrich Liese and Igor Vajda.
\newblock On divergences and informations in statistics and information theory.
\newblock {\em IEEE Transactions on Information Theory}, 52(10):4394--4412,
  2006.

\bibitem[Mir17]{renyi}
Ilya Mironov.
\newblock R{\'e}nyi differential privacy.
\newblock In {\em 2017 IEEE 30th Computer Security Foundations Symposium
  (CSF)}, pages 263--275. IEEE, 2017.

\bibitem[MV16]{complexity}
Jack Murtagh and Salil Vadhan.
\newblock The complexity of computing the optimal composition of differential
  privacy.
\newblock In {\em Theory of Cryptography Conference}, pages 157--175. Springer,
  2016.

\bibitem[NS08]{netflix}
Arvind Narayanan and Vitaly Shmatikov.
\newblock Robust de-anonymization of large sparse datasets.
\newblock In {\em 2008 ieee symposium on security and privacy}, pages 111--125.
  IEEE, 2008.

\bibitem[P{\'o}l20]{polya1920zentralen}
Georg P{\'o}lya.
\newblock {\"U}ber den zentralen grenzwertsatz der wahrscheinlichkeitsrechnung
  und das momentenproblem.
\newblock {\em Mathematische Zeitschrift}, 8(3-4):171--181, 1920.

\bibitem[PW14]{itlectures}
Yury Polyanskiy and Yihong Wu.
\newblock Lecture notes on information theory.
\newblock {\em Lecture Notes for ECE563 (UIUC) and}, 6:2012--2016, 2014.

\bibitem[Rag11]{raginsky2011shannon}
Maxim Raginsky.
\newblock Shannon meets blackwell and le cam: Channels, codes, and statistical
  experiments.
\newblock In {\em 2011 IEEE International Symposium on Information Theory
  Proceedings}, pages 1220--1224. IEEE, 2011.

\bibitem[She10]{shevtsova2010improvement}
IG~Shevtsova.
\newblock An improvement of convergence rate estimates in the lyapunov theorem.
\newblock In {\em Doklady Mathematics}, volume~82, pages 862--864. Springer,
  2010.

\bibitem[SMM18]{sommer2018privacy}
David Sommer, Sebastian Meiser, and Esfandiar Mohammadi.
\newblock Privacy loss classes: The central limit theorem in differential
  privacy.
\newblock 2018.

\bibitem[Ull17]{jon}
Jonathan Ullman.
\newblock Cs7880: Rigorous approaches to data privacy, spring 2017.
\newblock 2017.
\newblock \url{http://www.ccs.neu.edu/home/jullman/PrivacyS17/HW1sol.pdf}.

\bibitem[Usp37]{uspensky1937introduction}
James~Victor Uspensky.
\newblock Introduction to mathematical probability.
\newblock 1937.

\bibitem[Vaz13]{vazirani2013approximation}
Vijay~V Vazirani.
\newblock {\em Approximation algorithms}.
\newblock Springer Science \& Business Media, 2013.

\bibitem[WBK18]{wang2018subsampled}
Yu-Xiang Wang, Borja Balle, and Shiva Kasiviswanathan.
\newblock Subsampled r$\backslash$'enyi differential privacy and analytical
  moments accountant.
\newblock {\em arXiv preprint arXiv:1808.00087}, 2018.

\bibitem[WZ10]{wasserman_zhou}
Larry Wasserman and Shuheng Zhou.
\newblock A statistical framework for differential privacy.
\newblock {\em Journal of the American Statistical Association},
  105(489):375--389, 2010.

\bibitem[ZHZ{\etalchar{+}}19]{zhang2019bag}
Zhi Zhang, Tong He, Hang Zhang, Zhongyue Zhang, Junyuan Xie, and Mu~Li.
\newblock Bag of freebies for training object detection neural networks.
\newblock {\em arXiv preprint arXiv:1902.04103}, 2019.

\end{thebibliography}
}

\clearpage
\appendix
\addcontentsline{toc}{section}{Appendices}

\section{Technical Details in \Cref{sec:fDP}}
\label{app:fDP}
The Neyman--Pearson lemma is a useful tool in our paper. The following statement is adapted from \cite{lehmann2006testing}.
\begin{theorem} [Neyman-Pearson lemma]\label{thm:NPlemma}
Let $P$ and $Q$ be probability distributions on $\Omega$ with densities $p$ and $q$, respectively. For the hypothesis testing problem $H_0: P$ vs $H_1: Q$, a test $\phi:\Omega\to[0,1]$ is the most powerful test at level $\alpha$ if and only if there are two constants $h\in[0,+\infty]$ and $c\in[0,1]$ such that $\phi$ has the form
	\[
	\phi(\omega)=
	\left\{
	\begin{array}{ll}
	1, 		& \text{if } p(\omega)>hq(\omega), \\
	c, 		& \text{if } p(\omega)=hq(\omega), \\
	0, 		& \text{if } p(\omega)<hq(\omega),
	\end{array}
	\right.
	\]
	and $\E_P[\phi]=\alpha$.
\end{theorem}

Now, we use the Neyman--Pearson lemma to prove the proposition which gives sufficient and necessary conditions for trade-off functions.
\tradeoffthm*
In the entire appendix, from \ref{app:fDP} to \ref{app:SGD}, we will use $\T$ to denote the class of trade-off functions, and $\T^S$ the subclass of symmetric trade-off functions.
\begin{proof}[Proof of \Cref{prop:trade-off}]

	``only if'':
	Suppose $f=T(P_0,P_1)$. It is obviously non-increasing. The randomized testing rule that blindly rejects with probability $p$ achieves $(p,1-p)$ errors. It is suboptimal at level $p$, so $f(p)\leqslant 1-p$.

	Convexity follows from randomizing over two rejections rules. For given $\alpha,\alpha',t$, all in $[0,1]$, let $\phi$ and $\phi'$ be the rejection rules achieving errors $(\alpha,f(\alpha))$ and $(\alpha,f(\alpha))$ respectively. The rejection rule $\phi_t = t\phi+(1-t)\phi'$ achieves errors $t\alpha+(1-t)\alpha'$ and $tf(\alpha)+(1-t)f(\alpha')$. It is suboptimal at level $t\alpha+(1-t)\alpha'$, so we have
	\[f(t\alpha+(1-t)\alpha')\leqslant t\alpha+(1-t)\alpha'.\]

	As we remarked in the footnote, continuity in $(0,1]$ follows from properties we have proved.
	At 0 it requires a closer look at Neyman-Pearson lemma. Suppose $\alpha_n\to0$. Without loss of generality we can assume $\alpha_n$ is decreasing. We want to show $f(\alpha_n)\to f(0)$. Neyman-Pearson lemma tells us the optimal test $\phi_n$ at level $\alpha_n$ must have the form
	$$\phi_n(\omega)=
\left\{
\begin{array}{ll}
1, 		& \frac{p_1(\omega)}{p_0(\omega)}>h_n\\
c_n, & \frac{p_1(\omega)}{p_0(\omega)}=h_n\\
0, & \frac{p_1(\omega)}{p_0(\omega)}<h_n
\end{array}
\right.
$$
for some $c_n\in[0,1]$ and $h_n\in[0,+\infty]$. The fact that $\alpha_n$ is decreasing implies that $\phi_n(\omega)$ is monotone decreasing in $n$ except on a measure zero set, so it has a pointwise limit $\phi(\omega)$. Furthermore, $\phi(\omega)$ must be in the same form as $\phi_n$. Again by Neyman-Pearson lemma, $\phi$ must be the optimal test at level $\E_{P_0}[\phi]$. By dominated convergence theorem, $\E_{P_i}[\phi] = \lim_{n\to\infty}\E_{P_i}[\phi_n]$ for $i=0,1$. When $i=0$, this translates to $\E_{P_0}[\phi] = \lim_{n\to\infty} \alpha_n = 0$. So $\phi$ is at level 0. When $i=1$, we have
\[\E_{P_1}[\phi] = \lim_{n\to\infty}\E_{P_1}[\phi_n] = \lim_{n\to\infty} 1-f(\alpha_n)\]
where the second equality follows from the optimality of $\phi_n$. So
\[\lim_{n\to\infty} f(\alpha_n) = 1-\E_{P_1}[\phi] = f(0).\]
Here the second equality follows from the optimality of $\phi$. In fact, this argument works not just for 0 but for arbitrary $\alpha\in[0,1]$.

	``if'':
	Given $f$, we need to find $P,Q$. The common measurable space is the unit interval $[0,1]$. $P$ is the uniform distribution. $Q$ has density $-f'(1-x)$ on $[0,1)$ and an atom at 1 with $Q[\{1\}] = 1-f(0)$. In fact, $Q$ is constructed to have cdf $f(1-x)$, with the slight twist that the cdf is reset to be 1 at 1, because cdf has to be right continuous.

	It's easy to verify that $Q$ is indeed a probability distribution on $[0,1]$ using the properties of $f$.

	The likelihood ratio is simply $-f'(1-x)$ when $x<1$. At 1, it is 0 if $f(0)=1$ and $+\infty$ if $f(0)<1$. By convexity of $f$ it is non-decreasing, so the likelihood ratio rejection regions have the form $[h,1]$. Type I error is $P[h,1] = 1-h$. Type II error is $Q[0,h) = f(1-h)$.

\end{proof}
An equivalent object to our trade-off function is the ``testing region'' used in \cite{KOV}. For a trade-off function $f$, we define a special version of epigraph of $f$ as
\[\epi(f):= \{(\alpha,\beta)\mid \alpha\in[0,1], f(\alpha)\leqslant \beta \leqslant 1-\alpha\}.\]
Let $P,Q$ be distributions on $\Omega$. Recall that for a testing rule $\phi:\Omega\to[0,1]$, we denote its type I and type II errors by
$\alpha_\phi = \E_{P}[\phi], \quad \beta_\phi = 1 - \E_{Q}[\phi]$.
It is easy to see that $\epi(f)$ consists of all achievable type I and type II error pairs that is better than random guessing. Namely
	\[\epi(f)= \{(\alpha_\phi,\beta_\phi)\mid \phi:\Omega\to[0,1] \text{ measurable, } \alpha_\phi+\beta_\phi\leqslant 1\}.\]
This means considering $f$ and $\epi(f)$ are equivalent. For more information on the testing region, see Chapter 12 of \cite{itlectures}.

Next we justify the default assumption of symmetry.
\symmrep*
\begin{lemma}\label{lem:symmetry}
	If $f=\F(P,Q)$, then $f^{-1}=\F(Q,P)$.
\end{lemma}
\begin{proof}[Proof of \Cref{lem:symmetry}]
	The lemma is best illustrated from the epigraph point of view. It is an immediate consequence of the following claim: $(\alpha,\beta)\in \epi(f)$ if and only if $(\beta,\alpha)\in \epi(f^{-1})$, which is to say, $f(\alpha)\leqslant \beta \leqslant 1-\alpha$ if and only if $f^{-1}(\beta)\leqslant \alpha \leqslant 1-\beta$. In order to show this, the definition of $f^{-1}$, together with continuity of $f$, implies that $f(\alpha)\leqslant \beta\Leftrightarrow f^{-1}(\beta)\leqslant \alpha$. This justfies the claim and hence the lemma.
\end{proof}
\begin{proof}[Proof of \Cref{prop:symmetry}]
	Let $S,S'$ be neighboring datasets. Since $M$ is $f$-DP, we know that
	\begin{equation}\label{eq:hubei}
		\F\big(M(S),M({S'})\big)\geqslant f,\quad\F\big(M(S'),M({S})\big)\geqslant f.
	\end{equation}
	It follows easily from definition that if $f,g\in\T$ satisfy $g\geqslant f$, then $g^{-1}\geqslant f^{-1}$.
	So by \Cref{lem:symmetry} and the second inequality in \eqref{eq:hubei},
	$$\F\big(M(S),M({S'})\big) = \F\big(M(S'),M({S})\big)^{-1} \geqslant f^{-1}.$$
	Together with the first inequality in \eqref{eq:hubei}, we see for all neighboring datasets we have
	$$\F\big(M(S),M({S'})\big)\geqslant \max\{f,f^{-1}\}.$$

	It is straightforward to verify that if $f $ and $ g$ are both convex,
	continuous, non-increasing and below $\Id$ then $\max\{f,g\}$ also satisfy these properties. By \Cref{prop:trade-off}, we have $\max\{f,g\}\in\T$. So $f^\mathrm{S}=\max\{f,f^{-1}\}$ is in $\T$. The proof is complete.
\end{proof}

Recall that \Cref{eq:Gmu} states that
\[
	\F\big(\N(0,1), \N(\mu,1)\big)(\alpha) = \Phi\big(\Phi^{-1}(1-\alpha)-\mu\big).
\]
\begin{proof}[Proof of \Cref{eq:Gmu}]
	When $\mu\geqslant 0$, likelihood ratio of $\N(\mu,1)$ and $\N(0,1)$ is $\frac{\varphi(x-\mu)}{\varphi(x)} = \e^{\mu x - \frac{1}{2}\mu^2}$, a monotone increasing function in $x$. So the likelihood ratio tests must be thresholds: reject if the sample is greater than some $t$ and accept otherwise. Assuming $X\sim \N(0,1)$, the corresponding type I and type II errors are
	\[\alpha(t) = \P[X>t] = 1-\Phi(t),\quad \beta(t) = \P[X+\mu\leqslant t] = \Phi(t-\mu).\]
	Solving $\alpha$ from $t$, $t=\Phi^{-1}(1-\alpha)$. So
	\[G_\mu(\alpha) = \beta(\alpha) = \Phi\big(\Phi^{-1}(1-\alpha)-\mu\big).\]
\end{proof}

\postrep*
\begin{proof}[Proof of Lemma~\ref{lem:post}]
The idea is that whatever can be done with the processed outcome can also be done with the original outcome. Formally, if an optimal test $\phi:Z\to[0,1]$ for the problem $\Proc (P)$ vs $\Proc (Q)$ at level $\alpha$ can achieve type II error $\beta = \F\big(\Proc (P),\Proc (Q)\big)(\alpha)$, then it is easy to verify that $\phi\circ\Proc :Y\to[0,1]$ has the same errors $\alpha,\beta$ for the problem $P$ vs $Q$. The optimal error $T(P,Q)(\alpha)$ can only be smaller than $\beta$.
\end{proof}

The next result is a generalization of \Cref{eq:Gmu}, together with the interesting inverse.

Let $P$ be a probability distribution on $\R$ with density $p$, cdf $F:\R\to[0,1]$, quantile $F^{-1}:[0,1]\to[-\infty,+\infty]$ and $\xi$ be a random variable from the distribution $P$. Then we have
\begin{proposition}\label{prop:logconcave}
	$\F(\xi,t+\xi)(\alpha) = F(F^{-1}(1-\alpha)-t)$ holds for every $t>0$ if and only if the density $p$ is log-concave.
\end{proposition}
In particular, normal density is log-concave, so the expression of $G_\mu$ is a special case.
\begin{proof}[Proof of \Cref{prop:logconcave}]
	For convenience let
	\begin{equation*}\label{eq:logconcave}
		f_t(\alpha) := F(F^{-1}(1-\alpha)-t).
	\end{equation*}
	``if'': This is the easier direction. Fix $t>0$ and consider the log likelihood ratio of $\xi$ and $t+\xi$:
	$$\mathrm{llk} = \log p(x-t)-\log p(x).$$
	llk is increasing in $x$ because of log-concavity, so according to Neyman-Pearson lemma, the optimal rejection rule must have the form $1_{[h,+\infty)}$. Hence by a similar calculation as of Gaussian case, the trade-off function indeed has the form $f_t$.

	``only if'': We are given that $\F(\xi,t+\xi) = f_t$ holds for every $t>0$, and we want to show that $p$ is log-concave. Now that $f_t$ is a trade-off function  for every $t>0$, it must be convex. By chain rule$$f_t'(\alpha) = (-1)\cdot\frac{p(F^{-1}(1-\alpha)-t)}{p(F^{-1}(1-\alpha))}.$$
	Fix any $t>0$, convexity implies $f_t'(\alpha)$ is increasing in $\alpha$ for any $\alpha\in[0,1]$. Setting $x = F^{-1}(1-\alpha)$, we know $\frac{p(x-t)}{p(x)}$ is increasing in $x$ for all $x\in\R$, hence also $\log p(x-t)-\log p(x)$.

	For convenience let $g=\log p$. We know $g(x-t)-g(x)$ is increasing in $x, \forall t>0$. Equivalently, $g'(x-t)-g'(x)>0, \forall x, \forall t>0$, which means $g'(x)$ is decreasing, i.e. $g=\log p$ is concave. The proof is complete.
\end{proof}

Next we prove results presented in \Cref{sub:a_primal_dual_connection_with_}.
\ftoDPrep*
\begin{proof}[Proof of \Cref{prop:ftoDP}]
    The tangent line of $f$ with slope $k$ has equation $y=kx-f^*(k)$, so when $k=-\e^\ep$ the equation is
    \[y = -\e^\ep x-f^*(-\e^\ep).\]
    Compare it to $f_{\ep,\delta}$, we see $1-\delta = -f^*(-\e^\ep)$. By symmetry, the collection $\{f_{\ep,1+f^*(-\e^{\ep})}\}_{\ep\geqslant0}$ envelopes the function $f$.
\end{proof}
\GDPtoDPrep*
\begin{proof}[Proof of \Cref{corr:GDPtoDP}]
    By \Cref{prop:ftoDP}, $\mu$-GDP is equivalent to $(\ep,1+G_\mu^*(-\e^{\ep}))$-DP, so it suffices to compute the expression of $G_\mu^*(-\e^\ep)$.

    Recall that $G_\mu(x) = \Phi\big(\Phi^{-1}(1-x)-\mu\big)$.
    By definition,
    \[G_\mu^*(y) = \sup_{x\in[0,1]} yx-\Phi\big(\Phi^{-1}(1-x)-\mu\big).\]
    Let $t = \Phi^{-1}(1-x)$. Equivalently, $x = 1-\Phi(t)=\Phi(-t)$. Do the change of variable and we have
    \[G_\mu^*(y) = \sup_{t\in\R} ~y\Phi(-t)-\Phi(t-\mu).\]
    From the shape of $G_\mu$ we know the supremum must be achieved at the unique critical point. Setting the derivative of the objective function to be zero yields
    \begin{align*}
        \frac{\diff}{\diff t} \big[y\Phi(-t)-\Phi(t-\mu)\big] &= 0\\
        -y\varphi(-t) - \varphi(t-\mu) &=0\\
        y\e^{-\frac{1}{2}t^2}+\e^{-\frac{1}{2}(t-\mu)^2}&=0\\
        y+\e^{\mu t-\frac{1}{2}\mu^2}&=0
    \end{align*}
    So $t = \frac{\mu}{2}+\frac{1}{\mu}\log(-y)$. Plug this back in the expression of $G_\mu^*$ and we have
    \[G_\mu^*(y) = y\Phi\Big(-\frac{\mu}{2}-\frac{1}{\mu}\log(-y)\Big)-\Phi\Big(-\frac{\mu}{2}+\frac{1}{\mu}\log(-y)\Big).\]
    When $y=-\e^\ep$,
    \[G_\mu^*(-\e^\ep) = -\e^\ep\Phi\Big(-\frac{\mu}{2}-\frac{\ep}{\mu}\Big)-\Phi\Big(-\frac{\mu}{2}+\frac{\ep}{\mu}\Big).\]
    $1+G_\mu^*(-\e^{\ep})$ agrees with the stated formula in \Cref{corr:GDPtoDP}. The proof is complete.
\end{proof}


The rest of the section is devoted to group privacy results. The main theorem is

\groupthm*

For convenience we define an operation $\group$, which is function composition with a slight twist. For $f,g\in\T$,
\[f\group g (x) := f\big(1-g(x)\big).\]
$f^{\group k}$ is defined iteratively:
\[f^{\group k} = \underbrace{f\group\cdots\group f}_k.\]
Notice that $f\group g = 1-(1-f)\circ(1-g)$, so $f^{\group k} = 1-(1-f)^{\circ k}$.
\begin{lemma}
	The operation $\group$ has the following properties for $f,g\in\T$:
	\begin{enumerate}
		\item[(a)] $f\group g \in\T$.
		\item[(b)] $(f\group g)^{-1}=(g^{-1})\group (f^{-1}) $. In particular, if $f\in\T^S$, then $f^{\group k}\in\T^S$.
	\end{enumerate}
\end{lemma}
\begin{proof}
\begin{enumerate}
	\item[(a)]
	By \Cref{prop:trade-off}, it suffices to check the four properties for $f\group g$. Monotonicity and continuity are obvious. Convexity follows by the well-known fact that decreasing convex function composed with a concave function is convex. Finally, because $f(x)\leqslant 1-x,g(x)\leqslant 1-x$, we have
	\[f\group g (x) = f\big(1-g(x)\big)\leqslant 1-\big(1-g(x)\big) = g(x)\leqslant 1-x.\]
	\item[(b)]
	Recall that $f^{-1}(y) = \inf\{x\in[0,1]:f(x)\leqslant y\}$. We have
	\begin{align*}
		\big[(g^{-1})\group (f^{-1})\big](y)
		=g^{-1}\big(1-f^{-1}(y)\big)=\inf\{x\in[0,1]:g(x)\leqslant 1-f^{-1}(y)\}.
	\end{align*}
	For any two numbers $x,y\in[0,1]$, we have the following equivalence chain:
	$$g(x)\leqslant 1-f^{-1}(y) \Leftrightarrow f^{-1}(y)\leqslant 1-g(x) \Leftrightarrow f(1-g(x))\leqslant y\Leftrightarrow f\group g (x)\leqslant y.$$
	So
	\begin{align*}
		\big[(g^{-1})\group (f^{-1})\big](y)
		&=\inf\{x\in[0,1]:f\group g (x)\leqslant y\} = (f\group g)^{-1}(y).
	\end{align*}
	That is, $(g^{-1})\group (f^{-1}) = (f\group g)^{-1}$.
	The proof is complete.
\end{enumerate}
\end{proof}



\Cref{thm:group} is an immediate consequence of the following lemma:
\begin{lemma} \label{lem:group}
	Suppose $\F(P,Q)\geqslant f,\F(Q,R)\geqslant g$, then $\F(P,R)\geqslant g\group f$.
\end{lemma}
\begin{proof}
	Fix $\alpha\in[0,1]$. Suppose $\phi$ is the optimal testing rules of the problem $P$ vs $R$ at the level of $\alpha$. Then we know the type I error $\E_P[\phi]=\alpha$ and the type II error achieves the optimal value, i.e.
	\[1-\E_R[\phi]=\F(P,R)(\alpha).\]

	$\phi$ is suboptimal as a testing rule for the problem $Q$ vs $R$, so the type I and II errors must be above the trade-off function $g$. That is,
	\[1-\E_R[\phi]\geqslant \F(Q,R)(\E_Q[\phi]) \geqslant g(\E_Q[\phi]).\]
	Similarly, $\phi$ is also suboptimal for the problem $P$ vs $Q$. So $1-\E_Q[\phi]\geqslant f(\E_P[\phi]) = f(\alpha)$. Equivalently,
	\[\E_Q[\phi]\leqslant 1-f(\alpha).\]
	Put them together
	\begin{align*}
	\F(P,R)(\alpha) &= 1-\E_R[\phi]\\
	&\geqslant g(\E_Q[\phi])\\
	&\geqslant g\big(1-f(\alpha)\big) \quad\quad(g \text{ is decreasing})\\
	&=g\group f(\alpha).
	\end{align*}
	This completes the proof.
\end{proof}

\begin{proof}[Proof of \Cref{thm:group}]
	Suppose $S$ and $S'$ are $k$-neighbors, i.e. there exist datasets $S = S_0, S_1, \ldots, S_k = S'$ such that $S_i$ and $S_{i+1}$ are neighboring or identical for all $i = 0, \ldots, k-1$. By privacy of $M$, we know $T\big(M(S_i),M(S_{i+1})\big)\geqslant f$. Iteratively apply \Cref{lem:group} and we have
\[ T\big(M(S),M(S_2)\big)\geqslant f\group f, \quad T\big(M(S),M(S_3)\big)\geqslant f^{\group3} \quad \ldots \quad T\big(M(S),M(S')\big)\geqslant f^{\group k}.\]
We know that $f^{\group k} = 1-(1-f)^{\circ k}$, so the $f$-DP part of the claim is done.

The GDP part of the claim follows by an easy formula: $G_\mu\group G_{\mu'} = G_{\mu+\mu'}$.
To see this, recall that $G_\mu(\alpha)=\Phi(\Phi^{-1}(1-\alpha)-\mu)$.
\begin{align*}
	G_\mu\group G_{\mu'}(\alpha) = G_\mu\big(1-G_{\mu'}(\alpha)\big) = \Phi\big(\Phi^{-1}(G_{\mu'}(\alpha))-\mu\big) = \Phi\big(\Phi^{-1}(1-\alpha)-\mu-\mu'\big) = G_{\mu+\mu'}(\alpha).
\end{align*}
\end{proof}
In fact, similar conclusion holds for any log-concave noise. See \Cref{prop:logconcave}.

\grouplimit*
As What makes it even more interesting is the convergence occurs with very small $k$. In \Cref{fig:group} we set $\ep=0.5$ and $f = 1-(1-f_{\ep,0})^{\circ 2}$. So the blue curve in the last panel is $1-(1-f)^{\circ 2} = 1-(1-f_{\ep,0})^{\circ 4}$. Next we set $\mu=k\ep = 4\cdot 0.5 = 2$. It turns out these numbers are good enough for the condition $k\ep\to\mu$, because the predicted limit $\F\big(\mathrm{Lap}(0,1),\mathrm{Lap}(\mu,1)\big)$ (orange curve in the last panel) is almost indistinguishable from the blue curve $1-(1-f_{\ep,0})^{\circ 4}$.
\begin{figure}[!htp]
	\begin{center}
		\includegraphics[width=0.8\linewidth]{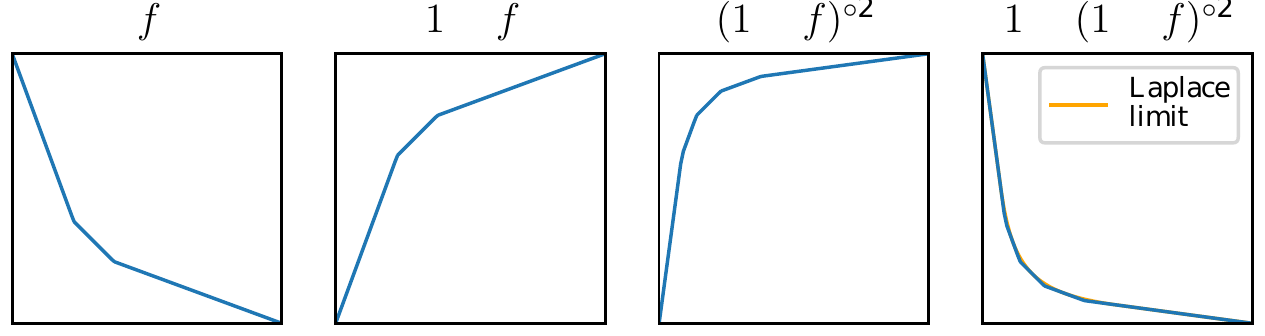}
	\end{center}
	\captionof{figure}{Group privacy corresponds to function composition. Here $f = 1-(1-f_{\ep,0})^{\circ 2}$ with $\ep=0.5$, so the blue curve in the last panel is $1-(1-f)^{\circ 2} = 1-(1-f_{\ep,0})^{\circ 4}$. Orange curve is the predicted limit $T\big(\mathrm{Lap}(0,1),\mathrm{Lap}(2,1)\big)$. The distinction is almost invisible even when $k$ is only 4.}
	\label{fig:group}
\end{figure}

\begin{lemma} \label{lem:lap}
	The trade-off function between Laplace distributions has expression
\begin{align*}
	\F\big(\mathrm{Lap}(0,1),\mathrm{Lap}(\mu,1)\big)(\alpha)=
	\left\{
	\begin{array}{ll}
		1-\e^{\mu}\alpha, 		& \alpha < \e^{-\mu}/2, \\
		\e^{-\mu}/4\alpha, & \e^{-\mu}/2\leqslant\alpha\leqslant1/2,\\
		\e^{-\mu}(1-\alpha), & \alpha>1/2.
	\end{array}
	\right.
\end{align*}
\end{lemma}

\begin{figure}[!htp]
	\begin{center}
		\includegraphics[width=0.7\linewidth]{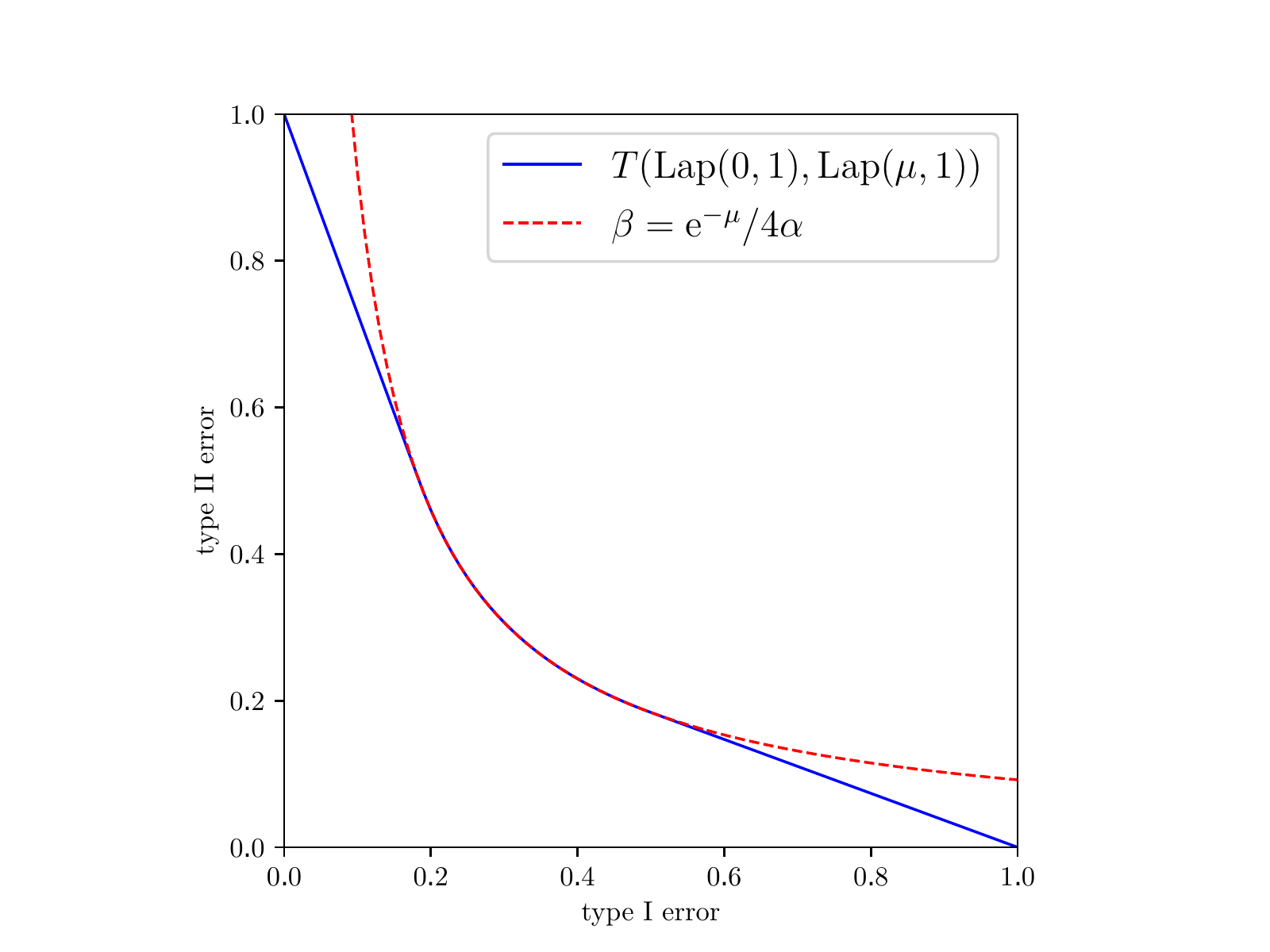}
	\end{center}
	\captionof{figure}{
	Graph of $
	\F\big(\mathrm{Lap}(0,1),\mathrm{Lap}(\mu,1)\big)$ with $\mu=1$. It agrees with the reciprocal function in the middle.
	}
	\label{fig:laplace}
\end{figure}
The graph of this function with $\mu = 1$ is illustrated in \Cref{fig:laplace}. In general, it consists of two symmetric line segments: $(0,1)$ connecting $(\e^{-\mu}/2, 1/2)$ and $(1/2,\e^{-\mu}/2)$ connecting $(1,0)$. Then $(\e^{-\mu}/2, 1/2)$ is connected to $(1/2,\e^{-\mu}/2)$ by the reciprocal function. It is easy to check that this function is $C^1$, i.e. has continuous derivative.
\begin{proof}[Proof of \Cref{lem:lap}]
Let $F$ be the cdf of $\mathrm{Lap}(0,1)$. By \Cref{prop:logconcave},
$$\F\big(\mathrm{Lap}(0,1),\mathrm{Lap}(\mu,1)\big)(\alpha) = F\big(F^{-1}(1-\alpha)-\mu\big).$$
Easy calculation yields
\begin{align*}
	F(x)=
	\left\{
	\begin{array}{ll}
	\e^{x}/2, 		& x \leqslant 0, \\
	1-\e^{-x}/2, & x > 0.
	\end{array}
	\right.
\end{align*}
So we must expect to divide into several categories. We will refer to the above two expressions as negative and positive regimes.

When $\alpha>1/2$, we are in negative regime. Solving $\e^{x}/2 = 1-\alpha$ gives us $F^{-1}(1-\alpha) = \log 2(1-\alpha)<0$. An additional $-\mu$ keeps us in negative regime, so
\[F\big(F^{-1}(1-\alpha)-\mu\big) = \exp\big(F^{-1}(1-\alpha)-\mu\big)/2 = \e^{\log 2(1-\alpha)-\mu}/2 = \e^{-\mu}(1-\alpha).\]
When $\alpha\leqslant1/2$, solving $1-\e^{-x}/2 = 1-\alpha$ gives us $F^{-1}(1-\alpha) = -\log 2\alpha\geqslant0$. If $-\log 2\alpha - \mu \leqslant 0$, i.e. $\e^{-\mu}/2\leqslant \alpha$, we are in negative regime and
\[F\big(F^{-1}(1-\alpha)-\mu\big) = \e^{-\log 2\alpha-\mu}/2 = \e^{-\mu}/4\alpha.\]
If $-\log 2\alpha - \mu > 0$, i.e. $\alpha<\e^{-\mu}/2$, we are in positive regime and
\[F\big(F^{-1}(1-\alpha)-\mu\big) = 1-\e^{\log 2\alpha+\mu}/2 = 1-\e^{\mu}\alpha.\]
The proof is complete.
\end{proof}


\begin{proof}[Proof of \Cref{prop:group_limit}]
For simplicity assume $\mu=1$. All arguments carry over for general $\mu$.

Let $f_n = 1-f_{\ep,0} = 1-f_{1/n,0}$. Fix $x_0$ and let $x_{n,k} = f_n^{\circ k}(x_0) = (1-f_{\ep,0})^{\circ k}(x_0)$. We are interested in showing
\[\lim_{n\to\infty}1-x_{n,n} = \F\big(\mathrm{Lap}(0,1),\mathrm{Lap}(1,1)\big)(x_0).\]
First we make a general observation: the sequence $\{x_{n,k}\}$ is increasing in $k$ for any $n$. This is because $f_{\ep,0}(x) \geqslant 1-x$ and hence $f_n(x)\geqslant x$.

\newcommand{\thresh}{\frac{1}{1+\e^{\frac{1}{n}}}}
Let $\theta_n = \thresh$. By the expression of $f_{\ep,0}$, we obtain the following two dynamics:
\begin{align*}
\begin{array}{rcll}
	f_n(x) &=& \e^{\frac{1}{n}}x, & \text{ if } x\leqslant \theta_n,\\
	1-f_n(x) &=& \e^{-\frac{1}{n}}(1-x),& \text{ if } x\geqslant \theta_n.
\end{array}
\end{align*}
The sequence $\{x_{n,k}\}$ evolves according to one of the two formula, potentially different for eack $k$. We will refer to $x\leqslant \theta_n$ case as \textit{linear dynamics} and $x\geqslant \theta_n$ case as \textit{flip linear regime} for evident reason.
For any $x_0$ and $n$, since $\{x_{n,k}\}$ is increasing in $k$, there exists a moment such that linear dynamics governs before and flip linear dynamics governs after. Extreme cases are one of the dynamics governs from $k=0$ to $n$.
We divide the analysis into three cases depending on the initial location $x_0$:
\begin{itemize}
	\item[(a)] $x_0< \frac{1}{2\e}$. In this case, for large enough $n$, the linear dynamics governs all the time. To see this, notice that $\theta_n$ increases to $\frac{1}{2}$ as $n\to\infty$. So for large enough $n$, $x_0< \frac{1}{\e}\cdot \theta_n$. It's easy to see that $x_{n,k}$ never exceeds $\theta_n$. Hence $x_{n,n} = \e x_0$.
	\item[(b)] $x_0\geqslant\frac{1}{2} = \sup_n \theta_n$. $x_{n,k}$ is born above threshold and remains above forever. Flip linear dynamics governs all the $n$ steps, so $1-x_{n,n} =\e^{-1}(1-x_0)$.
	\item[(c)] $\frac{1}{2e}\leqslant x < \frac{1}{2}$. Let $t$ be the time of dynamics change. More precisely,
	\begin{equation}\label{eq:group_threshold}
		t-1 = \max\{k:\e^{\frac{k}{n}}x\leqslant \theta_n.\}
	\end{equation}
	and
	\[x_{n,t} = \e^{\frac{1}{n}}x_{n,t-1}=\e^{\frac{t}{n}}x_0,\quad 1-x_{n,n} = \e^{-\frac{n-t}{n}}(1-x_{n,t}).\]
	Taking $n\to\infty$ in \eqref{eq:group_threshold} (using $\liminf$ and $\limsup$ when necessary), we know $\e^{\frac{t}{n}}\to\frac{1}{2x_0}$. So
	\[\lim_{n\to\infty}1-x_{n,n} = \lim_{n\to\infty}\e^{\frac{t}{n}-1}(1-x_{n,t}) = \lim_{n\to\infty}\e^{\frac{t}{n}-1}(1-\e^{\frac{t}{n}}x_0) = \e^{-1}\cdot\frac{1}{2x_0}(1-\frac{1}{2}) = \e^{-1}\cdot\frac{1}{4x_0}.\]
\end{itemize}
Collecting all three cases, we have
\begin{align*}
	\lim_{n\to\infty}1-x_{n,n} =
	\left\{
	\begin{array}{ll}
		1-\e x_0, 		& x_0< \frac{1}{2\e}, \\
		\e^{-1}\cdot\frac{1}{4x_0}, & \e^{-\mu}/2\leqslant\alpha\leqslant1/2,\\
		\e^{-1}(1-x_0), & x_0\geqslant\frac{1}{2}.
	\end{array}
	\right.
\end{align*}
By \Cref{lem:lap}, this agrees with $\F\big(\mathrm{Lap}(0,1),\mathrm{Lap}(1,1)\big)$.
Uniform convergence comes for free for trade-off functions once we have pointwise convergence. This is a direct consequence of \Cref{lem:uniform} below, which will be used multiple times in this paper.
\end{proof}
	\begin{lemma} \label{lem:uniform}
		Let $f_n:[a,b]\to\R$ be a sequence of non-increasing functions. If $f_n$ has pointwise limit $f:[a,b]\to\R$ where $f$ is continuous on $[a,b]$, then the limit is uniform.
	\end{lemma}

	This is an easy variant of P\'{o}lya's theorem (\cite{polya1920zentralen}. See also Theorem 2.6.1 in \cite{lehmann2004elements}). For completeness, we provide a proof.
	\begin{proof}[Proof of \Cref{lem:uniform}]
	We are going to show that for every $\ep>0$, there exists $N$ such that
	$$|f_n(x)-f(x)|< \ep, \quad\forall x\in[a,b], \forall n\geqslant N.$$
	Since $f$ is continuous on a closed interval, it is uniformly continuous. So for a fixed $\ep>0$, we can find $\delta>0$ such that whenever $x,y\in[a,b]$ satisfies $|x-y|<\delta$, we have $|f(x)-f(y)|<\ep/2$.\\
	Then we can divide $[a,b]$ into small intervals $a=x_0<x_1<\cdots<x_{m-1}<x_m=b$ such that each interval is shorter than $\delta$. For these $m+1$ points we can find $N$ such that
	\begin{equation}\label{eq:lsy}
		|f_n(x_i)-f(x_i)|< \frac{\ep}{2}, \quad\forall 0\leqslant i\leqslant m, \forall n\geqslant N.
	\end{equation}
	We claim this $N$ works for our purpose. For any $x\in[a,b]$, there exists a sub-interval that contains it, namely $[x_i,x_{i+1}]$. By monotonicity of $f_n$ we have
	\begin{equation}\label{eq:liu}
		f_n(x_{i+1})-f(x)\leqslant f_n(x)-f(x)\leqslant f_n(x_i)-f(x).
	\end{equation}
	Now for $n\geqslant N$,
	\begin{align*}
		f_n(x_i)-f(x) &= [f_n(x_i)-f(x_i)]+[f(x_i)-f(x)]\\
					&< \frac{\ep}{2} + [f(x_i)-f(x)] &&\text{(By \eqref{eq:lsy})}\\
					&<\frac{\ep}{2} +\frac{\ep}{2} =\ep. &&(\text{uniform continuity of $g$})
	\end{align*}
	This shows the right hand side of \eqref{eq:liu} is less than $\ep$. A similar argument for the left hand side yields
	$$|f_n(x)-f(x)|<\ep,$$
	which justifies the choice of $N$.
	\end{proof}


\section{Conversion from $f$-DP to Divergence Based DP} 
\label{app:relation}
As the title suggests, the central question of this section is the conversion from $f$-DP to divergence based DP. It boils down to the conversion from trade-off functions to various divergences. We first introduce the most general tool, and then give explicit formula for a large class of divergences, including {\Renyi} divergence. At the end we prove the claim we made in \Cref{sec:conn-with-blackw} that privacy notion based on R\'enyi divergence does not induce a strictly larger order than the Blackwell order.


Suppose we have a ``divergence'' $D(\cdot\|\cdot)$, which takes in a pair of probability distributions on a common measurable space and outputs a number. We say $D$ satisfies data processing inequality if $D\big(\mathrm{Proc}(P)\|\mathrm{Proc}(Q)\big)\geqslant D(P\|Q)$ for any post-processing Proc.
\begin{proposition} \label{thm:functional}
	If $D(\cdot\|\cdot)$ satisfies data processing inequality, then there exists a functional $l_D:\T\to\R$ that computes $D$ through the trade-off function:
	$$D(P\|Q) = l_D\big(\F(P,Q)\big).$$
\end{proposition}
\begin{proof}
	It's almost immediate from the following
	\begin{lemma} \label{lem:onlyonce}
	If $\F(P',Q')\geqslant\F(P,Q)$, then $D(P'\|Q') \leqslant D(P\|Q)$. In particular, $\F(P,Q)=\F(P',Q')$ implies $D(P\|Q) = D(P'\|Q')$.
	\end{lemma}
	To see why the lemma holds, notice by Blackwell's theorem, $\F(P',Q')\geqslant\F(P,Q) $ implies that there is a Proc such that $P'=\mathrm{Proc}(P),Q'=\mathrm{Proc}(Q)$, and by data processing inequality, $D(P'\|Q') \leqslant D(P\|Q)$.

	The lemma implies the existence of $l_D$ because we can define $l_D(f) = D(P\|Q)$ through any pair $P,Q$ such that $T(P,Q)=f$. This definition is independent of the choice of $P$ and $Q$.
\end{proof}
An immediate corollary is
\begin{corollary}\label{cor:wellbeing}
	If two trade-off functions $f,g$ satisfy $f\geqslant g$, then $l_D(f)\le l_D(g)$.
\end{corollary}
\paragraph{Example: $F$-divergence} 
Let $P,Q$ be a pair of distributions with density $p$ and $q$ with respect to some common dominating measure. For a convex function $F:(0,+\infty)\to\R$ such that $F(1) = 0$, the $F$-divergence $D_F(P\|Q)$ is defined as (see \cite{liese2006divergences})
\begin{align*}
	D_F(P\|Q) &= \int_{\{pq>0\}} F\Big(\frac{p}{q}\Big) \diff Q + F(0)Q[p=0]+\tau_F P[q=0]
\end{align*}
where $F(0) = \lim_{t\to 0^+} {F(t)}$ and $\tau_F:=\lim_{t\to+\infty} \frac{F(t)}{t}$.
We further set the rules $F(0) \cdot 0 = \tau_F\cdot 0=0$ even if $F(0)=+\infty$ or $\tau_F=+\infty$.
\begin{proposition} \label{prop:fdiv}
Let $z_f = \inf\{x\in[0,1]:f(x)=0\}$ be the first zero of $f$. The functional $l_F:\T\to\R$ that computes $F$-divergence has expression
$$l_F(f) = \int_0^{z_f} F\big({\big|f'(x)\big|}^{-1}\big)\cdot \big|f'(x)\big| \diff x + F(0)\cdot(1-f(0))+\tau_F\cdot (1-z_f).$$
In particular, when $f\in\T^S$ and $f(0)=1$, we have
\begin{equation}\label{eq:fdiv}
	l_F(f) = \int_0^1 F\big({\big|f'(x)\big|}^{-1}\big)\cdot \big|f'(x)\big| \diff x.
\end{equation}
\end{proposition}
\begin{proof}[Proof of \Cref{prop:fdiv}]
	For a given trade-off function $f$, in order to determine $l_F(f)$, it suffices to find $P,Q$ such that $f=T(P,Q)$ and then use the property $l_F(f) = D_F(P\|Q)$. Such a pair is constructed in the proof of \Cref{prop:trade-off}: $P$ is the uniform distribution on $[0,1]$ and $Q$ has density $|f'(1-x)|$ on $[0,1)$ and an atom at $1$ with $Q[\{1\}] = 1-f(0)$. When we set the dominating measure $\mu$ to be Lebesgue in $[0,1)$ and have an atom at 1 with measure 1, the densities $p$ and $q$ have expressions
	\[p(x) = 
	\left\{
	\begin{array}{ll}
		1, 		& x\in[0,1), \\
		0, & x=1.
	\end{array}
	\right.
	\quad\text{ and }\quad q(x) = 
	\left\{
	\begin{array}{ll}
		|f'(1-x)|, 		& x\in[0,1), \\
		1-f(0), & x=1.
	\end{array}
	\right.\]
	Readers should keep in mind that the value at 1 matters because the base measure $\mu$ has an atom there. For a trade-off function $f$, its derivative $f'(x)$ never vanishes before $f$ hits zero, i.e. $f'(x)>0$ for $x<z_f$ and $f'(x)=0$ for $x\ge z_f$. Equivalently, $\{q>0\} = (1-z_f,1]$ and $\{q=0\} = [0,1-z_f]$.
	So
	\begin{align*}
		D_F(P\|Q)
		&= \int_{\{pq>0\}} F\Big(\frac{p}{q}\Big) \diff Q + F(0)Q[p=0]+\tau_F P[q=0]\\
		&= \int_{1-z_f}^1 F(|f'(1-x)|^{-1}) \cdot |f'(1-x)| \diff x+ F(0)\cdot(1-f(0)) + \tau_F\cdot (1-z_f)\\
		&=\int^{z_f}_0 F(|f'(x)|^{-1}) \cdot |f'(x)| \diff x+ F(0)\cdot(1-f(0)) + \tau_F\cdot (1-z_f).
	\end{align*}
	Starting from the second line, the integral is Lebesgue integral. Now the proof is complete.
\end{proof}
Because of the generality of $F$-divergence, \Cref{eq:fdiv} has broad applications. Many important divergences can be computed via a simple formula. Below are some of the examples.
\begin{itemize}
	\item \textbf{Total variation distance} corresponds to $F(t) = \frac{1}{2}|t-1|$. Easy calculation yields
	\[l_{\mathrm{TV}}(f) = \frac{1}{2}\int_0^1 \big|1+f'(x)\big|\diff x.\]
	\item \textbf{KL divergence} corresponds to $F(t) = t\log t$. We have
	\[l_{\mathrm{KL}}(f) =-\int_0^1 \log\big|f'(x)\big|\diff x.
	\]
	This functional plays an important role in our central limit theorem. We call it $\mathrm{kl}(f)$ there.
	\item \textbf{{Power} divergence} of order $\alpha$ corresponds to $F_\alpha(t) =\frac{t^\alpha-\alpha(t-1)-1}{\alpha(\alpha-1)}$. The corresponding functional is
	\begin{align*}
		l_{F_\alpha}(f) = 
		\left\{
		\begin{array}{ll}
			\frac{1}{\alpha(\alpha-1)}\Big(\int_0^1 \big|f'(x)\big|^{1-\alpha}\diff x - 1\Big), 		& z_f = 1, \\
			+\infty, & z_f<1.
		\end{array}
		\right.
	\end{align*}
	\item \textbf{{\Renyi} divergence} of order $\alpha$ is defined as
	\[D_\alpha(P\|Q) = \tfrac{1}{\alpha-1}\log \big(\E_P(\tfrac{p}{q})^{\alpha-1}\big)=\tfrac{1}{\alpha-1}\log\int p^\alpha q^{1-\alpha}.\]
	It is related to power divergence of order $\alpha$ by
	\begin{equation}\label{eq:renyipower}
		D_\alpha(P\|Q) = \frac{1}{\alpha-1}\cdot \log\big(\alpha(\alpha-1)D_{F_\alpha}(P\|Q)+1\big).
	\end{equation}
	So the corresponding functional, which we denote by $l_\alpha^{\text{R{\'e}nyi}}$, has expression
	\begin{equation}\label{eq:renyi_functional}
		l_\alpha^{\text{R{\'e}nyi}}(f)=
		\left\{
		\begin{array}{ll}
			\frac{1}{\alpha-1}\log \int_0^1 \big|f'(x)\big|^{1-\alpha}\diff x, 		& z_f = 1, \\
			+\infty, & z_f<1.
		\end{array}
		\right.
	\end{equation}
\end{itemize}
\begin{proof}[Proof of \Cref{eq:renyipower}]
	\begin{align*}
		D_{F_\alpha}(P\|Q) &= \int q\cdot F_\alpha\Big(\frac{p}{q}\Big)\\
		&= \int q\cdot \frac{(\frac{p}{q})^\alpha - \alpha(\frac{p}{q}-1)-1}{\alpha(\alpha-1)}\\
		&=\frac{1}{\alpha(\alpha-1)}\cdot \int p^\alpha q^{1-\alpha} + 0 -\frac{1}{\alpha(\alpha-1)}\\
		&=\frac{1}{\alpha(\alpha-1)}\Big(e^{(\alpha-1)D_\alpha(P\|Q)}-1\Big).
	\end{align*}
	Solving for $D_\alpha(P\|Q)$ yields \eqref{eq:renyipower}.
\end{proof}
Introduced in \cite{renyi}, a mechanism $M$ is said to be $(\alpha, \epsilon)$-{\Renyi} differentially private (RDP) if
for all neighboring pairs $S, S'$ it holds that
\begin{equation}\label{eq:RDP}
D_\alpha(M(S) \| M(S')) \le \epsilon,
\end{equation}
A few other DP definitions, including zero concentrated differential privacy (zCDP) \cite{concentrated2} and truncated concentrated differential privacy (tCDP) \cite{tcdp}, are defined through imposing bounds in the form of \eqref{eq:RDP} with certain collections of $\alpha$. The following proposition provides the general conversion from $f$-DP to RDP via \eqref{eq:renyi_functional}.
\begin{proposition} \label{lem:}
	If a mechanism is $f$-DP, then it is $\big(\alpha,l_\alpha^{\text{R{\'e}nyi}}(f)\big)$-RDP for any $\alpha>1$.
\end{proposition}

Specializing to the most important subclass, we have a simple expression.
\begin{corollary}
	If a mechanism is $\mu$-GDP, then it is $(\alpha,\frac{1}{2}\mu^2\alpha)$-RDP for any $\alpha>1$.
\end{corollary}
\begin{proof}
By the property of $l_\alpha^{\text{R{\'e}nyi}}$, we know $l_\alpha^{\text{R{\'e}nyi}}(G_\mu) = D_\alpha\big(\N(0,1)\|\N(\mu,1)\big)$.
	Easy calculation shows $D_\alpha\big(\N(0,1)\|\N(\mu,1)\big) = \frac{1}{2}\mu^2\alpha$. Readers can refer to Proposition 7 in \cite{renyi} for a detailed derivation.
\end{proof}

The functional $l_D$ allows a consistent, easy conversion from an $f$-DP guarantee to all divergence based DP guarantees. The above conversion to RDP is, among all, the most useful example. On the other hand, conversion from divergence, either to trade-off function or to other divergences, often requires case by case analysis, sometimes significantly non-trivial. What's worse is that it is often hard to tell whether a given conversion between divergences is improvable or already lossless. For conversion between $F$-divergences, a systematic approach called joint range is developed in \cite{harremoes2011pairs}, but it is still significantly more complicated than \Cref{eq:fdiv}. On the other hand, \Cref{thm:functional} means conversion from trade-off to divergence is lossless and unimprovable.

This fine-grainedness of trade-off function (see also \Cref{sec:conn-with-blackw}) is somewhat expected: it summarizes the indistinguishability of a pair of distribution by a \textit{function}, which is an infinite dimensional object. In contrast, divergences usually just summarize by a number, which is obviously less informative by a function.

Connecting back to informativeness, we argue that, even when we consider $\{D_\alpha(P\|Q):\alpha>1\}$ as an infinite dimensional object, it still does not induce Blackwell's order. In the language of \Cref{sec:conn-with-blackw}, $\Ineq({\preceq_{\textup{R\'enyi}}})\supsetneq \Ineq({\preceq_{\mathrm{Blackwell}}})$. In other words, there are two pairs of distributions, one easier to distinguish than the other in the R\'enyi sense, but not in the Blackwell sense.

Let $P_\ep$ and $Q_{\ep}$ denote Bernoulli distributions with success probabilities $\frac{\e^\ep}{1+\e^\ep}$ and $\frac{1}{1+\e^\ep}$, respectively.
\begin{proposition}\label{prop:renyi_fail}
There exists $\ep>0$ such that the following two statements are both true:	
\begin{enumerate}
\item[(a)] 
For all $\alpha>1$, $D_\alpha(P_\ep \| Q_\ep) \leqslant D_\alpha\big(\N(0,1) \| \N(\ep,1)\big)$;
\item[(b)] $\TV(P_\ep, Q_\ep)> \TV\big(\N(0,1), \N(\ep,1)\big)$.
\end{enumerate}
\end{proposition}
Surprisingly, although the whole collection of {\Renyi} divergences asserts that the pair $\big(\N(0,1) ,\N(\ep,1)\big)$ is easier to distinguish than $(P_\ep , Q_\ep)$, one can nevertheless achieve smaller summed type I and type II errors when trying to distinguish $P_\ep, Q_\ep$.
Fact (a) equivalently says that $\big(\N(0,1) ,\N(\ep,1)\big)\preceq_{\text{R\'enyi}} (P_\ep \| Q_\ep)$, while fact (b) excludes the possibility that $\big(\N(0,1) ,\N(\ep,1)\big)\preceq_{\mathrm{Blackwell}} (P_\ep , Q_\ep)$, since otherwise data processing inequality of the total variation distance would imply $\TV(P_\ep, Q_\ep)\leqslant \TV\big(\N(0,1), \N(\ep,1)\big)$.

We point out that (a) in \Cref{{prop:renyi_fail}} in fact holds for all $\ep\geqslant0$, which is proved in \cite{concentrated2}, partially based on numerical evidence. Our proof is entirely analytical.
\begin{proof}[Proof of \Cref{prop:renyi_fail}]
	\begin{align*}
		D_\alpha(P_\ep\|Q_\ep)
		&= \frac{1}{\alpha-1}\log (p^\alpha q^{1-\alpha}+q^\alpha p^{1-\alpha})\\
		&= \frac{1}{\alpha-1}\log\, \frac{\e^{\ep\alpha}+\e^{\ep(1-\alpha)}}{1+\e^\ep}.\\
		&= \frac{1}{\alpha-1}\log\, \frac{\e^{\ep(\alpha-\frac{1}{2})}+\e^{\ep(\frac{1}{2}-\alpha)}}{\e^{-\frac{\ep}{2}}+\e^{\frac{\ep}{2}}}\\
		&= \frac{1}{\alpha-1}\log\, \frac{\cosh \ep(\alpha-\frac{1}{2})}{\cosh \frac{\ep}{2}}
	\end{align*}
	Now we claim that $\cosh x \cdot \e^{-\frac{1}{2}x^2}$ is monotone decreasing for $x\geqslant0$. To see this, simply take the derivative
	$$\big(\cosh x \cdot \e^{-\frac{1}{2}x^2}\big)' = \sinh x\cdot \e^{-\frac{1}{2}x^2} + \cosh x \cdot (-x) \cdot \e^{-\frac{1}{2}x^2} = (\tanh x - x)\cdot\cosh x\cdot \e^{-\frac{1}{2}x^2}.$$
	It is easy to show $\tanh x\leqslant x$ for $x\geqslant0$. Hence the derivative is always non-positive, which justifies the claimed monotonicity.
	Since $\alpha>1,\ep\geqslant0$, we have $\ep(\alpha-\frac{1}{2})\geqslant\frac{\ep}{2}\geqslant 0$. By the monotonicity,
	\begin{align*}
		\cosh \ep(\alpha-\frac{1}{2}) \cdot \e^{-\frac{1}{2}\ep^2(\alpha-\frac{1}{2})^2} &\leqslant \cosh \frac{\ep}{2} \cdot \e^{-\frac{1}{2}\cdot (\frac{\ep}{2})^2}.
	\end{align*}
	That is,
	\[
		\frac{\cosh \ep(\alpha-\frac{1}{2})}{\cosh \frac{\ep}{2}} \leqslant \e^{\frac{1}{2}\ep^2\alpha(\alpha-1)}.
	\]
	So for any $\ep\geqslant0$,
	\begin{align*}
		D_\alpha(P_\ep\|Q_\ep) = \frac{1}{\alpha-1}\cdot \log\, \frac{\cosh \ep(\alpha-\frac{1}{2})}{\cosh \frac{\ep}{2}} \leqslant \frac{1}{2}\ep^2\alpha = D_\alpha\big(\N(0,1)\|\N(\ep,1)\big).
	\end{align*}
	For the second part, easy calculation and Taylor expansion yields
	\begin{align*}
		\mathrm{TV}(P_\ep,Q_\ep) &= \frac{\e^\ep-1}{\e^\ep+1} = \tanh \frac{\ep}{2} = \frac{\ep}{2} +o(\ep^2),\\
		\mathrm{TV}\big(\N(0,1),\N(\ep,1)\big) &= 1-2\Phi(-\frac{\ep}{2}) = \int_{-\frac{\ep}{2}}^{\frac{\ep}{2}}\varphi(x)\diff x \\
		&= \varphi(0)\cdot\ep+o(\ep^2) = \frac{1}{\sqrt{2\pi}}\cdot\ep+o(\ep^2).
	\end{align*}
	Since $\frac{1}{\sqrt{2\pi}}<\frac{1}{2}$, for small enough $\ep$, $\mathrm{TV}\big(\N(0,1),\N(\ep,1)\big)< \mathrm{TV}(P_\ep,Q_\ep)$.
\end{proof}

In summary, \Cref{thm:functional} and \Cref{prop:fdiv} provide general tools to losslessly convert a trade-off function to divergences and hence justifies the fine-grainedness of trade-off functions. This is complementary to the informativeness argument in \Cref{sec:conn-with-blackw}.


\section{A Self-contained Proof of the Composition Theorem}
\label{app:self}
In this section we prove the well-definedness of $\otimes$ and Composition \Cref{thm:n_steps}.


We begin with the setting of the key lemma, which compares indistinguishability of two pairs of \textit{randomized algorithms}. Let $K_1, K_1': Y \to Z_1$ and $K_2, K_2': Y \to Z_2$ be two pairs of randomized algorithms. Suppose the following is true for these four algorithms: for each fixed input $y\in Y$, testing problem $K_1(y) $ vs $ K_1'(y)$ is harder than $K_2(y) $ vs $ K_2'(y)$. In mathematical language, let $f_i^y = T\big(K_i(y), K_i'(y)\big)$ (See the left panel of \Cref{fig:comparison}). The above assumption amounts to saying $f_1^y\geqslant f_2^y$. So far we have fixed the input $y$. In the two pairs of testing problems, if the input of the null comes from $P$ and the input of the alternative comes from $P'$, then intuitively both testing problems become easier than when inputs are fixed, because now the inputs also provide information. Formally, the observation comes from input-output joint distribution $\big(P, K_i(P)\big)$ or $\big(P', K_i'(P')\big)$ (with a little abuse of notation). Let $f_i = T\big( (P, K_i(P)), (P', K_i'(P')) \big), i=1,2$ be the trade-off functions of the joint testing problems (See the right panel of \Cref{fig:comparison}). As discussed, we expect that $f_1\leqslant f_1^y, f_2\leqslant f_2^y$ for all $y$. But what about $f_1$ and $f_2$? Which joint testing problem is harder? The following lemma answers the question.

\begin{lemma} \label{lem:comparison}
	If $f_1^y\geqslant f_2^y$ for all $y\in Y$, then $f_1\geqslant f_2$.
\end{lemma}

\usetikzlibrary{shapes,arrows}
\tikzstyle{line}=[draw] 
\usetikzlibrary{positioning,patterns,arrows,decorations.pathreplacing}
 \tikzset{main node/.style={},
            }
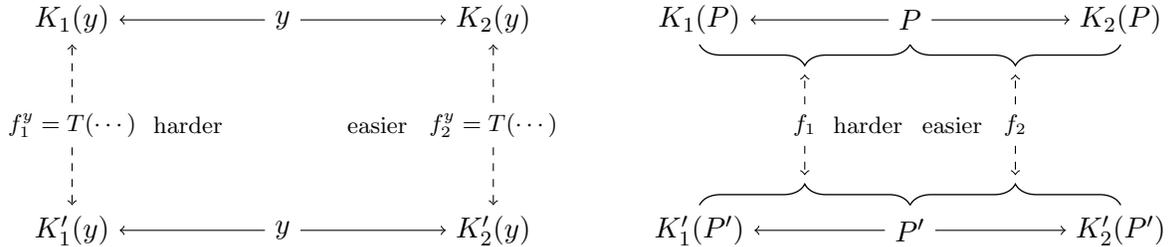
\begin{figure}[!htp]
	\begin{minipage}{.5\textwidth}
	  \centering
	\begin{tikzpicture}
\def\scale{0.7}
\def\h{4*\scale}
\def\x{4*\scale}
\def\y{0*\scale}
\def\algo{0.2}

    \node[main node] (sDown) at (\x,0) {$y$};
    \node[main node] (T1Down) at (0,\y) {$K_1'(y)$};
    \node[main node] (T2Down) at (2*\x,\y) {$K_2'(y)$};
    \node[main node] (sUp) at (\x,\y+\h+\y) {$y$};
    \node[main node] (T1Up) at (0,\y+\h) {$K_1(y)$};
    \node[main node] (T2Up) at (2*\x,\y+\h) {$K_2(y)$};
    \node[main node] (trade1) at (0,\y+\h*0.5) {\footnotesize{$f_1^y = T(\cdots)$}};
    \node[main node] (trade2) at (\x*2,\y+\h*0.5) {\footnotesize{$f_2^y = T(\cdots)$}};
    \node[main node] (hard) at (\x*0.55,\h*0.5) {\footnotesize harder};
    \node[main node] (easy) at (\x*1.45,\h*0.5) {\footnotesize easier};

\draw [->] (sDown) -- node[below] {} (T1Down); 
\draw [->] (sDown) -- node[below] {} (T2Down);
\draw [->] (sUp) -- node[above] {} (T1Up); 
\draw [->] (sUp) -- node[above] {} (T2Up);
\draw [->,dashed] (trade1) -- (T1Up);
\draw [->,dashed] (trade1) -- (T1Down);
\draw [->,dashed] (trade2) -- (T2Up);
\draw [->,dashed] (trade2) -- (T2Down); 
\def\s{0.55}
\end{tikzpicture}
	\end{minipage}
	\begin{minipage}{.5\textwidth}
	  \centering
		
\begin{tikzpicture}
\def\scale{0.7}
\def\h{4*\scale}
\def\x{4*\scale}
\def\y{0*\scale}
\def\algo{0.2}

    \node[main node] (sDown) at (\x,0) {$P'$};
    \node[main node] (T1Down) at (0,\y) {$K_1'(P')$};
    \node[main node] (T2Down) at (2*\x,\y) {$K_2'(P')$};
    \node[main node] (sUp) at (\x,\y+\h+\y) {$P$};
    \node[main node] (T1Up) at (0,\y+\h) {$K_1(P)$};
    \node[main node] (T2Up) at (2*\x,\y+\h) {$K_2(P)$};
    \node[main node] (trade1) at (\x*0.5,\y+\h*0.5) {\footnotesize{$f_1$}};
    \node[main node] (trade2) at (\x*1.5,\y+\h*0.5) {\footnotesize{$f_2$}};
    \node[main node] (hard) at (\x*0.8,\h*0.5) {\footnotesize harder};
    \node[main node] (easy) at (\x*1.2,\h*0.5) {\footnotesize easier};

\draw [->] (sDown) -- node[below] {} (T1Down); 
\draw [->] (sDown) -- node[below] {} (T2Down);
\draw [->] (sUp) -- node[above] {} (T1Up); 
\draw [->] (sUp) -- node[above] {} (T2Up);
\def\amp{8pt}
\def\shift{0.4cm}
\draw [thick,decorate,line width=0.5pt,decoration={brace,amplitude=\amp,mirror},xshift=0.0pt,yshift=-0cm] (T1Up.south) -- coordinate[pos=0.5, yshift=-\shift] (lu) (T1Up.south-|sUp) ;
\draw [thick,decorate,line width=0.5pt,decoration={brace,amplitude=\amp,mirror},xshift=0.0pt,yshift=-0cm] (T1Up.south-|sUp) -- coordinate[pos=0.5, yshift=-\shift] (ru) (T2Up.south) ;
\draw [thick,decorate,line width=0.5pt,decoration={brace,amplitude=\amp,mirror},xshift=0.0pt,yshift=-0cm] (T1Down.north-|sDown) -- coordinate[pos=0.5, yshift=\shift] (ld) (T1Down.north) ;
\draw [thick,decorate,line width=0.5pt,decoration={brace,amplitude=\amp,mirror},xshift=0.0pt,yshift=-0cm] (T2Down.north) -- coordinate[pos=0.5, yshift=\shift] (rd) (T1Down.north-|sDown) ;
\draw [->,dashed] (trade1) -- (lu);
\draw [->,dashed] (trade1) -- (ld);
\draw [->,dashed] (trade2) -- (ru);
\draw [->,dashed] (trade2) -- (rd); 
\def\s{0.55}
\end{tikzpicture}
	\end{minipage}
		\captionof{figure}{Assumption (left) and conclusion (right) of \Cref{lem:comparison}. Solid arrows indicate (random) mapping and dashed arrows indicate the trade-off function of the two ends. For example, $f_1$ in the right panel is the trade-off function of two joint distributions: $\big(P,K_1(P)\big)$ and $\big(P',K_1(P')\big)$.}
	  \label{fig:comparison}
\end{figure}

Let's first use the lemma to show the well-definedness of $\otimes$ and the composition theorem. Its own proof comes afterwards.
Recall that in \Cref{def:product}, $f\otimes g$ is defined as $T(P\times P',Q\times Q')$ if $f=T(P,Q), g = T(P',Q')$. To show this definition does not depend on the choice of $P,Q$ and $P',Q'$, it suffices to verify that when $f = T(P,Q) = T(\tilde{P},\tilde{Q})$, we have $T(P\times P', Q \times Q') = T(\tilde{P} \times P', \tilde{Q} \times Q')$. The following lemma is slightly stronger than what we need, but will be useful later.
\begin{lemma} \label{lem:welldefined}
If $T(P,Q) \geqslant T(\tilde{P},\tilde{Q})$, then
\[ T(P\times P', Q \times Q') \geqslant T(\tilde{P} \times P', \tilde{Q} \times Q').\]
As a consequence, if the assumption holds with an equality, then so does the conclusion.
\end{lemma}
\usetikzlibrary{positioning,patterns,arrows,decorations.pathreplacing}
 \tikzset{main node/.style={},
            }
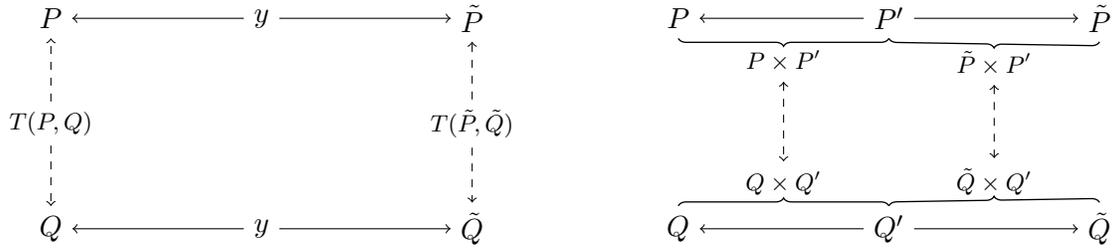
\begin{figure}[!htp]
\begin{minipage}{.5\textwidth}
  \centering
\begin{tikzpicture}
\def\scale{0.7}
\def\h{4*\scale}
\def\x{4*\scale}
\def\y{0*\scale}
\def\algo{0.2}

    \node[main node] (sDown) at (\x,0) {$y$};
    \node[main node] (T1Down) at (0,\y) {$Q$};
    \node[main node] (T2Down) at (2*\x,\y) {$\tilde{Q}$};
    \node[main node] (sUp) at (\x,\y+\h+\y) {$y$};
    \node[main node] (T1Up) at (0,\y+\h) {$P$};
    \node[main node] (T2Up) at (2*\x,\y+\h) {$\tilde{P}$};
    \node[main node] (trade1) at (0,\y+\h*0.5) {\footnotesize{$T(P,Q)$}};
    \node[main node] (trade2) at (\x*2,\y+\h*0.5) {\footnotesize{$T(\tilde{P},\tilde{Q})$}};

\draw [->] (sDown) -- (T1Down); 
\draw [->] (sDown) -- (T2Down);
\draw [->] (sUp) -- (T1Up); 
\draw [->] (sUp) -- (T2Up);
\draw [->,dashed] (trade1) -- (T1Up);
\draw [->,dashed] (trade1) -- (T1Down);
\draw [->,dashed] (trade2) -- (T2Up);
\draw [->,dashed] (trade2) -- (T2Down);
\end{tikzpicture}
\end{minipage}
\begin{minipage}{.5\textwidth}
  \centering
\begin{tikzpicture}
\def\scale{0.7}
\def\h{4*\scale}
\def\x{4*\scale}
\def\y{0*\scale}
\def\algo{0.2}

    \node[main node] (sDown) at (\x,0) {$Q'$};
    \node[main node] (T1Down) at (0,\y) {$Q$};
    \node[main node] (T2Down) at (2*\x,\y) {$\tilde{Q}$};
    \node[main node] (sUp) at (\x,\y+\h+\y) {$P'$};
    \node[main node] (T1Up) at (0,\y+\h) {$P$};
    \node[main node] (T2Up) at (2*\x,\y+\h) {$\tilde{P}$};
\def\amp{3pt}
\def\shift{0.3cm}
\def\size{\footnotesize}
\draw [thick,decorate,line width=0.5pt,decoration={brace,amplitude=\amp,mirror},xshift=0.0pt,yshift=-0cm] (T1Up.south) -- node[pos=0.5, yshift=-\shift] (lu) {\size$P\times P'$} (T1Up.south-|sUp) ;
\draw [thick,decorate,line width=0.5pt,decoration={brace,amplitude=\amp,mirror},xshift=0.0pt,yshift=-0cm] (T1Up.south-|sUp) -- node[pos=0.5, yshift=-\shift] (ru) {\size$ \tilde{P}\times P'$} (T2Up.south) ;
\draw [thick,decorate,line width=0.5pt,decoration={brace,amplitude=\amp,mirror},xshift=0.0pt,yshift=-0cm] (T1Down.north-|sDown) -- node[pos=0.5, yshift=\shift] (ld) {\size$Q\times Q'$} (T1Down.north) ;
\draw [thick,decorate,line width=0.5pt,decoration={brace,amplitude=\amp,mirror},xshift=0.0pt,yshift=-0cm] (T2Down.north) -- node[pos=0.5, yshift=\shift] (rd) {\size$\tilde{Q}\times Q'$} (T1Down.north-|sDown) ;

\draw [->] (sDown) -- (T1Down); 
\draw [->] (sDown) -- (T2Down);
\draw [->] (sUp) -- (T1Up); 
\draw [->] (sUp) -- (T2Up);

\draw [<->,dashed] (lu) -- (ld);
\draw [<->,dashed] (ru) -- (rd);
\end{tikzpicture}
\end{minipage}
	\captionof{figure}{\Cref{lem:comparison} implies well-definedness of $\otimes$.}
  \label{fig:well-definedness}
\end{figure}
\begin{proof}
	In order to fit it into the setting of \Cref{lem:comparison}, let the algorithms output a random variable independent of the input $y$. See \Cref{fig:well-definedness}. The input-output joint distributions are just product distributions, so by the comparison \cref{lem:comparison},
\[ T(P\times P', Q \times Q') \geqslant T(\tilde{P} \times P', \tilde{Q} \times Q').\]

	When $T(P,Q) = T(\tilde{P},\tilde{Q})$, we can apply the lemma in both directions and conclude that 
\[ T(P\times P', Q \times Q') = T(\tilde{P} \times P', \tilde{Q} \times Q').\]
	The proof is complete.
\end{proof}

Now that we have justified the definition of the composition tensor $\otimes$, \cref{lem:welldefined} can be written in a concise way:
\begin{equation}\label{eq:ordering}
	g_1\geqslant g_2\Rightarrow f\otimes g_1\geqslant f\otimes g_2.
\end{equation}
This is actually the second property we listed after the definition of $\otimes$.

For composition theorem, we prove the following two steps version:
\begin{lemma} \label{lem:two_steps}
	Suppose in a two-step composition, the two components $M_1:X\to Y, M_2:X\times Y \to Z$ satisfy
		\begin{enumerate}
			\item $M_1$ is $f$-DP;
			\item $M_2(\cdot,y):X\to Z$ is $g$-DP for each fixed $y\in Y$.
		\end{enumerate}
		Then the composition $M:X\to Y\times Z$ is $f\otimes g$-DP.
\end{lemma}
\usetikzlibrary{positioning,patterns,arrows,decorations.pathreplacing}
 \tikzset{main node/.style={},
            }
\begin{figure}[!htp]
\begin{minipage}{.5\textwidth}
  \centering
\begin{tikzpicture}
\def\scale{0.7}
\def\h{4*\scale}
\def\S{4*\scale}
\def\y{0*\scale}
\def\algo{0.2}

    \node[main node] (sDown) at (\S,0) {$y$};
    \node[main node] (T1Down) at (0,\y) {$M_2(S',y)$};
    \node[main node] (T2Down) at (2*\S,\y) {$Q'$};
    \node[main node] (sUp) at (\S,\y+\h+\y) {$y$};
    \node[main node] (T1Up) at (0,\y+\h) {$M_2(S,y)$};
    \node[main node] (T2Up) at (2*\S,\y+\h) {$Q$};
    \node[main node] (trade1) at (0,\y+\h*0.5) {{$\F(\cdots)$}};
    \node[main node] (trade2) at (\S*2,\y+\h*0.5) {$g$};

\draw [->] (sDown) -- (T1Down); 
\draw [->] (sDown) -- (T2Down);
\draw [->] (sUp) -- (T1Up); 
\draw [->] (sUp) -- (T2Up);
\draw [->,dashed] (trade1) -- (T1Up);
\draw [->,dashed] (trade1) -- (T1Down);
\draw [->,dashed] (trade2) -- (T2Up);
\draw [->,dashed] (trade2) -- (T2Down);
\end{tikzpicture}
\end{minipage}
\begin{minipage}{.5\textwidth}
  \centering
\begin{tikzpicture}
\def\scale{0.7}
\def\h{4*\scale}
\def\S{4*\scale}
\def\y{0*\scale}
\def\algo{0.2}

    \node[main node] (sDown) at (\S,0) {$M_1(S')$};
    \node[main node] (T1Down) at (0,\y) {$M_2\big(S',M_1(S')\big)$};
    \node[main node] (T2Down) at (2*\S,\y) {$Q'$};
    \node[main node] (sUp) at (\S,\y+\h+\y) {$M_1(S)$};
    \node[main node] (T1Up) at (0,\y+\h) {$M_2\big(S,M_1(S)\big)$};
    \node[main node] (T2Up) at (2*\S,\y+\h) {$Q$};
\def\amp{3pt}
\def\shift{0.3cm}
\def\size{\footnotesize}
\draw [thick,decorate,line width=0.5pt,decoration={brace,amplitude=\amp,mirror},xshift=0.0pt,yshift=-0cm] (T1Up.south) -- node[pos=0.5, yshift=-\shift] (lu) {\size$M(S)$} (T1Up.south-|sUp) ;
\draw [thick,decorate,line width=0.5pt,decoration={brace,amplitude=\amp,mirror},xshift=0.0pt,yshift=-0cm] (T1Up.south-|sUp) -- node[pos=0.5, yshift=-\shift] (ru) {\size$M_1(S)\otimes Q$} (T1Up.south-|T2Up) ;
\draw [thick,decorate,line width=0.5pt,decoration={brace,amplitude=\amp,mirror},xshift=0.0pt,yshift=-0cm] (T1Down.north-|sDown) -- node[pos=0.5, yshift=\shift] (ld) {\size$M(S')$} (T1Down.north) ;
\draw [thick,decorate,line width=0.5pt,decoration={brace,amplitude=\amp,mirror},xshift=0.0pt,yshift=-0cm] (T1Down.north-|T2Down) -- node[pos=0.5, yshift=\shift] (rd) {\size$M_1(S')\otimes Q'$} (T1Down.north-|sDown) ;

\draw [->] (sDown) -- (T1Down); 
\draw [->] (sDown) -- (T2Down);
\draw [->] (sUp) -- (T1Up); 
\draw [->] (sUp) -- (T2Up);
\draw [<->,dashed] (lu) -- (ld);
\draw [<->,dashed] (ru) -- (rd);
\end{tikzpicture}
\end{minipage}
\captionof{figure}{\Cref{lem:comparison} implies \Cref{lem:two_steps}.}
  \label{fig:comparison_proof}
\end{figure}

\begin{proof}[Proof of \Cref{lem:two_steps}]
	Let $Q,Q'$ be distributions such that $g = \F(Q,Q')$. Fix a pair of neighboring datasets $S$ and $S'$ and set everything as in \Cref{fig:comparison_proof}. The input $y$ is an element in the output space of $M_1$. Arrows to the left correspond to the mechanism $M_2$, while arrows to the right ignore the input $y$ and output $Q,Q'$ respectively.

	Here $f_1^y$ in \Cref{lem:comparison} is $T\big(M_2(S,y),M_2(S',y)\big)\geqslant g$, so the condition in \Cref{lem:comparison} checks. Consequently,
\begin{align*}
	\F\big(M(S),M(S')\big)&\geqslant \F\big(M_1(S)\times Q,M_1(S')\times Q'\big) &&\text{(\Cref{lem:comparison})}\\
	&= \F\big(M_1(S),M_1(S')\big) \otimes \F(Q,Q')&&\text{(Def. of $\otimes$)}\\
	&= \F\big(M_1(S),M_1(S')\big) \otimes g\\
	&\geqslant f\otimes g&&\text{(Privacy of $M_1$ and \eqref{eq:ordering})}
\end{align*}
The proof is complete.
\end{proof}

Now we prove \Cref{lem:comparison}. The proof is basically careful application of Neyman-Pearson Lemma \ref{thm:NPlemma}.

\begin{proof}[Proof of \Cref{lem:comparison}]
	In order to further simplify the notations, for $i=1,2$, let $\mu_i$ and $\mu_i'$ be the joint distributions $\big(P,K_i(P)\big)$ and $\big(P',K_i'(P')\big)$ respectively. Then $f_1 = \F(\mu_1,\mu_1'), f_2 = \F(\mu_2,\mu_2')$ and we need to show that the testing problem $\mu_1$ vs $\mu_1'$ is harder than $\mu_2$ vs $\mu_2'$.

	Consider the testing problem $\mu_1$ vs $\mu_1'$. For $\alpha\in[0,1]$, let $\phi_1:Y\times Z_1\to[0,1]$ be the optimal rejection rule at level $\alpha$. By definition of trade-off function, the power of this test is $1-f_1(\alpha)$.
	Formally,
	\[\E_{\mu_1}[\phi_1] = \alpha, \quad \E_{\mu_1'}[\phi_1] = 1-f_1(\alpha).\]
	It suffices to construct a rejection rule $\phi_2:Y\times Z_2\to[0,1]$ for the problem $\mu_2$ vs $\mu_2'$, at the same level $\alpha$ but with greater power, i.e.
	$$\E_{\mu_2} [\phi_2] = \alpha~~ \text{ and  }~~ \E_{\mu_2'} [\phi_2]\geqslant \E_{\mu_1'}[\phi_1] = 1-f_1(\alpha).$$
	If such $\phi_2$ exists, then by the sub-optimality of $\phi_2$ for the problem $\mu_2$ vs $\mu_2'$,
	\[1-f_2(\alpha)\geqslant \E_{\mu_2'} [\phi_2] \geqslant 1-f_1(\alpha),\]
	which is what we want.

	For $y\in Y$, let $\phi_1^y:Z_1\to[0,1]$ be the slice of $\phi_1$ at $y$, i.e. $\phi_1^y(z_1) = \phi_1(y,z_1)$. This is a rejection rule for the problem $K_1(y)$ vs $K_1'(y)$, sub-optimal in general. The type I error is
	\[\alpha^y := \E_{z_1\sim K_1(y)}[\phi_1^y(z_1)].\]
	The power is
	\[\E_{z_1\sim K_1'(y)}[\phi_1^y(z_1)]\leqslant 1-f_1^y(\alpha^y).\]
	The last inequality holds because $f_1^y=T\big(K_1(y),K_1'(y)\big)$ and that $\phi_1^y$ is sub-optimal for this problem.
	Let $\phi_2^y:Z_2\to[0,1]$ be the optimal rejection rule for the testing $K_2(y)$ vs $K_2'(y)$ at level $\alpha^y$. Construction of $\phi_2:Y\times Z_2\to[0,1]$ is simply putting together these slices $\phi_2^y$. Formally, $\phi_2(y,z_2) = \phi_2^y(z_2)$. Its level is $\alpha$ because $\alpha^y$ are averaged in terms of the same distribution $P$. More precisely,
	\begin{align*}
		\E_{\mu_2}[\phi_2] &= \E_{y\sim P}\big[\E_{z_2\sim K_2(y)}[\phi_2^y(z_2)]\big] &&\text{(Construction of $\phi_2$)}\\
		&= \E_{y\sim P}[\alpha^y]&&\text{($\phi_2^y$ has level $\alpha^y$)}\\
		&= \E_{y\sim P}\big[\E_{z_1\sim K_1(y)}[\phi_1^y(z_1)]\big]&&\text{(Def. of $\alpha^y$)}\\
		&= \E_{\mu_1}[\phi_1] = \alpha.
	\end{align*}
	Let's compute its power:
	\begin{align*}
		\E_{\mu_2'}[\phi_2] &= \E_{y\sim P'}\big[\E_{z_2\sim K_2'(y)}[\phi_2^y(z_2)]\big]\\
		&= \E_{y\sim P'}\big[1-f^y_2(\alpha^y)\big] &&\text{($\phi_2^y$ is optimal)}\\
		&\geqslant \E_{y\sim P'}\big[1-f^y_1(\alpha^y)\big]&&\text{($f^y_1\geqslant f^y_2$)}\\
		&\geqslant\E_{y\sim P'}\big[\E_{z_1\sim K_1'(y)}[\phi_1^y(z_1)]\big] &&\text{($\phi_1^y$ is sub-optimal)}\\
		&=\E_{\mu_1'}[\phi_1] = 1-f_1(\alpha).&&\text{(Optimality of $\phi_1$ for $\mu_1$ vs $\mu_1'$)}
	\end{align*}
	So $\phi_2$ constructed this way does have the desired level and power. The proof is complete.
\end{proof} 
\section{Omitted Proofs in \Cref{sec:composition-theorems}}
\label{app:CLT}


We first collect the basic properties of $\otimes$ listed in \Cref{sub:composition_theorem}.
\begin{proposition}\label{prop:tensor_properties}
	The product $\otimes$ defined in \Cref{def:product} has the following properties:
	\begin{enumerate}
	\setlength\itemsep{0.15em}
	\setcounter{enumi}{-1}
	\item The product $\otimes$ is well-defined.
	\item The product $\otimes$ is commutative and associative.
	\item If $g_1\geqslant g_2$, then $f\otimes g_1 \geqslant f\otimes g_2$.
	\item $f \otimes \Id = \Id \otimes f = f$.
	\item $(f\otimes g)^{-1} = f^{-1}\otimes g^{-1}$.

	\item For GDP, $G_{\mu_1} \otimes G_{\mu_2} \otimes \cdots \otimes G_{\mu_n} = G_{\mu}$, where $\mu = \sqrt{\mu_1^2+\cdots+\mu_n^2}$.
	\end{enumerate}
\end{proposition}
Property 0 and 2 are already proved in \Cref{app:self}. So we only prove 1,3,4,5 here.
\begin{proof}[Proof of Properties (1,3,4,5)]
	We will assume $f=T(P,P'), g =T(Q,Q')$ in the entire proof. The upshot is that
\[T(P,P')\otimes T(Q,Q') = T(P\times Q,P'\times Q').\]
\begin{enumerate}
	\item Commutativity:
		$$f\otimes g = T(P,P')\otimes T(Q,Q') = T(P\times Q,P'\times Q') \stackrel{(a)}{=} T(Q\times P,Q'\times P') =T(Q,Q') \otimes T(P,P') = g\otimes f.$$
		In step $(a)$, we switch the order of the components of the product, which obviously keeps the trade-off function unchanged.

		Associativity:
		Let $h=T(R,R')$.
		\begin{align*}
			(f\otimes g)\otimes h &= T(P\times Q,P'\times Q') \otimes T(R,R') =  T(P\times Q\times R,P'\times Q'\times R')\\
			f\otimes (g\otimes h) &= T(P,P')\otimes T(Q\times R,Q'\times R') =  T(P\times Q\times R,P'\times Q'\times R')
		\end{align*}
		So $(f\otimes g)\otimes h = f\otimes (g\otimes h)$.
	\item Let $R$ be an arbitrary degenerate distribution, i.e. $R$ puts mass 1 on a single point. Then $\Id = T(R,R)$ and
		\[f\otimes \Id = T(P\times R, P'\times R) = T(P,P') = f.\]
	\item
		By \Cref{lem:symmetry}, taking the inverse amounts to flipping the arguments of $T(\cdot,\cdot)$.
		\begin{align*}
			(f\otimes g)^{-1} = T(P'\times Q',P\times Q) = T(P',P)\otimes T(Q',Q) = f^{-1}\otimes g^{-1}.
		\end{align*}
	\item Let $\bmu=(\mu_1,\mu_2)\in\R^2$ and $I_2$ be the $2\times 2$ identity matrix. Then
		\begin{align*}
			G_{\mu_1}\otimes G_{\mu_2} &= T\big(\N(0,1),\N(\mu_1,1)\big)\otimes T\big(\N(0,1),\N(\mu_1,1)\big)\\
			&= T\big(\N(0,1)\times \N(0,1),\N(\mu_1,1)\times \N(\mu_2,1)\big)\\
			&= T\big(\N(0,I_2),\N(\bmu,I_2)\big)
		\end{align*}
		Again we use the invariance of trade-off functions under invertible transformations. $\N(0,I_2)$ is rotation invariant, So we can rotate $\N(\bmu,I_2)$ so that the mean is $(\sqrt{\mu_1^2+\mu_2^2},0)$. Continuing the calculation
		\begin{align*}
			G_{\mu_1}\otimes G_{\mu_2} &=T\big(\N(0,I_2),\N(\bmu,I_2)\big)\\
			&= T\big(\N(0,1)\times \N(0,1),\N(\sqrt{\mu_1^2+\mu_2^2},1)\times \N(0,1)\big)\\
			&= T\big(\N(0,1),\N(\sqrt{\mu_1^2+\mu_2^2},1)\big)\otimes T\big(\N(0,1),\N(0,1)\big)\\
			&= G_{\sqrt{\mu_1^2+\mu_2^2}}\otimes \Id\\
			&= G_{\sqrt{\mu_1^2+\mu_2^2}}.
		\end{align*}
\end{enumerate}
\end{proof}



	

The following proposition explains why our central limit theorems need $f_n$ to approach $\Id$.

\begin{proposition} \label{prop:trivial_limit}
	For any trade-off function $f$ that is not $\Id$,
	$$\lim_{n\to+\infty} f^{\otimes n}(\alpha)=0, \quad\forall\alpha\in(0,1].$$
	In fact, the convergence is exponentially fast.
\end{proposition}
\begin{proof}[Proof of \Cref{prop:trivial_limit}]
	For any trade-off function $f$, let $P,Q$ be probability measures such that $\F(P,Q)=f$. The existence is guaranteed by \Cref{prop:trade-off}. It is well-known that $1-\mathrm{TV}(P,Q)$ is the minimum sum of type I and type II error, namely,
	$$1-\mathrm{TV}(P,Q) = \min_{\alpha\in[0,1]} \alpha+f(\alpha).$$
	We claim that the following limit suffices to prove the theorem:
	\begin{equation}\label{eq:dtv}
		\lim_{n\to\infty} \mathrm{TV}(P^n,Q^{ n})=1.
	\end{equation}
	To see why it suffices, recall that by definition $\F(P^n,Q^{ n})=f^{\otimes n}$. Hence
	$$1-\mathrm{TV}(P^n,Q^{ n}) = \min_{\alpha\in[0,1]} \alpha+f^{\otimes n}(\alpha).$$
	Let $\alpha_n$ be the type i error that achieves minimum in the above equation, i.e.
	$$\alpha_n+f^{\otimes n}(\alpha_n) = 1-\mathrm{TV}(P^n,Q^{ n}).$$
	The total variation limit \eqref{eq:dtv} implies $\alpha_n\to0$ and $f^{\otimes n}(\alpha_n)\to0$. For each $n$, consider the piecewise linear function that interpolates $(0,1),(\alpha_n,f^{\otimes n}(\alpha_n))$ and $(1,0)$, which will be denoted by $h_n$. By the convexity of $f^{\otimes n}$ we know that $f^{\otimes n}\leqslant h_n$ in $[0,1]$. It suffices to show that $h_n(\alpha)\to 0,\forall \alpha\in(0,1]$. Since $\alpha_n\to0$, for large enough	$n$, $h_n(\alpha)$ is evaluated on the lower linear segment of $h_n$. So $h_n(\alpha)\leqslant h_n(\alpha_n)\leqslant f^{\otimes n}(\alpha_n) \to 0$. This yields the desired limit of $f^{\otimes n}$.\\
	Now we use Hellinger distance $H^2(P,Q) := \E_Q\big[(1-\sqrt{\frac{P}{Q}})^2\big]$ to show the total variation limit \eqref{eq:dtv}.\\
	An elementary inequality relating total variation and Hellinger distance is
	$$\frac{1}{2}H^2(P,Q)\leqslant \mathrm{TV}(P,Q)\leqslant H(P,Q).$$
	Another nice property of Hellinger distance is it tensorizes in the following sense:
	$$1-\frac{H^2(P^n,Q^{ n})}{2} = \Big(1-\frac{H^2(P,Q)}{2}\Big)^n.$$
	$f$ is not the diagonal $\alpha\mapsto1-\alpha$, so $P\neq Q$. Hence $\mathrm{TV}(P,Q)> 0$. By the second inequality in the sandwich bound, $H^2(P,Q)>0$. By the tensorization property, $H^2(P^n,Q^{ n})\to 2$. By the first inequality in the sandwich bound and that $\mathrm{TV}$ is bounded by 1 we have
	$$ \frac{1}{2} H^2(P^n,Q^{ n})\leqslant \mathrm{TV}(P^n,Q^{ n})\leqslant1.$$This shows $\mathrm{TV}(P^n,Q^{ n})\to 1$ and completes the proof.
\end{proof}

Now we set out the journey to prove the Berry-Esseen style central limit theorem \ref{thm:Berry}. We first restate the theorem.
\berryrep*
Our approach is to consider the log-likelihood ratio between the distributions of the composition mechanism on neighboring datasets. This log-likelihood ratio can be reduced to the sum of \textit{independent} components that each correspond to the log-likelihood ratio of a trade-off function in the tensor product. This reduction allows us to carry over the classical Berry--Esseen bound to Theorem~\ref{thm:Berry}.

As the very first step, let's better understand the functionals $\kl,\kappa_2$ and $\bar{\kappa}_3$ used in the statement of the theorem. We focus on symmetric $f$ with $f(0)=1$, although some of the following discussion generalizes beyond that subclass. Recall that
\begin{align*}
	\kl(f) &= -\int_0^1\log |f'(x)|\diff x\\
	\kappa_2(f)&=\int_0^1\big(\log |f'(x)|\big)^2\diff x\\
	\bar{\kappa}_3(f)&=\int_0^1\big|\log |f'(x)|+\kl(f)\big|^3\diff x
\end{align*}
First we finish the argument mentioned in \Cref{sub:a_berry_esseen_type_of_clt} that these functionals are well-defined and take values in $[0,+\infty]$. For $\kappa_2$ and $\bar{\kappa}_3$, as well as the non-central version $\kappa_3$, the argument is easy because the integrands are non-negative.

For $\kl$, the only possible singularities of the integrand is 0 and 1. If 1 is singular then $\log |f'(x)|\to -\infty$ near 1. This is okay because the functionals are allowed to take value $+\infty$. We need to rule out the case when 0 is a singularity and $\int^{\epsilon}_0\log |f'(x)|\diff x=+\infty$. That cannot happen because $\log |f'(x)|\leqslant |f'(x)|-1$ and $|f'(x)| = -f'(x)$ is integrable in $[0,1]$ as it is the derivative of $-f$, an absolute continuous function. 
Non-negativity of $\kl$ follows from Jensen's inequality.

In the discussion of \Cref{prop:fdiv}, we showed that $\kl(T(P,Q)) = D_{\mathrm{KL}}(P\|Q)$. This explains the name of this functional. In fact, $\kappa_2$ also corresponds to a divergence called \textit{exponential divergence} (\cite{eguchi1985differential}).

We introduce a notation that will be useful in the calculation below. For a trade-off function $f$, let $Df$ be a function with the following expression:
\[Df(x) = |f'(1-x)| = -f'(1-x).\]
In fact, this is the density introduce in the proof of \Cref{prop:trade-off}.

By a simple change of variable, the three functionals can be re-written as
\begin{align*}
	\kl(f) &= -\int_0^1\log Df(x)\diff x\\
	\kappa_2(f)&=\int_0^1\big(\log Df(x)\big)^2\diff x\\
	\bar{\kappa}_3(f)&=\int_0^1\big|\log Df(x)+\kl(f)\big|^3\diff x
\end{align*}
The following ``shadows'' of the above functionals will appear in the proof:
\begin{align*}
	\lk(f) &:= \int_0^1Df(x)\log Df(x)\diff x\\
	\tilde{\kappa}_2(f)&:=\int_0^1Df(x)\big(\log Df(x)\big)^2\diff x\\
	\tilde{\kappa}_3(f)&:=\int_0^1Df(x)\big|\log Df(x)-\lk(f)\big|^3\diff x
\end{align*}
These functionals are also well-defined on $\T$ and take values in $[0,+\infty]$. The argument is similar to that of $\kl, \kappa_2$ and $\bar{\kappa}_3$.


The following calculations turn out to be useful in the proof.
\begin{proposition}\label{prop:functionals}
	Suppose $f\in\T^S$ and $f(0)=1$. Then
	\begin{align*}
		\kl(f) &= \lk(f)\\
		\kappa_2(f)&= \tilde{\kappa}_2(f)\\
		\bar{\kappa}_3(f)&=\tilde{\kappa}_3(f).
	\end{align*}
\end{proposition}
\begin{proof}
	Our approach, taking $\kappa_2$ as example, is to show $\kappa_2(f^{-1}) = \tilde{\kappa}_2(f)$. By definition of symmetry, $f^{-1} = f$ and hence the desired result follows. First observe for $f\in\T^S$ with $f(0)=1$, $f^{-1}$ agrees with the ordinary function inverse, hence we can apply calculus rule as follows:
	\[Df^{-1}(x) = -\frac{\diff f^{-1}}{\diff x}(1-x) = \frac{-1}{f'(f^{-1}(1-x))}.\]
	We only prove $\kappa_2(f^{-1})= \tilde{\kappa}_2(f)$ here and the other two identities can be proved similarly.
	\begin{align*}
		\kappa_2(f^{-1}) &= \int_0^1\big(\log Df^{-1}(x)\big)^2\diff x \\
		&= \int_0^1\big(-\log\big[-f'(f^{-1}(1-x))\big]\big)^2\diff x\\
		&= \int_0^1\log^2\big[-f'(f^{-1}(1-x))\big]\diff x
	\end{align*}
	Let $y=f^{-1}(1-x)$, then $f'(y)\diff y=-\diff x$, and $x=0$ corresponds to $y=0$, $x=1$ corresponds to $y=1$.
	\begin{align*}
		\kappa_2(f^{-1})
		&= \int_0^1\log^2 [-f'(y)]\cdot\big(-f'(y)\big)\diff y&& (y=f^{-1}(1-x))\\
		&= \int_0^1\log^2 [-f'(1-z)]\cdot\big(-f'(1-z)\big)\diff z&& (z=1-y)\\
		&= \int_0^1Df(z)\big(\log Df(z)\big)^2\diff z\\
		&= \tilde{\kappa}_2(f).
	\end{align*}
\end{proof}
We remark that by properly extending the definition of the shadow functionals, identities like $\kl(f^{-1}) = \lk(f)$ holds for general trade-off function $f$.


Before we finally start the proof, let's recall Berry-Esseen theorem for random variables. Suppose we have $n$ independent random variables $X_1,\ldots, X_n$ with $\E X_i = \mu_i, \Var X_i = \sigma_i^2, \E|X_i-\mu_i|^3 = \rho_i^3$. Consider the normalized random variable
$$S_n := \frac{\sum_{i=1}^n X_i-\mu_i}{\sqrt{\sum_{i=1}^n \sigma^2_i}}.$$
Denote its cdf by $F_n$. Then
\begin{theorem}[Berry-Esseen]\label{thm:BerryRV}
There exists a universal constant $C>0$ such that
	\[\sup_{x\in\R}|F_n(x)-\Phi(x)|\leqslant C\cdot \frac{\sum_{i=1}^n \rho_i^3}{\big(\sum_{i=1}^n\sigma_i^2\big)^{\frac{3}{2}}}.\]
\end{theorem}
To the best of our knowledge, the best $C$ is 0.5600 due to \cite{shevtsova2010improvement}.

\bigskip

Now we proceed to the proof of \Cref{thm:Berry}.
\begin{proof}[Proof of \Cref{thm:Berry}]
For simplicity let
\[\f := f_{1}\otimes f_{2} \otimes \cdots \otimes f_{n}.\]

First let's find distributions $P_0$ and $P_1$ such that $\F(P_0,P_1)=\f$.

First, by symmetry, if $f_i(0)<1$, then $f_i'(x)=0$ in some interval $[1-\ep,1]$ for some $\ep>0$, which yields $\kl(f_i)=+\infty$. So we can assume $f_i(0)$ for all $i$.

Recall that $Df_{i}(x) = -f'_{i}(1-x)$. Let $P$ be the uniform distribution on $[0,1]$ and $Q_{i}$ be the distribution supported on $[0,1]$ with density $Df_{i}$. These are the distributions constructed in the proof of \Cref{prop:trade-off}. Since $f_i$ are all symmetric and $f_i(0)=1$, the supports of $P$ and all $Q_{i}$ are all exactly $[0,1]$, and we have $\F(P,Q_{i})=f_{i}$. Hence by definition $\f = \F(P^{ n},Q_{1}\times\cdots\times Q_{n})$.

Now let's study the hypothesis testing problem $P^{ n}$ vs $Q_{1}\times\cdots\times Q_{n}$. Let
\[L_{i}(x) := \log \frac{\diff Q_{i}}{\diff P}(x) = \log Df_{i}(x)\]
be the log likelihood ratio. 
Since both hypotheses are product distributions, Neyman-Pearson lemma implies that the optimal rejection rules of this testing problem must be a threshold function of the quantity $\sum_{i=1}^n L_{i}$.
We need to study $\sum_{i=1}^n L_{i}(x_i)$ under both the null and the alternative hypothesis, i.e. when $(x_1,\ldots, x_n)$ comes from $P^{ n}$ and $Q_{1}\times\cdots\times Q_{n}$. From here we implement the following plan: first find the quantities that exhibit central limit behavior, then express $\alpha$ and $\f(\alpha)$ in terms of these quantities.

For further simplification, let
$$T_n := \sum_{i=1}^n L_{i}.$$
As we turn off the $x_i$ notation, we should bear in mind that $T_n$ has different distributions under $P^{ n}$ and $Q_{1}\times\cdots\times Q_{n}$, but it is an independent sum in both cases.

In order to find quantities with central limit behavior, it suffices to normalize $T_n$ under both distributions. The mysterious functionals we introduced are specifically designed for this purpose.
\begin{align*}
	\E_P[L_{i}] &= \int_0^1 \log Df_{i}(x_i)\diff x_i = -\kl(f_{i}),\\
	\E_{Q_{i}}[L_{i}] &= \int_0^1 Df_{i}(x_i)\log Df_{i}(x_i)\diff x_i = \lk(f_{i}) = \kl(f_{i}).
\end{align*}
In the last step we used \Cref{prop:functionals}. With the bold vector notation,
\begin{align*}
	\E_{P^n}[T_n]  &= \sum_{i=1}^n -\kl(f_{i}) = -\|\boldsymbol{\kl}\|_1,\\
	\E_{Q_{1}\times\cdots\times Q_{n}}[T_n] &= \sum_{i=1}^n \kl(f_{i})= \|\boldsymbol{\kl}\|_1.
\end{align*}
Similarly for the variances:
\begin{align*}
	\Var_{P}[L_{i}] &=\E_P[L_{i}^2] - \big(\E_P[L_{i}]\big)^2 = \kappa_2(f_{i}) - \kl^2(f_{i})
	,\\
	\Var_{Q_{i}}[L_{i}] &= \E_{Q_{i}}[L_{i}^2] - \big(\E_{Q_{i}}[L_{i}]\big)^2 = \tilde{\kappa}_2(f_{i}) - \lk^2(f_{i})= \kappa_2(f_{i}) - \kl^2(f_{i}).
\end{align*}
\[
	\Var_{P^n}[T_n] =\Var_{Q_{1}\times\cdots\times Q_{n}}[T_n] = \sum_{i=1}^n \kappa_2(f_{i}) - \kl^2(f_{i}) = \|\boldsymbol{\kappa_2}\|_1 - \|\boldsymbol{\kl}\|_2^2.
\]
In order to apply Berry-Esseen theorem (for random variables) we still need the centralized third moments:
\begin{align*}
	\E_P|L_i-\E_P[L_i]|^3 &=\int_0^1\big|\log Df_i(x)+\kl(f_i)\big|^3\diff x = \bar{\kappa}_3(f_i),\\
	\E_{Q_i}|L_i-\E_{Q_i}[L_i]|^3 &=\int_0^1Df_i(x)\big|\log Df_i(x)-\lk(f_i)\big|^3\diff x = 
	\tilde{\kappa}_3(f_i)= \bar{\kappa}_3(f_i).
\end{align*}
Let $F_n$ be the cdf of $\frac{T_n + \|\boldsymbol{\kl}\|_1}{\sqrt{\|\boldsymbol{\kappa_2}\|_1 - \|\boldsymbol{\kl}\|_2^2}}$ under $P^{ n}$, and $\tilde{F}^{(n)}$ be the cdf of $\frac{T_n-\|\boldsymbol{\kl}\|_1}{\sqrt{\|\boldsymbol{\kappa_2}\|_1-\|\boldsymbol{\kl}\|_2^2}}$ under $Q_{1}\times\cdots\times Q_{n}$. By Berry-Esseen Theorem \ref{thm:BerryRV},
\begin{align}\label{eq:closetonormal}
	\sup_{x\in\R}|F_n(x)-\Phi(x)|
	&\leqslant C\cdot\frac{\|\boldsymbol{\bar{\kappa}_3}\|_1}{\big(\|\boldsymbol{\kappa_2}\|_1 - \|\boldsymbol{\kl}\|_2^2\big)^{\frac{3}{2}}}=\gamma
\end{align}
and similarly $\sup_{x\in\R}|\tilde{F}^{(n)}(x)-\Phi(x)|\leqslant\gamma$.


So we find the quantities that exhibit central limit behavior. Now let's relate them with $\f$. Consider the testing problem $(P^{ n},Q_{1}\times\cdots\times Q_{n})$. For a fixed $\alpha\in[0,1]$, let the optimal rejection rule (potentially randomized) at level $\alpha$ be $\phi$.
By Neyman-Pearson lemma, $\phi$ must be a thresholding on $T_n$. An equivalent form that highlights the central limit behavior is the following:
$$\phi=
\left\{
\begin{array}{ll}
1, 		& \frac{T_n + \|\boldsymbol{\kl}\|_1}{\sqrt{\|\boldsymbol{\kappa_2}\|_1 - \|\boldsymbol{\kl}\|_2^2}}>t, \\
p, & \frac{T_n + \|\boldsymbol{\kl}\|_1}{\sqrt{\|\boldsymbol{\kappa_2}\|_1 - \|\boldsymbol{\kl}\|_2^2}}=t,\\
0, & \frac{T_n + \|\boldsymbol{\kl}\|_1}{\sqrt{\|\boldsymbol{\kappa_2}\|_1 - \|\boldsymbol{\kl}\|_2^2}}<t
\end{array}
\right.
$$
Here $t\in\R\cup\{\pm\infty\}$ and $p\in[0,1]$ are parameters uniquely determined by the condition $\E_{P^n}[\phi]=\alpha$. With this form $\E_{P^n}[\phi]$ can be easily spelled out in terms of $F_n$:
\begin{align*}
	\E_{P^n}[\phi] &= P^n\Big[\frac{T_n + \|\boldsymbol{\kl}\|_1}{\sqrt{\|\boldsymbol{\kappa_2}\|_1 - \|\boldsymbol{\kl}\|_2^2}}>t\Big] + p \cdot P^n\Big[\frac{T_n + \|\boldsymbol{\kl}\|_1}{\sqrt{\|\boldsymbol{\kappa_2}\|_1 - \|\boldsymbol{\kl}\|_2^2}}=t\Big]\\
	&= 1-F_n(t) + p \cdot [F_n(t)-F_n(t^-)].
\end{align*}
Here $F_n(t^-)$ is the left limit of the function $F_n$ at $t$.
Simple algebra yields
\[
	1-\alpha = 1- \E_{P^n}[\phi]= (1-p)F_n(t)+pF_n(t^-)
\]
and consequently the inequality
\[
	F_n(t^-)\leqslant 1-\alpha\leqslant F_n(t).
\]
For $\E_{Q_{1}\times\cdots\times Q_{n}}[\phi]$ it is helpful to introduce another letter $\tau := t-\mu$. In the theorem statement $\mu$ was defined to be $\frac{2\|\boldsymbol{\kl}\|_1}{\sqrt{\|\boldsymbol{\kappa_2}\|_1 - \|\boldsymbol{\kl}\|_2^2}}$
so we have the equivalence
\begin{equation}\label{eq:equiv}
	\frac{T_n + \|\boldsymbol{\kl}\|_1}{\sqrt{\|\boldsymbol{\kappa_2}\|_1 - \|\boldsymbol{\kl}\|_2^2}}>t
	\Leftrightarrow \frac{T_n-\|\boldsymbol{\kl}\|_1}{\sqrt{\|\boldsymbol{\kappa_2}\|_1-\|\boldsymbol{\kl}\|_2^2}}
	>
	\tau.
\end{equation}
With this extra notation we have
\begin{align*}
	1-\f(\alpha)
	=&\,\, \E_{Q_{1}\times\cdots\times Q_{n}}[\phi]\\
	=&\,\, Q_{1}\times\cdots\times Q_{n}\Big[\frac{T_n + \|\boldsymbol{\kl}\|_1}{\sqrt{\|\boldsymbol{\kappa_2}\|_1 - \|\boldsymbol{\kl}\|_2^2}}>t\Big]+\\
	&\,\, p \cdot Q_{1}\times\cdots\times Q_{n}\Big[\frac{T_n + \|\boldsymbol{\kl}\|_1}{\sqrt{\|\boldsymbol{\kappa_2}\|_1 - \|\boldsymbol{\kl}\|_2^2}}=t\Big]
	&& \text{(Def. of $\phi$)}\\
	=&\,\, Q_{1}\times\cdots\times Q_{n}\Big[\frac{T_n-\|\boldsymbol{\kl}\|_1}{\sqrt{\|\boldsymbol{\kappa_2}\|_1-\|\boldsymbol{\kl}\|_2^2}}>\tau\Big]+\\
	&\,\, p \cdot Q_{1}\times\cdots\times Q_{n}\Big[\frac{T_n-\|\boldsymbol{\kl}\|_1}{\sqrt{\|\boldsymbol{\kappa_2}\|_1-\|\boldsymbol{\kl}\|_2^2}}=\tau\Big]
	&& \text{ (By \eqref{eq:equiv}) }\\
	=&\,\, 1-\tilde{F}^{(n)}(\tau) + p \cdot [\tilde{F}^{(n)}(\tau)-\tilde{F}^{(n)}(\tau^-)].
\end{align*}
Similar algebra as before yields
\[\f(\alpha) = (1-p)\tilde{F}^{(n)}(\tau)+p\tilde{F}^{(n)}(\tau^-)\]
and hence
\[
\tilde{F}^{(n)}(\tau^-)\leqslant \f(\alpha)\leqslant \tilde{F}^{(n)}(\tau).
\]
So far we have
\begin{align}
	F_n(t^-)\leqslant 1-\alpha\leqslant F_n(t),\\
\tilde{F}^{(n)}(\tau^-)\leqslant \f(\alpha)\leqslant \tilde{F}^{(n)}(\tau).\label{eq:sana}
\end{align}
In \eqref{eq:closetonormal} we show $F_n$ and $\tilde{F}^{(n)}$ are within distance $\gamma$ to the cdf of standard normal, so
\[
	\Phi(t)-\gamma\leqslant F_n(t^-)\leqslant 1-\alpha\leqslant F_n(t)\leqslant \Phi(t)+\gamma
\]
and hence
\begin{equation}\label{eq:mina}
	\Phi^{-1}(1-\alpha-\gamma)\leqslant t\leqslant \Phi^{-1}(1-\alpha+\gamma).
\end{equation}
Using \eqref{eq:sana} and \eqref{eq:mina}, 
\begin{align*}
	\f(\alpha)&\leqslant\tilde{F}^{(n)}(\tau)\\
	&\leqslant \Phi(\tau)+\gamma\\
	&= \Phi(t-\mu)+\gamma\\
	&\leqslant\Phi(\Phi^{-1}(1-\alpha+\gamma)-\mu)+\gamma\\
	&=G_{\mu}(\alpha-\gamma)+\gamma
\end{align*}
Similarly we can show that $\f(\alpha)\geqslant G_{\mu}(\alpha+\gamma)-\gamma$. The proof is now complete.
\end{proof}

Next we prove the asymptotic version. Recall that our goal is
\asymprep*
\begin{proof}[Proof of \Cref{thm:CLT}]
	We will first construct pointwise convergence $f_{n1}\otimes f_{n2} \otimes \cdots \otimes f_{nn}\to G_{2K/s}$ and then conclude uniform convergence from a general theorem.

	Apply Berry-Esseen Theorem \ref{thm:Berry} to the $n$-th row of the triangular array and we have
	\begin{equation*}\label{eq:sandwich}
		G_{\mu_n}(\alpha+\gamma_n)-\gamma_n\leqslant f_{n1}\otimes f_{n2} \otimes \cdots \otimes f_{nn}(\alpha)\leqslant G_{\mu_n}(\alpha-\gamma_n)+\gamma_n.
	\end{equation*}
	Here $\mu_n$ and $\gamma_n$ are the counterparts of $\mu$ and $\gamma$ defined in \Cref{thm:Berry} when applied to $f_{n1},\ldots,f_{nn}$. Namely,
	\begin{align*}
		\mu_n=&\,\,\frac{2\|\boldsymbol{\kl}^{(n)}\|_1}{\sqrt{\|\boldsymbol{\kappa_2}^{(n)}\|_1 - \|\boldsymbol{\kl}^{(n)}\|_2^2}},\\
		\gamma_n=&\,\,0.56\cdot\frac{\|\boldsymbol{\bar{\kappa}_3}^{(n)}\|_1}{\big(\|\boldsymbol{\kappa_2}^{(n)}\|_1 - \|\boldsymbol{\kl}^{(n)}\|_2^2\big)^{\frac{3}{2}}}
	\end{align*}
	Here the bold vector notation with a superscript $(n)$ denotes the vector for the $n$-th row. For example, $\boldsymbol{\kl}^{(n)} = \big(\kl(f_{n1}),\ldots, \kl(f_{nn})\big)$.

	By the sandwich inequality, pointwise convergence of $f_{n1}\otimes f_{n2} \otimes \cdots \otimes f_{nn}$ follows from the two limits
	\begin{equation}\label{eq:wendy}
		G_{\mu_n}(\alpha+\gamma_n)-\gamma_n\to G_{2K/s}(\alpha), \quad G_{\mu_n}(\alpha-\gamma_n)+\gamma_n\to G_{2K/s}(\alpha).
	\end{equation}
	To prove these, let's first show $\gamma_n\to 0$ and $\mu_n\to2K/s$.

	Reformulating the assumptions in bold vector notations, we have
	\[||\boldsymbol{\kl}^{(n)}||_1\to K,\quad ||\boldsymbol{\kl}^{(n)}||_\infty\to 0,\quad ||\boldsymbol{\kappa_2}^{(n)}||_1\to s^2,\quad ||\boldsymbol{\kappa_3}^{(n)}||_1\to 0.\]
	In addition to these, it suffices to show
	\begin{equation}\label{eq:ning}
		\|\boldsymbol{\kl}^{(n)}\|_2^2\to0 ~~\text{ and }~~ \|\boldsymbol{\bar{\kappa}_3}^{(n)}\|_1\to0.
	\end{equation}
	For the first half, notice that $\|\boldsymbol{\kl}^{(n)}\|_2^2=\langle \boldsymbol{\kl}^{(n)}, \boldsymbol{\kl}^{(n)}\rangle\leqslant \|\boldsymbol{\kl}^{(n)}\|_\infty\cdot \|\boldsymbol{\kl}^{(n)}\|_1\to0$. In fact, $\|\boldsymbol{\kl}^{(n)}\|_\infty\to0$ is not only sufficient but also necessary, because $\|\boldsymbol{\kl}^{(n)}\|_\infty\leqslant\|\boldsymbol{\kl}^{(n)}\|_2$.

	Next we use the assumptions to show $\|\boldsymbol{\bar{\kappa}_3}^{(n)}\|_1\to0$.
	We need a lemma
	\begin{lemma} \label{lem:cubicmoments}
	For a trade-off function $f$,
		$$\bar{\kappa}_3(f)\leqslant \kappa_3(f) + 3\kl(f)\cdot\kappa_2(f)+3\kl^2(f)\cdot \sqrt{\kappa_2(f)}+\kl^3(f).$$
	\end{lemma}
	\begin{proof}[Proof of \Cref{lem:cubicmoments}]
		\begin{align*}
			\bar{\kappa}_3(f)=&\phantom{+}\int_0^1\big|\log Df(x)+\kl(f)\big|^3\diff x\\
			\leqslant&\phantom{+} \int_0^1\big(\big|\log Df(x)\big|+\big|\kl(f)\big|\big)^3\diff x\\
			\leqslant&\phantom{+} \int_0^1\big|\log Df(x)\big|^3\diff x+3\kl(f)\cdot\int_0^1\big|\log Df(x)\big|^2\diff x\\
			&+3\kl^2(f)\cdot\int_0^1\big|\log Df(x)\big|\diff x + \kl^3(f)\\
			\leqslant&\phantom{+} \kappa_3(f) + 3\kl(f)\cdot\kappa_2(f)+3\kl^2(f)\cdot \sqrt{\kappa_2(f)}+\kl^3(f).
		\end{align*}
		In the last step we used Jensen's inequality.
	\end{proof}
	Apply \Cref{lem:cubicmoments} to each $f_{ni}$ and sum them up:
	\begin{align*}
		\|\boldsymbol{\bar{\kappa}_3}^{(n)}\|_1 
		&\leqslant \|\boldsymbol{\kappa_3}^{(n)}\|_1 + 3\textstyle\sum_i \kl(f_{ni})\cdot \kappa_2(f_{ni})+ 3\textstyle\sum_i \kl(f_{ni})\cdot \sqrt{\kappa_2(f_{ni})} \cdot \kl(f_{ni}) + \textstyle\sum_i \kl(f_{ni})\cdot \kl^2(f_{ni}).
	\end{align*}
	Using $|\sum a_i b_i|\leqslant |\sum a_i| \cdot \max |b_i|$ and Cauchy-Schwarz inequality yields
	\begin{align*}
		\|\boldsymbol{\bar{\kappa}_3}^{(n)}\|_1 
		&\leqslant \|\boldsymbol{\kappa_3}^{(n)}\|_1 + 3 \|\boldsymbol{\kl}^{(n)}\|_\infty\cdot \|\boldsymbol{\kappa_2}^{(n)}\|_1  + 3\|\boldsymbol{\kl}^{(n)}\|_\infty \cdot \Big(\textstyle\sum_i \sqrt{\kappa_2(f_{ni})} \cdot \kl(f_{ni})\Big) + \|\boldsymbol{\kl}^{(n)}\|_\infty^2 \cdot \|\boldsymbol{\kl}^{(n)}\|_1\\
		&\leqslant \|\boldsymbol{\kappa_3}^{(n)}\|_1 + 3 \|\boldsymbol{\kl}^{(n)}\|_\infty\cdot \|\boldsymbol{\kappa_2}^{(n)}\|_1  + 3\|\boldsymbol{\kl}^{(n)}\|_\infty \cdot \sqrt{\|\boldsymbol{\kappa_2}^{(n)}\|_1 \cdot \|\boldsymbol{\kl}^{(n)}\|_2^2} + \|\boldsymbol{\kl}^{(n)}\|_\infty^2 \cdot \|\boldsymbol{\kl}^{(n)}\|_1.
	\end{align*}
	By the assumptions
	\[||\boldsymbol{\kl}^{(n)}||_1\to K,\quad ||\boldsymbol{\kl}^{(n)}||_\infty\to 0,\quad ||\boldsymbol{\kappa_2}^{(n)}||_1\to s^2,\quad ||\boldsymbol{\kappa_3}^{(n)}||_1\to 0\]
	and $\|\boldsymbol{\kl}^{(n)}\|_2^2\to0$ which we just proved,
	it's easy to see that all four terms goes to 0 as $n$ goes to infinity.

	The two limits \eqref{eq:ning} we have just proved imply $\mu_n\to2K/s$ and $\gamma_n\to 0$. Given these, convergence \eqref{eq:wendy} is easy once we notice that $G_{\mu}(\alpha)=\Phi(\Phi^{-1}(1-\alpha)-\mu)$ is continuous in both $\alpha$ and $\mu$.

	If the readers are concerned with $1-\alpha-\gamma_n$ exceeding $[0,1]$, then observe that when $\alpha\in(0,1)$, $1-\alpha-\gamma_n$ eventually ends up in $(0,1)$ where $\Phi^{-1}$ is well-defined and continuous. So the only concern is at 0 and 1. If $\alpha=0$, $\Phi^{-1}(1-\alpha-\gamma_n)\to+\infty$ so $G_{\mu_n}(0+\gamma_n)-\gamma_n\to1 = G_{2K/s}(0)$. A similar argument works for $\alpha=1$.

	Anyway, we have shown pointwise convergence. Uniform convergence is again a direct consequence of \Cref{lem:uniform}.
	The proof is now complete.
\end{proof}

Next we explain the effect of tensoring $f_{0,\delta}$.

	\begin{equation}
	f\otimes f_{0,\delta}(\alpha) =
		\left\{
		\begin{array}{ll}
		(1-\delta)\cdot f(\frac{\alpha}{1-\delta}), 		& 0\leqslant \alpha \leqslant 1-\delta \\
		0, & 1-\delta\leqslant \alpha\leqslant 1.
		\end{array}
		\right.\tag{\ref{prop:ruibbit}}
	\end{equation}
\begin{proof}[Proof of \Cref{prop:ruibbit}]
	First, $f_{0,\delta}$ is the trade-off function of two uniform distributions $f_{0,\delta} = T\big(U[0,1],U[\delta,1+\delta]\big)$.
	To see this, observe that any optimal test $\phi$ for $U[0,1]$ vs $U[\delta,1+\delta]$ 
	must have the following form:
	$$\phi(x)=
	\left\{
	\begin{array}{ll}
	1, 		&  x\in(1,1+\delta]\\
	p, & x\in[\delta,1],\\
	0, & x\in[0,\delta)
	\end{array}
	\right.
	$$
	That is, we know it must be from $U[0,1]$ if we see something in $[0,\delta)$, and must be from $U[\delta,1+\delta]$ if we see something in $(1,1+\delta]$. Otherwise the only thing we can do is random guessing. It's easy to see that the errors of such $\phi$ linearly interpolates between $(0,1-\delta)$ and $(1-\delta,0)$, i.e. type I and type II error add up to $1-\delta$. On the other hand, by definition, $f_{0,\delta}(\alpha) = \max{\{1-\delta-\alpha,0\}}$. So they indeed agree with each other.

	Now suppose $f=T(P,Q)$. By definition of tensor	product, $f\otimes f_{0,\delta} = T(P\times U[0,1], Q\times U[\delta,1+\delta])$. If the optimal test for $P$ vs $Q$ at level $\alpha$ is $\phi_\alpha$, then an optimal test for $P\times U[0,1]$ vs $Q\times U[\delta,1+\delta]$ must be of the following form:
	$$\tilde{\phi}_\alpha(\omega, x)=
	\left\{
	\begin{array}{ll}
	1, 		&  x\in(1,1+\delta]\\
	\phi_\alpha(\omega), & x\in[\delta,1],\\
	0, & x\in[0,\delta)
	\end{array}
	\right.
	$$
	The errors are
	\begin{align*}
		\E_{P\times U[0,1]}[\tilde{\phi}_\alpha] &= P\big[x\in(1,1+\delta]\big]+ P\big[x\in[\delta,1]\big]\cdot \E_{P}[\phi_\alpha(\omega)] \\
		&= 0+(1-\delta)\alpha=(1-\delta)\alpha\\
		1-\E_{Q\times U[\delta,1+\delta]}[\tilde{\phi}_\alpha] &= 1-P\big[x\in(1,1+\delta]\big]- P\big[x\in[\delta,1]\big]\cdot \E_{Q}[\phi_\alpha(\omega)] \\
		&=1- \delta- (1-\delta)\big(1-f(\alpha)\big) = (1-\delta)f(\alpha)
	\end{align*}
	This completes the proof.
\end{proof}

\DPCLTrep*
\begin{proof}[Proof of \Cref{thm:DPCLT}]
As in the main body, we first apply rules $f_{\ep,\delta} = f_{\ep,0}\otimes f_{0,\delta}$ and $f_{0,\delta_1}\otimes f_{0,\delta_2} = f_{0,1-(1-\delta_1)(1-\delta_2)}$ to get
\begin{align*}
	f_{\ep_{n1},\delta_{n1}}\otimes \cdots \otimes f_{\ep_{nn},\delta_{nn}}
	&= \big(f_{\ep_{n1},0}\otimes \cdots \otimes f_{\ep_{nn},0}\big)\otimes \big(f_{0,\delta_{n1}}\otimes \cdots \otimes f_{0,\delta_{nn}}\big)\\
	&= \big(\underbrace{f_{\ep_{n1},0}\otimes \cdots \otimes f_{\ep_{nn},0}}_{f^{(n)}}\big)\otimes f_{0,\delta^{(n)}}
\end{align*}
with $\delta^{(n)} = 1-\prod_{i=1}^n(1-\delta_{ni})$.
For the second factor, let's first prove the limit $\delta^{(n)}\to 1-\e^{-\delta}$.
Changing the product into sum, we have
\[\log (1-\delta^{(n)}) = \textstyle\sum_{i=1}^n\log(1-\delta_{ni})\]
The limit almost follows from the Taylor expansion $\log (1+x) = x+o(x)$, but we need to be a little more careful as the number of summation terms also goes to infinity. 
Since $\max_{1\leqslant i\leqslant n} \delta_{ni}\to 0$, we can assume for large $n$, $\delta_{ni}<r$ for some $r$ such that when $|x|<r$, the following Taylor expansion holds for some constant $C$:
\[|\log(1-x)+x|\leqslant Cx^2.\]
With this,
\begin{align*}
	\big|\textstyle\sum_{i=1}^n \log(1-\delta_{ni})+\delta_{ni}\big|
	\leqslant C\cdot\textstyle\sum_{i=1}^n\delta_{ni}^2\leqslant C\cdot\max_{i} \delta_{ni}\cdot \textstyle\sum_i \delta_{ni} \to 0.
\end{align*}
Therefore, $\log (1-\delta^{(n)}) = \textstyle\sum_{i=1}^n\log(1-\delta_{ni})$ has the same limit as $\sum_{i=1}^n\delta_{ni}$. In other words, $\log (1-\delta^{(n)})\to \delta$, or equivalently, $\delta^{(n)}\to 1-\e^{-\delta}$.

For a fixed $x\in[0,1]$, $f_{0,\delta^{(n)}}(x) = \max\{0,1-\delta^{(n)}-x\}$ is continuous in $\delta^{(n)}$. Hence we have the pointwise limit $f_{0,\delta^{(n)}}\to f_{0,1-\e^{-\delta}}$.

For the first factor $f^{(n)} = f_{\ep_{n1},0}\otimes \cdots \otimes f_{\ep_{nn},0}$, we will apply \Cref{thm:CLT}. Let's check the conditions.

By the continuity of the function $x\mapsto x\tanh\tfrac{x}{2}$ at 0, the assumption $\max_{1\leqslant i\leqslant n} \ep_{ni}\to 0$ implies
\[\max_{1\leqslant i \leqslant n}\kl(f_{ni})=\max_{1\leqslant i \leqslant n} \ep_{ni}\tanh\tfrac{\ep_{ni}}{2}\to0.\]
Next, we show
\[\sum_{i=1}^n \kl(f_{ni})=\sum_{i=1}^n \ep_{ni}\tanh\frac{\ep_{ni}}{2}\to K=\frac{\mu^2}{2}.\]
Preparing for the same Taylor expansion trick, let $n$ be large enough so that Taylor expansion $|\tanh x -x|\leqslant Cx^2$ applies to all $\delta_{ni}$.
\begin{align*}
	\bigg|\sum_{i=1}^n \ep_{ni}\tanh\frac{\ep_{ni}}{2} -\sum_{i=1}^n \frac{\ep_{ni}^2}{2}\bigg| &= \sum_{i=1}^n \ep_{ni}\Big|\tanh\frac{\ep_{ni}}{2}-\frac{\ep_{ni}}{2}\Big|\\
	&\le C\cdot \sum_{i=1}^n \ep_{ni}\cdot \ep_{ni}^2\\
	&\le C\cdot \max_{1\leqslant i \leqslant n} \ep_{ni}\cdot \sum_{i=1}^n\ep_{ni}^2\to0.
\end{align*}
So $\sum_{i=1}^n \ep_{ni}\tanh\frac{\ep_{ni}}{2}$ and $\sum_{i=1}^n \tfrac{\ep_{ni}^2}{2}$ has the same limit, which by our assumption is $\mu^2/2$.

For second moment, $\sum_{i=1}^n \kappa_2(f_{ni})=\sum_{i=1}^n \ep_{ni}^2$ has limit $\mu^2$. That is, $s$ in \Cref{thm:CLT} is equal to $\mu$.

For third moment,
\[\sum_{i=1}^n \kappa_3(f_{ni})=\sum_{i=1}^n \ep_{ni}^3 \leqslant \big(\max_{1\le i\le n} \ep_{ni} \big)\cdot\sum_{i=1}^n \ep_{ni}^2\to 0.\]

All four conditions of \Cref{thm:CLT} check, so we can conclude the limit of $f^{(n)}$ is the GDP trade-off function with parameter $2K/s = s = \mu$.

The last step is to combine the two limits $f^{(n)}\to G_\mu$ and $f_{0,\delta^{(n)}}\to f_{0,1-\e^{-\delta}}$. By \Cref{prop:ruibbit},

\begin{equation*}
f^{(n)}\otimes f_{0,\delta^{(n)}}(\alpha) =
	\left\{
	\begin{array}{ll}
	(1-\delta^{(n)})\cdot f^{(n)}(\frac{\alpha}{1-\delta^{(n)}}), 		& 0\leqslant \alpha \leqslant 1-\delta^{(n)}, \\
	0, & 1-\delta^{(n)}\leqslant \alpha\leqslant 1.
	\end{array}
	\right.
\end{equation*}
\Cref{lem:uniform} tells us $f^{(n)} $  uniformly converges to $ G_\mu$, so we have the limit
\[f^{(n)}(\tfrac{\alpha}{1-\delta^{(n)}})\to G_{\mu}(\tfrac{\alpha}{1-(1-\e^{-\delta})})\]
This implies the pointwise limit
\[
	f_{\ep_{n1},\delta_{n1}}\otimes \cdots \otimes f_{\ep_{nn},\delta_{nn}} = f^{(n)}\otimes f_{0,\delta^{(n)}}\to G_\mu\otimes  f_{0,1-\e^{-\delta}}.
\]
Again, uniform convergence comes for free via \Cref{lem:uniform}.
\end{proof}

The next two corollaries are Berry-Esseen style central limit theorems for the composition of pure $\ep$-DP. Given the existence of \Cref{thm:fast}, these results are relatively loose, but might be good enough if we have large $n$. Nonzero $\delta$ is allowed following a similar argument as in \Cref{thm:DPCLT}.
\begin{corollary}\label{cor:e0CLT}
	Set $t_i = \tanh \frac{\ep_i}{2}$ and
	\begin{align*}
		\mu=&\,\,\frac{2\sum_{i=1}^n \ep_it_i}{\big(\sum_{i=1}^n \ep_i^2 (1-t_i^2)\big)^{1/2}},\\
		\gamma=&\,\,0.56\cdot \frac{\sum_{i=1}^n \ep_i^3(1-t_i^4)}{\big(\sum_{i=1}^n \ep_i^2 (1-t_i^2)\big)^{3/2}}.
	\end{align*}
	Then for any $\alpha\in[0,1]$,
	\begin{equation*}
		G_\mu(\alpha+\gamma)-\gamma\leqslant f_{\ep_1,0}\otimes f_{\ep_2,0} \otimes \cdots \otimes f_{\ep_n,0}(\alpha)\leqslant G_\mu(\alpha-\gamma)+\gamma.
	\end{equation*}
\end{corollary}

In order to highlight the $1/\sqrt{n}$ convergence rate, we also derive an easy version in the homogeneous case.
\begin{corollary}\label{cor:e0Berry}
	Let $\mu = 2\sqrt{n}\sinh\frac{\ep}{2},
		\gamma = \frac{0.56}{\sqrt{n}}\cdot\frac{\cosh \ep}{\cosh\frac{\ep}{2}}$.
	Then
	\[
		G_\mu(\alpha+\gamma)-\gamma\leqslant f_{\ep,0}^{\otimes n}(\alpha)\leqslant G_\mu(\alpha-\gamma)+\gamma.
	\]
\end{corollary}
To see $\gamma = O(1/\sqrt{n})$, note that for the limit to be meaningful, $\ep$ has to be $o(1)$, which implies $\frac{\cosh \ep}{\cosh\frac{\ep}{2}}\approx 1$.

Both of them rely on evaluating the moment functionals on $f_{\ep,0}$. We summarize the results in the following lemma:
\begin{lemma} \label{lem:epfunctionals}
	Let $t = \tanh \frac{\ep}{2}$. Then
	\begin{align*}
		\kl(f_{\ep,0}) = \ep t, \quad \kappa_2(f_{\ep,0}) = \ep^2, \quad \kappa_3(f_{\ep,0}) = \ep^3, \quad \bar{\kappa}_3(f_{\ep,0}) = \ep^3(1-t^4).
	\end{align*}
\end{lemma}

\begin{proof}[Proof of \Cref{lem:epfunctionals}]
	For convenience, let $p = \frac{\e^{\ep}}{\e^{\ep}+1}$ and $q = 1-p = \frac{1}{\e^{\ep}+1}$. We have
	\begin{align*}
		p &= \frac{\e^{\ep}}{\e^{\ep}+1} = \frac{\e^\frac{\ep}{2}}{\e^\frac{\ep}{2}+\e^{-\frac{\ep}{2}}} = \frac{\cosh \frac{\ep}{2}+\sinh \frac{\ep}{2}}{2\cosh \frac{\ep}{2}} = \frac{1}{2}(1+t)\\
		q &= \frac{1}{\e^{\ep}+1} = \frac{\e^{-\frac{\ep}{2}}}{\e^\frac{\ep}{2}+\e^{-\frac{\ep}{2}}} = \frac{\cosh \frac{\ep}{2}-\sinh \frac{\ep}{2}}{2\cosh \frac{\ep}{2}} = \frac{1}{2}(1-t).
	\end{align*}
	The log likelihood ratio is $-\ep$ with probability $p$ and $\ep$ with probability $q$, so
	\begin{align*}
		\kl(f_{\ep,0}) &= -[(-\ep)\cdot p + \ep\cdot q] = \ep(p-q) = \ep t,\\
		\kappa_2(f_{\ep,0}) &= \ep^2(p+q) = \ep^2,\\
		\kappa_3(f_{\ep,0}) &= \ep^3(p+q) = \ep^3
	\end{align*}
	and
	\begin{align*}
		\bar{\kappa}_3(f_{\ep,0}) &= p|-\ep+\ep(p-q)|^3+ q|\ep+\ep(p-q)|^3 \\
		&= 8\ep^3\cdot pq(p^2+q^2)\\
		&=2\ep^3\cdot (1-t^2)\cdot \frac{1}{2}(1+t^2)\\
		&=\ep^3(1-t^4).
	\end{align*}
\end{proof}

\begin{proof}[Proof of \Cref{cor:e0CLT} ]
	Follows directly from \Cref{thm:Berry} and \Cref{lem:epfunctionals}.
\end{proof}

\begin{proof}[Proof of \Cref{cor:e0Berry}]
	All we need to do is to simplify the expression of $\mu$ and $\gamma$ assuming all $\ep_i = \ep$ and hence $t_i = t = \tanh \frac{\ep}{2}$. 
	\begin{align*}
		\mu&=\frac{2n \ep t}{\sqrt{n\ep^2 (1-t^2)}}\\
		&=2\sqrt{n}\cdot\frac{t}{1-t^2} = 2\sqrt{n}\sinh\frac{\ep}{2}\\
		\gamma&=0.56\cdot \frac{n \ep^3(1-t^4)}{\big(n \ep^2 (1-t^2)\big)^{3/2}}\\
		&=\frac{0.56}{\sqrt{n}}\cdot\frac{1+t^2}{\sqrt{1-t^2}}\\
		&=\frac{0.56}{\sqrt{n}}\cdot\frac{\cosh \ep}{\cosh\frac{\ep}{2}}.
	\end{align*}
	The proof is complete.
\end{proof}

\bigskip


\section{Proof of \Cref{thm:fast}}
\label{app:fast}
This section is devoted to the proof of \Cref{thm:fast}. 
Since we always assume $\delta=0$, it is dropped from the subscript and we use $f_\ep$ to denote $f_{\ep,0}$. As in the proof of \Cref{thm:Berry}, the first step is to express $f_\ep^{\otimes n}$ in the form
$$1-f_\ep^{\otimes n}(\alpha) = F_n\big[x_n-F_n^{-1}(1-\alpha)\big]$$
with $F_n\to\Phi$ and $x_n\to 1$. Then we show both convergences have rate $1/n$. 
\begin{proof}[Proof of \Cref{thm:fast}]
	Let's find $F_n$ first. Fix $\ep$ and let $p = \tfrac{1}{1+e^\ep}, q = 1-p = p\cdot e^\ep$. Recall that $f_\ep^{\otimes n} = T\big(B(n,p),B(n,q)\big)$ and we know that it is the linear interpolation of points given by binomial tails.
	The main goal here is to avoid the linear interpolation.
	
	For the simple hypothesis testing problem $B(n,p)$ vs $B(n,q)$, we know via Neyman-Pearson that every optimal rejection rule $\phi$ must have the following form:
	$$
		\phi(x)=\left\{
		\begin{array}{ll}
		1, 		& \text{if } x>k, \\
		0, 		& \text{if } x<k, \\
		1-c, 		& \text{if } x=k.
		\end{array}
		\right.
	$$
	It rejects (i.e. decides that the sample comes from $B(n,q)$) if it sees something greater than $k$, accepts if it sees something smaller than $k$, and reject with probability $1-c$ if it sees $k$. Such tests are parameterized by $(k,c)$ where $k\in\{0,1,\ldots, n\}$ and $c\in[0,1)$.

	The corresponding type I and type II errors are denoted by $\alpha_{(k,c)}$ and $\beta_{(k,c)}$. 
	Let $X\sim B(n,p)$ and $Y\sim U[0,1]$ be independent random variables. We have
	\begin{align*}
		\alpha_{(k,c)} &= \E_{x\sim B(n,p)}[\phi(x)] = \E[\phi(X)]\\
		&= \P[X>k] + (1-c)\P[X=k] \\
		&= \P[X>k] + \P[Y>c]\cdot \P[X=k] \\
		&= \P[X+Y>k+c]\\
		\beta_{(k,c)}&= \E_{x\sim B(n,q)}[1-\phi(x)]= \E[1-\phi(n-X)]\\
		&=\P[n-X<k]+c\cdot \P[n-X = k]\\
		&=\P[X>n-k]+\P[Y>1-c]\cdot \P[X = n- k]\\
		&= \P[X+Y>n+1-k-c]
	\end{align*}
	$X+Y$ supports on $[0,n+1]$ and has a piecewise constant density. As a consequence, the cdf $F_{X+Y}$ is a bijection between $[0,n+1]$ and $[0,1]$. So for a fixed type I error $\alpha\in[0,1]$, the optimal testing rule $(k,c)$ is uniquely determined by the formula
	$$k+c = F_{X+Y}^{-1} (1-\alpha).$$
	And we have for the trade-off function:
	$$1-f_\ep^{\otimes n}(\alpha) = F_{X+Y}\big( n+1 - F_{X+Y}^{-1} (1-\alpha)\big).$$
	Now we proceed to write $F_{X+Y}$ in a form that reveals its central limit behavior. First notice $\E[X+Y] = np+\tfrac{1}{2}, \Var[X+Y] = \Var[X] + \Var[Y] = npq+\tfrac{1}{12}$. For simplicity denote this variance by $\sigma^2$. Let $F_n$ be the normalized cdf of $X+Y$, i.e.
	$$F_n(x) = P\Big[\tfrac{X+Y - \E[X+Y]}{\sqrt{\Var[X+Y]}}\leqslant x\Big] = F_{X+Y}\Big[np+\tfrac{1}{2} + x\sigma\Big].$$
	Simple algebra yields
	\begin{equation}\label{eq:fast1}
	1-f_\ep^{\otimes n}(\alpha) = F_n\Big[\tfrac{n(q-p)}{\sigma}-F_n^{-1}(1-\alpha)\Big].
	\end{equation}
	It's easy to show that $\tfrac{n(q-p)}{\sigma}\to 1$ and $F_n\to\Phi$ pointwise. However, we need to show that the convergence rates are both $1/n$, which is technically involved, especially for the convergence of $F_n$. In view of this, we pack the conclusions into the following lemmas, and provide the proofs later:
	\begin{lemma} \label{lem:pqlimit}
	With $\ep = 1/\sqrt{n}$ and $p,q,\sigma$ defined as above,
	\[\frac{n(q-p)}{\sigma} = 1-\frac{1}{8n}+o(n^{-1}).\]
	\end{lemma}
	As a consequence, there exists $C>0$ such that
	\begin{equation}\label{eq:fast2}
	\big|\tfrac{n(q-p)}{\sigma}-1\big|\leqslant\tfrac{C}{n}.
	\end{equation}
	\begin{lemma}\label{prop:ch.f.}
		There is a positive number $C$ such that $|F_n(x)-\Phi(x)|\leqslant \tfrac{C}{n}$ holds for $n\geqslant 2$.
	\end{lemma}
	Since $\Phi(x)\geqslant F_n(x)-\frac{C}{n}$, setting $x = F_n^{-1}(1-\alpha)$ yields
	\[\Phi\big(F_n^{-1}(1-\alpha)\big)\geqslant F_n\big(F_n^{-1}(1-\alpha)\big)-\tfrac{C}{n} = 1-\alpha-\tfrac{C}{n}.\]
	Hence
	\begin{equation}\label{eq:fast3}
		F_n^{-1}(1-\alpha)\geqslant \Phi^{-1}\big(1-\alpha-\tfrac{C}{n}\big).
	\end{equation}
	With (\ref{eq:fast1}--\ref{eq:fast3}) and \Cref{prop:ch.f.} we have
	\begin{align*}
		1-f_\ep^{\otimes n}(\alpha) &= F_n\Big[\tfrac{n(q-p)}{\sigma}-F_n^{-1}(1-\alpha)\Big]\\
		&\leqslant \Phi\Big[\tfrac{n(q-p)}{\sigma}-F_n^{-1}(1-\alpha)\Big]+\tfrac{C}{n}\\
		&\leqslant \Phi\Big[1+\tfrac{C}{n}-F_n^{-1}(1-\alpha)\Big]+\tfrac{C}{n}\\
		&\leqslant \Phi\Big[1+\tfrac{C}{n}-\Phi^{-1}(1-\alpha-\tfrac{C}{n})\Big]+\tfrac{C}{n}.
	\end{align*}
	The function $\Phi$ is $\frac{1}{\sqrt{2\pi}}$-Lipschitz, so
	\[1-f_\ep^{\otimes n}(\alpha)\leqslant \Phi\Big[1-\Phi^{-1}(1-\alpha-\tfrac{C}{n})\Big]+\tfrac{1}{\sqrt{2\pi}}\cdot\tfrac{C}{n}+\tfrac{C}{n}.\]
	By blowing up the current $C$ and using the symmetry of standard normal, we have
	\begin{align*}
		f_\ep^{\otimes n}(\alpha) &\geqslant 1-\Phi\Big[1-\Phi^{-1}(1-\alpha-\tfrac{C}{n})\Big]-\tfrac{C}{n}\\
		&= \Phi\Big[\Phi^{-1}(1-\alpha-\tfrac{C}{n}) - 1 \Big]- \tfrac{C}{n}\\
		& = G_1(\alpha+\tfrac{C}{n})-\tfrac{C}{n}.
	\end{align*}
	Similarly, we can show the upper bound
	\[f_\ep^{\otimes n}(\alpha)\leqslant G_1(\alpha-\tfrac{C}{n})+\tfrac{C}{n}.\]
	The proof is now complete.
\end{proof}
Next we show \Cref{lem:pqlimit} and \Cref{prop:ch.f.}.
\begin{proof}[Proof of \Cref{lem:pqlimit}]
	The proof is basically careful Taylor expansion. We will frequently use the assumption that $\ep = 1/\sqrt{n}$. First we factor the objective as
	\begin{align*}
		\frac{n(q-p)}{\sigma} &= 2\sqrt{n}(q-p)\cdot\frac{\sqrt{n}}{2\sigma} = \frac{2(q-p)}{\ep}\cdot\frac{\sqrt{n}}{2\sigma}
	\end{align*}
	and consider Taylor expansions of the two factors separately.
	For the first factor, recall that
	\[q-p = \frac{e^\ep-1}{e^\ep+1} = \frac{e^{\tfrac{\ep}{2}}-e^{-\tfrac{\ep}{2}}}{e^{\tfrac{\ep}{2}}+e^{-\tfrac{\ep}{2}}} = \tanh \tfrac{\ep}{2}.\]
	Using the Taylor expansion $\tanh x  = x-x^3/3+o(x^4)$, we have
	\begin{equation}\label{eq:pq1}
		\tfrac{2(q-p)}{\ep} = \tanh \tfrac{\ep}{2} \,\big/\, \tfrac{\ep}{2} = 1-\tfrac{1}{3}(\tfrac{\ep}{2})^2+o(\ep^3) = 1-\tfrac{1}{12n} + o(n^{-3/2}).
	\end{equation}
	For the second one, since $p+q=1$, we have $4pq=(p+q)^2-(p-q)^2 = 1-(q-p)^2$.
	A shorter expansion shows $q-p = \tanh\tfrac{\ep}{2} = \tfrac{\ep}{2} +o(\ep^2)$, and hence
	\begin{align*}
		4pq= 1-\big(\tfrac{\ep}{2}+o(\ep^2)\big)^2 = 1-\tfrac{\ep^2}{4}+o(\ep^3)= 1-\tfrac{1}{4n}+o(n^{-3/2}).
	\end{align*}
	Recall that $\sigma$ is defined to be $\sqrt{npq+\frac{1}{12}}$. Using the above expansion of $4pq$, we have
	\begin{align*}
		\frac{\sqrt{n}}{2\sigma} &= \sqrt{\frac{n}{4\sigma^2}} = \sqrt{\frac{n}{4npq+\tfrac{1}{3}}} = \big(4pq+\tfrac{1}{3n}\big)^{-1/2} = \big(1+\tfrac{1}{12n}+o(n^{-3/2})\big)^{-1/2}
	\end{align*}
	Since $(1+x)^{-1/2} = 1-\tfrac{1}{2}x+o(x)$, we have
	\begin{equation}\label{eq:pq2}
		\frac{\sqrt{n}}{2\sigma} = 1-\tfrac{1}{2}\big(\tfrac{1}{12n}+o(n^{-3/2})\big)+o(n^{-1}) = 1-\tfrac{1}{24n}+o(n^{-1}).
	\end{equation}
	Combining the expansions \eqref{eq:pq1} and \eqref{eq:pq2},
	\begin{align*}
		\frac{n(q-p)}{\sigma} &= \frac{2(q-p)}{\ep}\cdot\frac{\sqrt{n}}{2\sigma}\\
		&=\Big(1-\frac{1}{12n}+o(n^{-3/2})\Big)\cdot\Big(1-\frac{1}{24n}+o(n^{-1})\Big)\\
		&=1-\frac{1}{8n}+o(n^{-1}).
	\end{align*}
	The proof is complete.
\end{proof}
Then we move on to the more challenging \Cref{prop:ch.f.}.
\begin{proof}[Proof of \Cref{prop:ch.f.}]
	The proof is inspired by Problem 6 on page 305 of \cite{uspensky1937introduction}. Though involved, the idea is not hard: reduce the bound on cdfs to a bound on characteristic functions (ch.f. for short) by an appropriate Fourier inversion, then control the ch.f. by careful Taylor expansion.

	Recall that $\ep,p,q,\sigma$ depend on $n$ via
	\[\ep = \frac{1}{\sqrt{n}}, \quad p = \frac{1}{1+e^\ep}, \quad q=\frac{e^\ep}{1+e^\ep},\quad \sigma = \sqrt{npq+\tfrac{1}{12}}.\]
	Random variables $X\sim B(n,p),Y\sim U[0,1]$. $F_n$ is the normalized cdf of $X+Y$. More precisely, since
	$$\E[X+Y] = np+\frac{1}{2}, \quad \Var[X+Y] = \Var\, X + \Var\, Y = npq+\tfrac{1}{12} = \sigma^2,$$
	$F_n$ is the cdf of $\sigma^{-1}(X+Y-\frac{1}{2}-np)$. Our goal is to show that $\sup_{x\in\R}|F_n(x)-\Phi(x)|=O(\frac{1}{n})$.

	First let's compute the characteristic function (ch.f. for short) $\varphi_n$ of the distribution $F_n$.
	\begin{align*}
		\varphi_n(t)
		&= \E[e^{it\sigma^{-1}(X+Y-\frac{1}{2}-np)}]\\
		&= e^{-i{np}t/\sigma}\cdot\E[e^{it/\sigma(X+Y-\frac{1}{2})}]\\
		&= e^{-i{np}t/\sigma}\cdot \varphi_{X}(t/\sigma)\cdot \varphi_{Y-\frac{1}{2}}(t/\sigma).
	\end{align*}
	Easy calculation shows that the ch.f. of $X$ is $(pe^{it}+q)^n$ and that of $Y-\tfrac{1}{2}$ is $\tfrac{\sin t/2}{t/2}$. So
	\begin{align*}
		\varphi_n(t)
		&= e^{-i{np}t/\sigma}\cdot(pe^{it/\sigma}+q)^n \cdot \tfrac{\sin t/2\sigma}{t/2\sigma}\\
		&= (pe^{iqt/\sigma}+qe^{-ipt/\sigma})^n\cdot \tfrac{\sin t/2\sigma}{t/2\sigma}.
	\end{align*}
	The base $pe^{iqt/\sigma}+qe^{-ipt/\sigma}$ is a convex combination of two complex numbers on the unit circle, so we have $|\varphi_n(t)|\leqslant \tfrac{|\sin t/2\sigma|}{|t/2\sigma|} \leqslant \min\{\tfrac{2\sigma}{|t|}, 1\}$.

	Now let's connect back to cdf. We need some form of Fourier inversion formula. 
	Let $\varphi(t) = e^{-t^2/2}$ be the ch.f. of the standard normal.
	\begin{lemma} \label{lem:chf}
		We have the following inversion formula
		$$F_n(x) - \Phi(x) = -\frac{1}{2\pi i}\int_{-\infty}^{+\infty}e^{-itx}\cdot\frac{\varphi_n(t)-\varphi(t)}{t} \diff t.$$
	\end{lemma}
	The integrand is integrable over $\R$ because: (1) At infinity $|\varphi_n(t)|=O(\frac{1}{t})$, so the integrand has modulus $O(\frac{1}{t^2})$; (2) When $t\to0$,
	$$\tfrac{\varphi_n(t)-\varphi(t)}{t} = \tfrac{\varphi_n(t)-1-\varphi(t)+1}{t}=\tfrac{\varphi_n(t)-\varphi_n(0)}{t}-\tfrac{\varphi(t)-\varphi(0)}{t}\to\varphi_n'(0)-\varphi'(0) = \E_{Z\sim F_n} [Z]$$
	is a finite number. So the integrand is continuous at 0.

	\Cref{lem:chf} makes it possible to control $F_n(x) - \Phi(x)$ by controlling $\varphi_n(t)-\varphi(t)$.
	\begin{align*}
		2\pi|F_n(x) - \Phi(x)| \leqslant&\phantom{+}\int^{+\infty}_{-\infty}\frac{|\varphi_n(t)-\varphi(t)|}{|t|}\diff t\\
		\leqslant&\phantom{+}
		\int_{|t|\leqslant r\sigma}\frac{|\varphi_n(t)-\varphi(t)|}{|t|} \diff t &&(I_1)\\
		&+\int_{|t|>r\sigma}\frac{|\varphi_n(t)|}{|t|} \diff t &&(I_2)\\
		&+\int_{|t|>r\sigma}\frac{|\varphi(t)|}{|t|} \diff t &&(I_3)
	\end{align*}
	It suffices to find some constant $r$ such that all three integrals are $O(\frac{1}{n})$. This is done via the following three lemmas.
	\begin{lemma} \label{lem:I1}
		There exist universal constants $r>0, C>0$ such that when $|t|\leqslant r\sigma$,
	\[|\varphi_n(t)-\varphi(t)|\leqslant Ce^{-\tfrac{t^2}{8}}\cdot\big(\tfrac{t^2}{n}+\tfrac{|t|^3}{n}+\tfrac{t^4}{n}\big).\]
	\end{lemma}
	Consequently,
	\begin{align*}
		I_1 &= \int_{|t|\leqslant r\sigma}\frac{|\varphi_n(t)-\varphi(t)|}{|t|} \diff t\\
		&\leqslant \int_\R Ce^{-\tfrac{t^2}{8}}\cdot\big(\tfrac{|t|}{n}+\tfrac{t^2}{n}+\tfrac{|t|^3}{n}\big)\diff t=O(\tfrac{1}{n}).
	\end{align*}
	\begin{lemma} \label{lem:I2}
		For $r<\pi$, we have
		$I_2\leqslant (2+\frac{48}{r^2})\cdot\frac{1}{n}$.
	\end{lemma}
	\begin{lemma} \label{lem:I3}
		For $n\geqslant 2$, $I_3\leqslant\frac{10}{r^2}\cdot \frac{1}{n}\cdot e^{-0.1r^2n}$ holds for any positive $r$.
	\end{lemma}
	So we can select a small enough $r$ such that all three estimates hold, which implies $I_1=O(\frac{1}{n}), I_2=O(\frac{1}{n})$ and $I_3\ll\frac{1}{n}$. In summary,
	\[|F_n(x) - \Phi(x)|\leqslant \frac{1}{2\pi}(I_1+I_2+I_3) = O(\frac{1}{n}).\]
	Assuming correctness of \Cref{lem:chf,lem:I1,lem:I2,lem:I3}, the proof of \Cref{thm:fast} is complete.
\end{proof}
The rest is to prove \Cref{lem:chf,lem:I1,lem:I2,lem:I3}. We deal with the three integrals first, and then come back to inversion formula.
\begin{proof}[Proof of \Cref{lem:I1}]
	Let $w=pe^{itq/\sigma}+qe^{-itp/\sigma}$. Then $\varphi_n(t) = w^n\cdot \frac{\sin t/2\sigma}{t/2\sigma}$. We have
	\begin{align}\label{eq:twoterms}
		|\varphi_n(t)-\varphi(t)| &= |w^n\cdot \frac{\sin t/2\sigma}{t/2\sigma} - e^{-\frac{1}{2}t^2}|\notag\\
		&\leqslant |w^n- e^{-\frac{1}{2}t^2}| + |w|^n \cdot \Big|1-\frac{\sin t/2\sigma}{t/2\sigma}\Big|
	\end{align}
	All we need is a positive $r$ such that when $|t|\leqslant r\sigma$, both of the above terms are small. We are going to shrink $r$ as we need from time to time.

	First, on the disk $|z|\leqslant r$ we have Taylor expansion 
	$e^z=1+z+\tfrac{1}{2}z^2+\tfrac{1}{6}z^3+O(|z|^4)$.
	So for $|t|\leqslant r\sigma$ we have
	\begin{align}\label{eq:taylor}
		pe^{itq/\sigma} &= p\big(1+\tfrac{itq}{\sigma}+\tfrac{1}{2}\big(\tfrac{itq}{\sigma}\big)^2+\tfrac{1}{6}\big(\tfrac{itq}{\sigma}\big)^3+O(\tfrac{t^4}{\sigma^4})\big)\notag\\
		qe^{-itp/\sigma} &= q\big(1-\tfrac{itp}{\sigma}+\tfrac{1}{2}\big(\tfrac{itp}{\sigma}\big)^2-\tfrac{1}{6}\big(\tfrac{itp}{\sigma}\big)^3+O(\tfrac{t^4}{\sigma^4})\big)\notag\\
		w &= 1-\tfrac{t^2}{2\sigma^2}\cdot (pq^2+qp^2) + \tfrac{t^3}{6\sigma^3}\big(-ipq^3-qp^3(-i)\big)+O(\tfrac{t^4}{\sigma^4})\notag\\
		&=1-\tfrac{pq}{\sigma^2}\cdot\tfrac{t^2}{2} + \tfrac{t^3}{6\sigma^3}\cdot ipq(p-q)+O(\tfrac{t^4}{\sigma^4}).
	\end{align}
	Obviously this implies $w = 1-\tfrac{pq}{\sigma^2}\cdot\tfrac{t^2}{2} + o(\tfrac{t^2}{\sigma^2})$ (we will return to the more delicate \eqref{eq:taylor} soon).
	Since $npq \geqslant pq \geqslant \frac{1}{4} - \frac{1}{16n} \geqslant \frac{3}{16}$, we have
	\[\tfrac{pq}{\sigma^2}  = \tfrac{pq}{npq+\frac{1}{12}}=\tfrac{npq}{npq+\frac{1}{12}}\cdot \tfrac{1}{n}\geqslant \tfrac{1}{n} \cdot \tfrac{3}{16}/(\tfrac{3}{16}+\tfrac{1}{12}) = \tfrac{9}{13n}>\tfrac{2}{3n}.\]
	That is, the quadratic term is more than $\tfrac{t^2}{3n}$. We can tune $r$ so that the little $o$ remainder is even smaller, i.e. $|w-1+\tfrac{pq}{\sigma^2}\cdot\tfrac{t^2}{2}|<\tfrac{t^2}{12n}$. This implies
	$$|w|<1-\tfrac{t^2}{3n}+\tfrac{t^2}{12n} = 1-\tfrac{t^2}{4n}\leqslant e^{-\tfrac{t^2}{4n}}.$$
	One consequence is we can bound the second term in \eqref{eq:twoterms}. By Taylor expansion again, $\frac{\sin x}{x} = 1+O(x^2)$, so
	\begin{equation}\label{eq:first}
		|w|^n \cdot \big|1-\tfrac{\sin t/2\sigma}{t/2\sigma}\big| \leqslant e^{-\tfrac{t^2}{4}} \cdot O(\tfrac{t^2}{\sigma^2}) = e^{-\tfrac{t^2}{4}} \cdot O(\tfrac{t^2}{n}). 
	\end{equation}
	The first term in \eqref{eq:twoterms} requires a more careful analysis. Let $z = e^{-\tfrac{t^2}{2n}}$ and $\gamma = e^{-\tfrac{t^2}{4n}}$. Our goal is $|w^n-z^n|$. We have proved $|w|<\gamma$, while $|z|=\gamma^2<\gamma$ is obviously true. We have
	\[|w^n-z^n|\leqslant |w^n-w^{n-1}z|+\cdots +|wz^{n-1}-z^n|\leqslant n|w-z|\cdot \gamma^{n-1}.\]
	Without loss of generality assume $n\geqslant 2$, then $\gamma^{n-1} = e^{-\frac{t^2}{4}\cdot \frac{n-1}{n}}\leqslant e^{-\tfrac{t^2}{8}}$. That is,
	\begin{equation}\label{eq:a}
		|w^n- e^{-\frac{1}{2}t^2}| \leqslant n|w-e^{-\frac{t^2}{2n}}|\cdot e^{-\frac{1}{8}t^2}.
	\end{equation}
	For $n|w-e^{-\frac{t^2}{2n}}|$ we need \eqref{eq:taylor} again. First decompose it as
	\begin{align}\label{eq:b}
		n|w-e^{-\frac{t^2}{2n}}| \leqslant n\big|w-1+\tfrac{t^2}{2n}\big| + n\big|e^{-\frac{t^2}{2n}} - 1+\tfrac{t^2}{2n}\big|.
	\end{align}
	Since $|p-q|=\tfrac{e^\ep-1}{e^\ep+1}\leqslant e^\ep-1 = O(\ep) = O(\tfrac{1}{\sqrt{n}})$ and $\sigma^{-1} = O(\frac{1}{\sqrt{n}})$, we have
	\[w = 1-\tfrac{pq}{\sigma^2}\cdot\tfrac{t^2}{2} + \tfrac{t^3}{6\sigma^3}\cdot ipq(p-q)+O(\tfrac{t^4}{\sigma^4}) = 1-\tfrac{pq}{\sigma^2}\cdot\tfrac{t^2}{2} +O(\tfrac{t^3}{n^2})+O(\tfrac{t^4}{n^2}).\]
	Note that neither of two ``big $O$'' dominate each other, because $t$ can be as small as 0 and as large as $r\sigma = O(\sqrt{n})$.
	Using the more delicate expansion of $w$ and that $\sigma^2 = npq+\frac{1}{12}$, we have
	\begin{align}\label{eq:c}
		n\big|w-1+\tfrac{t^2}{2n}\big| &= n\big|\tfrac{t^2}{2n}-\tfrac{pq}{\sigma^2}\cdot\tfrac{t^2}{2} +O(\tfrac{t^3}{n^2})+O(\tfrac{t^4}{n^2})\big|\notag\\
		&= \big|\tfrac{t^2}{2}-\tfrac{npq}{\sigma^2}\cdot\tfrac{t^2}{2} +O(\tfrac{t^3}{n})+O(\tfrac{t^4}{n})\big|\notag\\
		&= \big|\tfrac{t^2}{2}\cdot \tfrac{1}{12\sigma^2} +O(\tfrac{t^3}{n^2})+O(\tfrac{t^4}{n^2})\big|\notag\\
		&=O(\tfrac{t^2}{n^2})+O(\tfrac{t^3}{n^2})+O(\tfrac{t^4}{n^2}).
	\end{align}
	By Taylor expansion again, we can tune $r$ so that when $|t|\leqslant r\sigma$, we have
	\begin{equation}\label{eq:d}
		n\big|e^{-\frac{t^2}{2n}} - 1+\tfrac{t^2}{2n}\big| = n\cdot O(\tfrac{t^4}{n^2}) = O(\tfrac{t^4}{n}).
	\end{equation}
	Now plug \eqref{eq:c} and \eqref{eq:d} back into \eqref{eq:b}, and then into \eqref{eq:a} to get
	\[|w^n- e^{-\frac{1}{2}t^2}| \leqslant e^{-\frac{1}{8}t^2}\cdot\big(O(\tfrac{t^2}{n^2})+O(\tfrac{t^3}{n^2})+O(\tfrac{t^4}{n^2})\big).\]
	This ends the analysis of the first term of \eqref{eq:twoterms}. Combining with the estimate of the first term, we have
	\[|\varphi_n(t)-\varphi(t)| \leqslant e^{-\frac{1}{8}t^2}\cdot\big(O(\tfrac{t^2}{n^2})+O(\tfrac{t^3}{n^2})+O(\tfrac{t^4}{n^2})\big).\]
\end{proof}
\begin{proof}[Proof of \Cref{lem:I2}]
	For this integral we only care about the modulus of $|\varphi_n(t)|$. Let's simplify it first.
	\begin{align*}
		|\varphi_n(t)| &= |pe^{iqt/\sigma}+qe^{-ipt/\sigma}|^n \cdot \Big|\tfrac{\sin t/2\sigma}{t/2\sigma}\Big|
	\end{align*}
	Let $\theta = t/\sigma$. Using $|z|^2 = z\bar{z}$, we have
	\begin{align*}
		|pe^{iq\theta}+qe^{-ip\theta}|^2
		&= \big(pe^{iq\theta}+qe^{-ip\theta}\big) \cdot \big(pe^{-iq\theta}+qe^{ip\theta}\big)\\
		&= p^2+q^2+pq(e^{i\theta}+e^{-i\theta})\\
		&= 1-2pq+2pq \cos \theta\\
		&= 1-4pq\sin^2 \tfrac{\theta}{2}.
	\end{align*}
	So 
	\[|\varphi_n(t)| = \big(1-4pq\sin^2 \tfrac{t}{2\sigma}\big)^{n/2} \cdot \Big|\tfrac{\sin t/2\sigma}{t/2\sigma}\Big|.\]
	We see from this expression that the integrand of $I_2$ is an even function. Therefore,
	\begin{align*}
	\frac{1}{2}I_2 =& \int_{r\sigma}^{+\infty}\tfrac{|\varphi_n(t)|}{t} \diff t\\
		=& \int_{r\sigma}^{+\infty}\tfrac{1}{t} \big(1-4pq\sin^2 \tfrac{t}{2\sigma}\big)^{n/2} \cdot \Big|\tfrac{\sin t/2\sigma}{t/2\sigma}\Big|\diff t\\
		=& \int_{r/2}^{+\infty}\tfrac{1}{t^2} \big(1-4pq\sin^2 t\big)^{n/2} \cdot |\sin t|\diff t
	\end{align*}
	In the last step we do a change of variable $s = t/2\sigma$ and rename $s$ to $t$. Next, we break down the integral at $k\pi$, and upper bound the $\frac{1}{t^2}$ factor by its value at the left end of the interval, so that the rest of the integrand is periodic.
	\begin{align*}
		\frac{1}{2}I_2
		\leqslant&\phantom{+}
		\int_{r/2}^{\pi}\tfrac{1}{t^2} \big(1-4pq\sin^2 t\big)^{n/2} \cdot |\sin t|\diff t\\
	&+\sum_{k=1}^{+\infty}\int_{k\pi}^{(k+1)\pi}\tfrac{1}{t^2} \big(1-4pq\sin^2 t\big)^{n/2} \cdot |\sin t|\diff t\\
		\leqslant&\phantom{+} \Big(\tfrac{4}{r^2}+\sum_{k=1}^{+\infty}\tfrac{1}{k^2\pi^2}\Big)\underbrace{\int_{0}^{\pi} \big(1-4pq\sin^2 t\big)^{n/2} \cdot \sin t\diff t}_{J}
	\end{align*}
	The integral $J$ can be estimated as follows:
	\begin{align*}
		J
		&=\int_{0}^{\pi} \big(1-4pq\sin^2 t\big)^{n/2} \cdot \sin t\diff t\\
		&= -\int_{0}^{\pi} \big(1-4pq(1-\cos^2 t)\big)^{n/2} \diff\cos t\\
		&= \int_{-1}^{1} \big(1-4pq(1-x^2)\big)^{n/2} \diff x\\
		&= 2\int_{0}^{1} (1-4pq + 4pqx^2)^{n/2} \diff x.
	\end{align*}
	We have seen that $1-4pq = (p-q)^2 = \tanh^2 \frac{\ep}{2}$. It is easy to show that $\tanh x \leqslant x$ for $x\geqslant0$. So
	$$
		pq = \tfrac{1}{4}(1-\tanh^2 \tfrac{\ep}{2})\geqslant \tfrac{1}{4}(1-\tfrac{\ep^2}{4}) = \tfrac{1}{4} - \tfrac{1}{16n}.
	$$
	Since $0\leqslant x \leqslant 1$, we have 
	\[1-4pq + 4pqx^2 \leqslant \tfrac{1}{4n}+(1- \tfrac{1}{4n})x^2.\]
	Hence
	\begin{align*}
		J \leqslant 2\int_{0}^{1} \Big(\tfrac{1}{4n}+(1-\tfrac{1}{4n})x^2\Big)^{n/2} \diff x.
	\end{align*}
	It's easy to check that $\tfrac{1}{4n-1}$ and 1 are the two roots of the quadratic equation $\tfrac{1}{4n}+(1-\tfrac{1}{4n})x^2 = x$. So we have $\tfrac{1}{4n}+(1-\tfrac{1}{4n})x^2\leqslant x$ between the two roots, i.e. for $x\in[\tfrac{1}{4n-1},1]$. For the rest of the interval, we upper bound the integrand by 1. That is,
	\begin{align*}
		\int_{0}^{1} \Big(\tfrac{1}{4n}+(1-\tfrac{1}{4n})x^2\Big)^{n/2} \diff x
		\leqslant\int_{0}^{\tfrac{1}{4n-1}} 1\diff x + \int_{\tfrac{1}{4n-1}}^1 x^{n/2}\diff x\leqslant \tfrac{1}{4n-1} +\tfrac{1}{n/2+1} \leqslant \tfrac{3}{n}.
	\end{align*}
	So we have $J\leqslant \frac{6}{n}$. Returning to $I_2$, with the well-known identity $\sum_{k=1}^{+\infty}\tfrac{1}{k^2} = \frac{\pi^2}{6}$, we have
	\begin{align*}
		I_2&\leqslant 2\Big(\tfrac{4}{r^2}+\sum_{k=1}^{+\infty}\tfrac{1}{k^2\pi^2}\Big)\cdot J\\
		&= \big(\tfrac{8}{r^2}+\tfrac{\pi^2}{6}\cdot\tfrac{2}{\pi^2}\big)\cdot\tfrac{6}{n}\\
		&= \big(2+\tfrac{48}{r^2}\big)\cdot\tfrac{1}{n}
	\end{align*}
	The estimate of $I_2$ is complete.
\end{proof}
\begin{proof}[Proof of \Cref{lem:I3}]
	First notice the following simple facts:
	\begin{enumerate}
	 	\item When $t>r\sigma$, we have $\frac{1}{t}\leqslant t\cdot \frac{1}{r^2\sigma^2}$.
	 	\item $\sigma^2>0.2n$ for any $n$.
	\end{enumerate}
	The second follows from a bound we derive in the proof of \Cref{lem:I2}: $pq\geqslant \tfrac{1}{4} - \tfrac{1}{16n}$. In fact,
	\[\sigma^2 = npq+\tfrac{1}{12} \geqslant \tfrac{n}{4} - \tfrac{1}{16}+\tfrac{1}{12} = \tfrac{n}{4} - \tfrac{1}{48} > \tfrac{n}{5}.\]
	Using these two facts, we can bound $I_3$ as follows:
	\begin{align*}
		I_3 
		&= \int_{|t|>r\sigma}\frac{|\varphi(t)|}{|t|} \diff t\\
		& = 2\int_{r\sigma}^{+\infty}\frac{1}{t}e^{-\frac{t^2}{2}}\diff t\\
		&\leqslant \frac{2}{r^2\sigma^2}\cdot\int_{r\sigma}^{+\infty}te^{-\frac{t^2}{2}}\diff t\\
		&= \frac{2}{r^2\sigma^2}\cdot e^{-\frac{t^2}{2}}\Big|^{r\sigma}_{+\infty}\\
		&= \frac{2}{r^2\sigma^2}\cdot e^{-\frac{r^2\sigma^2}{2}}\\
		&\leqslant \frac{10}{r^2}\cdot \frac{1}{n}\cdot e^{-0.1r^2n}
	\end{align*}
	The estimate of $I_3$ is complete.
\end{proof}

We are done with the three integrals. Before we dive into the proof of the inversion formula \ref{lem:chf}, we make a few observations.

First, one cannot hope to obtain this lemma by showing 
$$F_n(x)=\tfrac{1}{2\pi}\int_{-\infty}^{+\infty}-e^{-itx}\cdot\tfrac{\varphi_n(t)}{it} \diff t$$
and a similar expression for $\Phi(x)$ separately because this alternative integrand is not even integrable. To see this, notice $\varphi_n(0)=1$, so the integrand $\approx \frac{1}{t}$ around 0.

Inversion formula \ref{lem:chf} has the same form as Lemma 3.4.19 of \cite{durrett2019probability}. However, the ch.f.s are assumed to be (absolutely) integrable there, while $\varphi_n$ is not. To see this, recall that Fourier inversion tells us that if the ch.f. is absolutely integrable, then the probability distribution has continuous density (see e.g. \cite{durrett2019probability}, Theorem 3.3.14). This is not true for $X+Y$ because its density is piecewise constant. So $\varphi_n$ cannot be in $L^1(\R)$. There seems to be no shortcut, so let's work out our own proof.
\begin{proof}[Proof of \Cref{lem:chf}]
	Applying the general inversion formula (see e.g. \cite{durrett2019probability} Theorem 3.3.11) to $F_n$, we have
	\begin{align*}
	F_n(x)-F_n(a)&=\frac{1}{2\pi}\lim_{T\to+\infty}\int^T_{-T}\frac{e^{-ita}-e^{-itx}}{it}\cdot\varphi_n(t)\diff t
	\end{align*}
	$\varphi_n$ is continuous and decays in the rate $\tfrac{1}{|t|}$, so the integrand is dominated by $O(t^{-2}\wedge 1)$ and hence the limit on $T$ is equal to the Lebesgue integral. That is,
	\begin{equation}\label{eq:FnInversion}
		F_n(x)-F_n(a)=\frac{1}{2\pi}\int^{+\infty}_{-\infty}\frac{e^{-ita}-e^{-itx}}{it}\cdot\varphi_n(t)\diff t.
	\end{equation}
	Similarly,
	\[
	\Phi(x)-\Phi(a)=\frac{1}{2\pi}\int^{+\infty}_{-\infty}\frac{e^{-ita}-e^{-itx}}{it}\cdot\varphi(t)\diff t.
	\]
	Note that in \eqref{eq:FnInversion}, we cannot let $a\to-\infty$ and use Riemann-Lebesgue lemma because $\frac{\varphi_n(t)}{t}$ is not integrable, as discussed before the proof. However, subtracting the two formula yields
	\begin{equation}\label{eq:ligoat}
		\big(F_n(x)-\Phi(x)\big)-\big(F_n(a)-\Phi(a)\big)=\frac{1}{2\pi}\int^{+\infty}_{-\infty}\frac{e^{-ita}-e^{-itx}}{it}\cdot\big(\varphi_n(t)-\varphi(t)\big)\diff t
	\end{equation}
	Consider the part involving $a$
	\[\int^{+\infty}_{-\infty}e^{-ita}\cdot\frac{\varphi_n(t)-\varphi(t)}{it}\diff t.\]
	We argued right after introducing \Cref{lem:chf} that $\frac{\varphi_n(t)-\varphi(t)}{it} \in L^1(\R)$, so by Riemann-Lebesgue lemma we have the limit
	\[\lim_{a\to-\infty}\int^{+\infty}_{-\infty}e^{-ita}\cdot\frac{\varphi_n(t)-\varphi(t)}{it}\diff t = 0.\]
	Take the limit $a\to-\infty$ on both sides of \eqref{eq:ligoat} and we have
	\begin{align*}
		F_n(x)-\Phi(x)&=\frac{1}{2\pi}\int^{+\infty}_{-\infty}\frac{-e^{-itx}}{it}\cdot\big(\varphi_n(t)-\varphi(t)\big)\diff t\\
		&=-\frac{1}{2\pi i}\int_{-\infty}^{+\infty}e^{-itx}\cdot\frac{\varphi_n(t)-\varphi(t)}{t} \diff t.
	\end{align*}
	The proof is now complete.
\end{proof}
\section{Omitted Details in \Cref{sec:subsampling}}
\label{app:property}
We begin this appendix with a small example showing our subsampling theorem is generically unimprovable.
\paragraph{Tightness} 
\label{par:example_and_tightness}

Consider the mechanism $\tm$ that randomly releases one individual's private information in the dataset. The privacy analysis is easy: without loss of generality we can assume two neighboring datasets differ in the first individual.  Effectively we are trying to distinguish uniform distributions over $\{1,2,\ldots, n\}$ and $\{1',2,\ldots, n\}$. It's not hard to see that the trade-off function of these two uniform distributions is $f_{0,1/n}$, i.e. $(\ep,\delta)$-DP with $\ep=0, \delta=1/n$. This is exact --- the adversary has tests that achieve every point on the curve.

Our \cref{thm:subsample} yields the same result, showing its tightness. To see this, let $M$ be the identity map that takes in one individual and outputs his/her entire private information. Then $\tm = M\circ\Sample_{\frac{1}{n}}$. Privacy of $M$ is described by $f\equiv 0$. By \Cref{thm:subsample}, $\tm$ is $C_{1/n}(f)$-DP. \Cref{fig:subsample} shows that $C_{1/n}(f)=f_{0,1/n}$.

Next we show the following two equations:
	\begin{align*}
		\ep'&=\log(1-p + pe^\ep),\\
		\delta'&=p\big(1+f^*(-e^{\ep})\big)
	\end{align*}
	can be re-parameterized into 
	\begin{equation}
		\delta'=1+f_p^*(-e^{\ep'})\tag{\ref{eq:fp}}
	\end{equation}
	where $f_p = pf+(1-p)\Id$.
\begin{proof}[Proof of \Cref{eq:fp}]
	
	Since $\ep\mapsto\log(1-p + pe^\ep)$ maps $[0,+\infty)$ to $[0,+\infty)$ monotonically, we can solve $\ep$ from $\ep'$ and plug into $\delta'$. We have
	\begin{align*}
		\frac{1}{p}(1-e^{\ep'}) = 1-e^{\ep} \quad\text{ and }\quad
		\delta'=p\big(1+f^*(-e^{\ep})\big)
		&=p\big(1+f^*(\tfrac{1}{p}(1-e^{\ep'})-1)\big).
	\end{align*}
	Let $y = -e^{\ep'}$ and it suffices to show for any $y\leqslant-1$,
	\begin{equation}\label{eqn:ning}
		1+f_p^*(y) = p\big(1+f^*(\tfrac{1}{p}(1+y)-1)\big).
	\end{equation}
	To see this, expand $f_p^*$ as follows
	\begin{align*}
		f_p^*(y) &= \sup_{x} yx-f_p(x)\\
		&=\sup_{x} yx-pf(x)-(1-p)(1-x)\\
		&=p-1  + \sup_{x} (y+1-p)x -pf(x)\\
		&=p-1  + p\cdot\sup_{x} (\tfrac{1}{p}(1+y)-1)x -f(x)\\
		&=p-1 +pf^*(\tfrac{1}{p}(1+y)-1)
	\end{align*}
	\eqref{eqn:ning} follows directly.
\end{proof}

Next we provide the general tool mentioned in \Cref{sub:proof_of_subsample_theorems} that convert collections of $(\ep,\delta)$-DP guarantee in the form of \eqref{eq:fp} to some $f$-DP.

The symmetrization operator $\Symm:\T\to\T^S$ maps a general trade-off function to a symmetric trade-off function.  It's defined as follows:
\begin{definition} \label{def:symm}
	For $f\in\T$, let $\bar{x} = \inf\{x\in[0,1]:-1\in\partial f(x)\}$. The symmetrization operator $\Symm:\T\to\T^S$ is defined as
	\[\Symm(f):= \left\{
	\begin{array}{ll}
	\min\{f,f^{-1}\}^{**}, &\text{ if }\,\, \bar{x}\leqslant f(\bar{x}),\\
	\max\{f,f^{-1}\}, &\text{ if }\,\, \bar{x}>f(\bar{x}).
	\end{array}
	\right.\]
\end{definition}

\begin{proposition}\label{prop:asymm_env}
	Let $f\in\T$, not necessarily symmetric. Suppose a mechanism is $(\ep,1+f^*(-e^{\ep}))$-DP for all $\ep\geqslant 0$, then it is $\Symm(f)$-DP.
\end{proposition}


Recall from basic convex analysis that double convex conjugate $f^{**}$ is the greatest convex lower bound of $f$. If $f$ itself is convex then $f^{**}=f$. For $f$ symmetric , $f=f^{-1}$. By convexity of $f$, we have $\Symm(f)=f$ in both cases. So \Cref{prop:ftoDP} is a special case of \Cref{prop:asymm_env}. The first half of \Cref{prop:asymm_env} is \Cref{prop:asymm_for_proof}, the part we used in the proof of our subsampling theorem.

\begin{figure}[h]
  \centering
  \includegraphics[width=0.7\linewidth]
{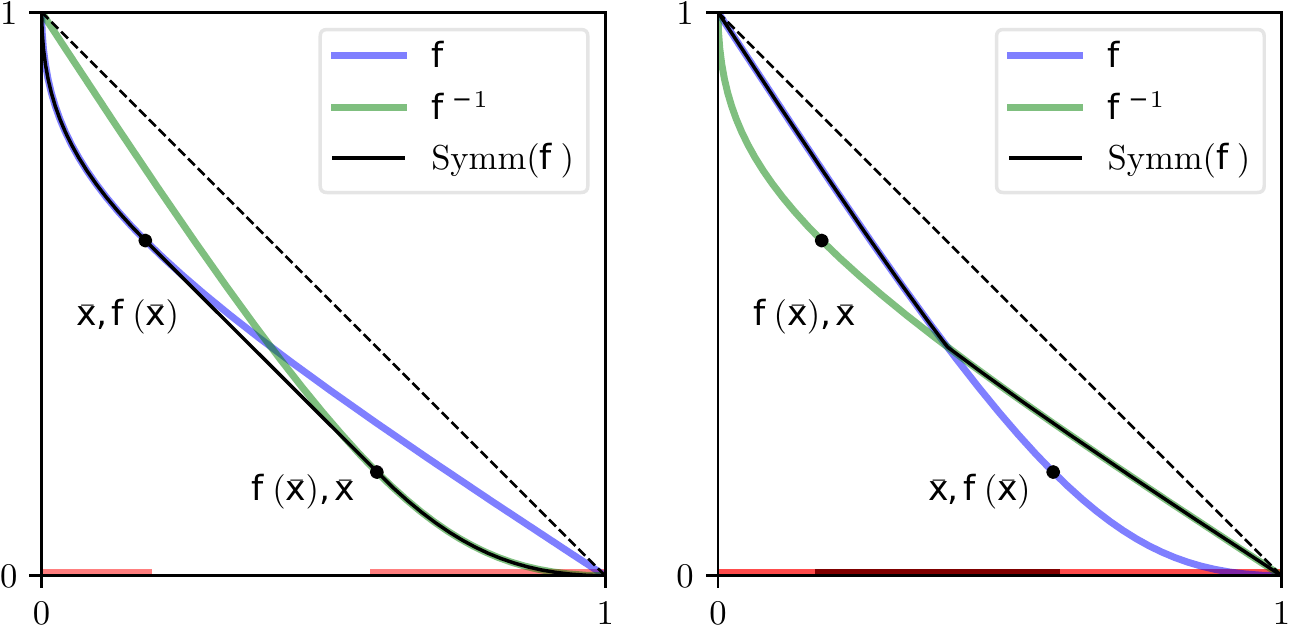}
  \captionof{figure}{Action of Symm. Left panel: $\bar{x}\leqslant f(\bar{x})$. Right panel: $\bar{x} > f(\bar{x})$. For both panels the effective parts (red bars on $x$-axes) are $[0,\bar{x}]$ of $f$ and $[f(\bar{x}),1]$ of $f^{-1}$. No overlap in the left panel since $\bar{x} < f(\bar{x})$, so interpolate with straight line; overlap in the right panel so the max is taken.}
  \label{fig:symm}
\end{figure}

From \Cref{fig:symm} it's not hard to see that
\begin{align*}
	\min\{f,f^{-1}\}^{**}(x) =
	\left\{
	\begin{array}{ll}
	f(x),&x\in[0,\bar{x}], \\
	\bar{x}+f(\bar{x})-x, & x\in[\bar{x},f(\bar{x})], \\
	f^{-1}, & x\in[f(\bar{x}),1].
	\end{array}
	\right.
\end{align*}


\begin{proof}[Proof of \Cref{prop:asymm_env}]
	$M$ being $\big(\ep,\delta(\ep)\big)$-DP means that for any neighboring datasets $S$ and $S'$,
	\[T\big(M(S),M(S')\big)(x)\geqslant-e^\ep x+1-\delta(\ep).\]
	Fix $x\in[0,1]$. Since the DP condition holds for all $\ep\geqslant 0$, the lower bound still holds when we take the supremum over $\ep\geqslant 0$. In other words, $M$ is ${f}_\mathrm{env}$-DP with
	$${f}_\mathrm{env}(x) = \max\{0,\,\,\sup_{\ep\geqslant0}1-\delta(\ep)-e^\ep x\}.$$
	By \Cref{prop:symmetry} $M$ is also $\max\{{f}_\mathrm{env},{f}_\mathrm{env}^{-1}\}$-DP. The proof will be complete if we can show $\max\{{f}_\mathrm{env},{f}_\mathrm{env}^{-1}\} = \Symm(f)$.

	\begin{figure}[h]
	  \centering
	  \includegraphics[width=0.7\linewidth]
	{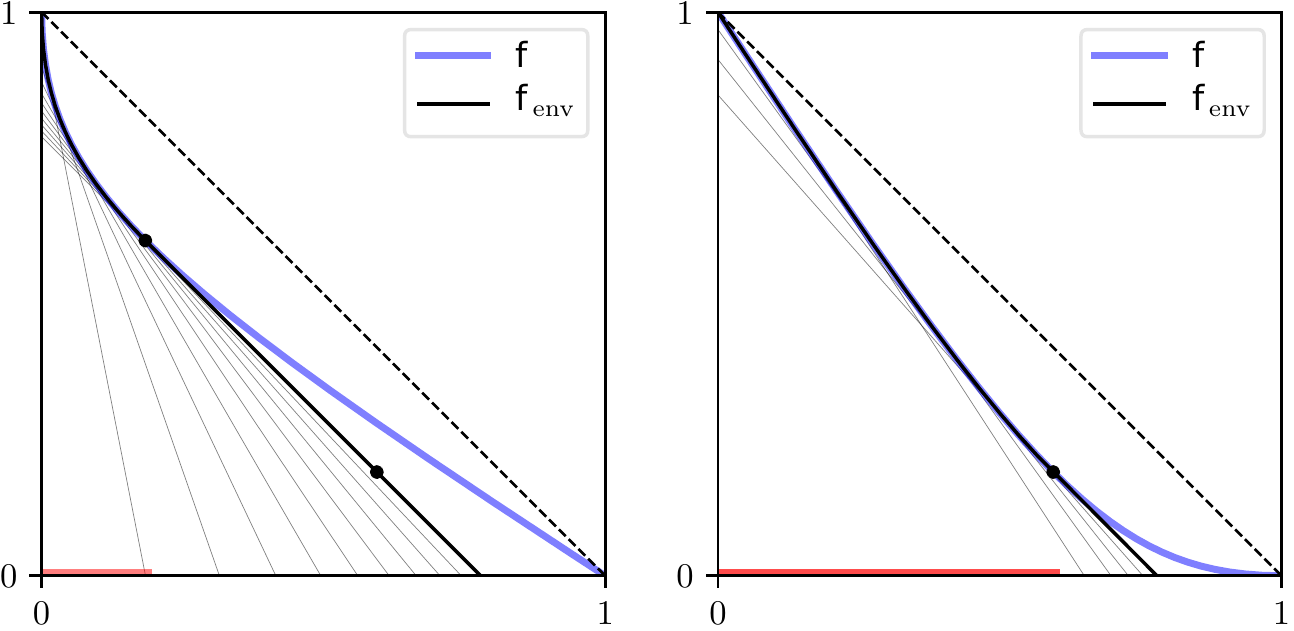}
	  \captionof{figure}{Symm explained. Left panel: $\bar{x}\leqslant f(\bar{x})$. Right panel: $\bar{x} > f(\bar{x})$. }
	  \label{fig:symm2}
	\end{figure}

	We achieve this by first showing:
	\[{f}_\mathrm{env}(x) =
	\left\{
	\begin{array}{ll}
	f(x),&x\in[0,\bar{x}], \\
	\bar{x}+f(\bar{x})-x, & x\in[\bar{x},\bar{x}+f(\bar{x})], \\
	0, & x\in[\bar{x}+f(\bar{x}),1].
	\end{array}
	\right.\]
	From \Cref{fig:symm2} it is almost obvious. We still provide the argument below.

	Plug in $\delta(\ep) = 1+f^*(-e^{\ep})$ and change the variable $y=-e^\ep$:
	\begin{align*}
		\sup_{\ep\geqslant0}[-e^\ep x+1-\delta(\ep)]
		&=\sup_{\ep\geqslant0}[-e^\ep x-f^*(-e^{\ep})]\\
		&=\sup_{y\leqslant-1}[y x-f^*(y)]
	\end{align*}
	From convex analysis we know if $y\in\partial f(x)$ then $yx = f(x)+f^*(y)$. By definition of $\bar{x}$, if $x\leqslant\bar{x}$, then at least one subgradient $y\in\partial f(x)$ is no greater than $-1$. So this specific $y x-f^*(y) = f(x)$ is involved in the supremum, i.e. $\sup_{y\leqslant-1}[y x-f^*(y)] = f(x)$. This justifies the expression for the first segment.

	When $x>\bar{x}$, the supremum is always attained at $y=-1$. In fact, if we let $l_y(x) = y x-f^*(y)$, then
	\begin{lemma} \label{lem:ning}
	 	$l_y(x)\leqslant l_{-1}(x)$ when $y\leqslant -1$ and $x>\bar{x}$.
	\end{lemma}
	\begin{proof}[Proof of \Cref{lem:ning}]
	$l_y$ is the supporting linear function of $f$ with slope $y$. It suffices to show that $l_y(x)$ is monotone increasing in $y$. To see this, change the variable from the slope $y$ to the supporting location $u$. As $f$ is convex, $y=f'(u)$ is increasing in $u$. In terms of $u$, $l_y(x) = f(u) + f'(u)(x-u)$. Taking derivative with respect to $u$:
	\[\frac{\partial}{\partial u }\, l_y(x) = f'(u) + f''(u)(x-u) + f'(u)\cdot(-1) = f''(u)(x-u).\]
	$y\leqslant -1$ corresponds to location $u\leqslant \bar{x}$ and hence $u<x$. So we see that $\frac{\partial}{\partial u }\, l_y(x)= f''(u)(x-u) \geqslant 0$. $l_y(x)$ is increasing in $u$, and hence increasing in $y$, completing the proof of the lemma.
	\end{proof}
	So the supremum is attained at $y=-1$. The value is $\sup_{y\leqslant -1}\big[yx-f^*(y)\big] = l_{-1}(x)$. Support function with slope $-1$ must support $f$ at $\bar{x}$.
	The location yields expression $l_{-1}(x) = f(\bar{x}) - (x-\bar{x})$. This justifies the expression for the second segment. The third one is simply the result of thresholding at 0.

	It's straightforward to verify that
	\[
		f_\mathrm{env}^{-1}(x) =
		\left\{
		\begin{array}{ll}
		\bar{x}+f(\bar{x})-x,&x\in[0,f(\bar{x})], \\
		f^{-1}(x), & x\in[f(\bar{x}),1].
		\end{array}
		\right.
	\]
	Obviously $\max\{{f}_\mathrm{env},{f}_\mathrm{env}^{-1}\}=\Symm(f)$. When the intervals $[0,\bar{x}]$ and $[f(\bar{x}),1]$ are disjoint, $f$ and $f^{-1}$ are effective separately, and the linear interpolation fills the blank. On the other hand when they intersect, max is taken and the linear function is never effective.
\end{proof}

We conclude this appendix with the proof of the classical privacy amplification by subsampling theorem. It primarily follows \cite{jon}, but is written so that potential generalization and improvement are in reach.
\DPsubsamplerep*
\begin{proof}[Proof of \Cref{lem:DPsubsample}]
	Let $\xv$ and $\xv'$ be neighboring datasets, each with $n$ individuals. Without loss of generality, assume $\xv$ and $\xv'$ differ in the first individual. We are ultimately interested in $\tm(\xv)$ and $\tm(\xv')$. They are generated as follows.

	Let $I\subseteq[n]$ be any size $m$ subset of the index set $[n]=\{1,2,\ldots, n\}$. $\xv_I$ and $\xv_I'$ denote the $m$ individuals indexed by $I$ in corresponding datasets, both of which can be the input of $M$. When $I$ is uniformly sampled from the $n\choose m$ subsets of $[n]$ of cardinality $m$, $M(\xv_I)$ is $\tm(\xv)$.


	Let $\phi$ be an arbitrary rejection rule. For a fixed $I\subseteq [n]$, when $\phi$ is used for the problem $M(\xv_I)$ vs $M(\xv'_I)$, the corresponding type I and type II errors are
	\begin{equation}\label{eq:defS}
		\alpha_I:=\E[\phi(M(\xv_I))]\quad\text{ and }\quad\beta_I:=1-\E[\phi(M(\xv'_I))]
	\end{equation}
	respectively. The expectations are over the randomness of $M$.

	When $\phi$ is used for $\tm(\xv)$ vs $\tm(\xv')$, the type I error $\alpha$ and type II error $\beta$ satisfy
	\[
		\alpha=\E[\phi(\tm(\xv))]= \E_I\,\E[\phi(M(\xv_I))] = \E_I[\alpha_I] 
	\]
	and
	\begin{align*}
		\beta=1-\E[\phi(\tm(\xv'))]= \E_I\big[1-\E[\phi(M(\xv'_I))]\big] = \E_I[\beta_I]. 
	\end{align*}
	Ultimately we are going to show that $\beta$ has a lower bound in terms of $\alpha$.
	This is possible because for each $I$, $\beta_I$ has a lower bounded in terms of $\alpha_I$, whose form depends on whether the ``difference'' individual 1 is sampled.

	When $1\not\in I$, $\xv_I = \xv'_I$. By definition \eqref{eq:defS}, $\alpha_I+\beta_I = 1$.
	When $1\in I$, $\xv_{I}$ and $\xv'_{I}$ are neighbors in $X^m$, so $(\ep,\delta)$ privacy of $M$ becomes effective here. Let $k = -e^\ep, b = 1-\delta$. From \Cref{thm:privacy_testing} we know $\beta_I\geqslant k\alpha_I+b$.

	So we should separate these cases $1\in I$ and $1\notin I$ and average them respectively.
	Define
	\begin{align*}
		A&=\E[\alpha_I|I\ni 1], 
		\quad
		B=\E[\beta_I|I\ni 1] 
	\end{align*}
	and
	\begin{align*}
		\bar{A}&=\E[\alpha_I|I\not\ni 1] ,
		\quad
		\bar{B}=\E[\beta_I|I\not\ni 1].
	\end{align*}
	Let $p=m/n$. Then $P[I\ni 1]=p, P[I\not\ni 1]=1-p$. Further averages of these quantities give us $\alpha$ and $\beta$:
	$$\alpha = pA+(1-p)\bar{A}, \quad \beta = pB+(1-p)\bar{B}.$$
	Linearity passes through expectations, so we have $\bar{A}+\bar{B} = 1$ and
	\[B = \E[\beta_I|I\ni 1] \geqslant \E[k\alpha_I+b|I\ni 1] = k\E[\alpha_I|I\ni 1]+b = kA+b.\]
	



	This inequality $B\geqslant kA+b$ comes from the difference between $\xv$ and $\xv'$. The smart observation is, there are a lot more neighboring datasets we can exploit. For example $S_I$ and $S_{\tilde{I}}'$ when $I = \{1,2,\ldots, m\}, \tilde{I} = \{2,3,\ldots, m+1\}$. They share individuals $2,3,\ldots, m$ and differ in the one left. This yields another inequality $B\geqslant k\bar{A}+b$. We fill in the details Jon Ullman omits in his lecture note.
	
	For $I\ni1$, we can replace 1 with any of the $n-m$ indices not in $I$ and obtain $n-m$ different subsets $I^{(1)},\ldots, I^{(n-m)}$. Note that none of these contains 1. In another word, we have the following correspondence:
	\[
		\begin{array}{ccc}
			\ni 1 &  & \not\ni1 \\
			I & \longleftrightarrow & I^{(1)} \\
			& \vdots &\\
			I & \longleftrightarrow & I^{(n-m)}
		\end{array}
	\]
	Each $\xv_{I^{(j)}}$ is a neighbor of $\xv'_I$, so by the privacy of $M$, we have $\beta_I\geqslant k\alpha_{I^{(j)}}+b$. Sum these up for each $I\ni1$ and the $n-m$ replacements of $I$, we have
	\[(n-m)\cdot\sum_{I\ni 1} \beta_I\geqslant m\cdot\sum_{I\not\ni 1} k\alpha_I+b.\]
	The right hand side factor $m$ comes from a different counting: if $I\not\ni1$ then each of the $m$ items in $I$ could have been 1 before the replacement. So each $I\not\ni1$ appears $m$ times in the summation.

	Multiply by $\frac{n}{m(n-m)}\cdot\binom{n}{m}^{-1}$,
	\[\frac{n}{m}\cdot\bigg(\sum_{I\ni 1} \beta_I\bigg)\cdot \binom{n}{m}^{-1}\geqslant b+ k \cdot \frac{n}{n-m}\cdot\bigg(\sum_{I\not\ni 1} \alpha_I\bigg)\cdot \binom{n}{m}^{-1}.\]
	By the simple Bayes' rule, this is $B\geqslant k\bar{A}+b$.

	Remember we want a lower bound of $\beta$. The best possible lower bound we can ge t via these relations is the minimum of the following linear program:
	\begin{align*}
		\min_{A,B,\bar{A},\bar{B}} \quad & \beta\\
		\textrm{s.t.}\quad
		& pB+(1-p)\bar{B}=\beta\\
		& pA+(1-p)\bar{A}=\alpha\\
		& \bar{A}+\bar{B} = 1 \\
		& B\geqslant kA+b\\
		& B\geqslant k\bar{A}+b \\
		& A,B,\bar{A},\bar{B}\in [0,1]
	\end{align*}
	\begin{lemma} \label{lem:LP}
		The minimum of the above linear program is no less than $p(k\alpha+b)+(1-p)(1-\alpha)$.
	\end{lemma}
	\begin{proof}[Proof of \Cref{lem:LP}]
		We are going to remove the $A,B,\bar{A},\bar{B}\in [0,1]$ constraint and find the exact minimum of the relaxation, which is a lower bound of the original minimum.

		We have $\beta+\alpha = p(A+B)+(1-p)(\bar{A}+\bar{B})=p(A+B)+(1-p)$. So equivalently we can try to solve
		\begin{align*}
			\min_{A,B,\bar{A}} \quad & A+B\\
			\textrm{s.t.}\quad
			& pA+(1-p)\bar{A}=\alpha\\
			& B\geqslant kA+b\\
			& B\geqslant k\bar{A}+b
		\end{align*}
		When $A=\bar{A}=\alpha$, the lower bound they impose on $B$ is $k\alpha+b$. Notice that $k=-\e^\ep\leqslant -1$. The two consequences are:
		\begin{enumerate}
			\setlength\itemsep{-0.3em}
			\item $A>\alpha$ is worse than $A=\alpha$, because the convex combination equality requires $\bar{A}<\alpha$. $B$ has to increase anyway. Both $A$ and $B$ increase, so the objective $A+B$ also increases.
			\item If $A$ is decreased from $\alpha$ by some amount, $B$ has to increase by $\e^\ep$ times that amount, which is not worth it.
		\end{enumerate}
		So the minimum of the relaxed linear program is achieved at $A=\bar{A}=\alpha, B=k\alpha+b,\bar{B} = 1-\alpha$, thereby inducing the claimed lower bound.
	\end{proof}

	So $\beta\geqslant p(k\alpha+b)+(1-p)(1-\alpha)$. Changing back to $\ep,\delta$, we have
	\[\beta\geqslant p(-e^\ep\alpha+1-\delta)+(1-p)(1-\alpha) = -[p\e^\ep+1-p]\alpha+1-p\delta = -\e^{\ep'}\alpha+1-\delta'.\]
	By \Cref{thm:privacy_testing}, $\tm$ is $(\ep',\delta')$-DP.
\end{proof}
\section{Omitted Proofs in \Cref{sec:application_in_sgd}}
\label{app:SGD}
Our first goal is to prove
\mixturerep*
First we point out that the functional $\chi^2_+$ is computing a variant of $\chi^2$-divergence. Recall that $\chi^2$-divergence is an $F$-divergence with $F(t) = (t-1)^2$. We define $\chi^2_+$-divergence to be the $F$-divergence with $F(t) = (t-1)^2_+ =
	\left\{
	\begin{array}{ll}
		0, &t\leqslant1,\\
		(t-1)^2, & t>1.
	\end{array}
	\right.$
As in \Cref{app:relation}, let $z_f = \inf\{x\in[0,1]:f(x)=0\}$ be the first zero of $f$. 
\begin{proposition} \label{prop:chiplus}
	For a pair of distributions $P$ and $Q$ such that $T(P,Q)=f$ is a symmetric trade-off function with $f(0)=1$,
	\[\chi^2_+(P\|Q) = \chi^2_+(f).\]
\end{proposition}
\begin{proof}[Proof of \Cref{prop:chiplus}]
	By \Cref{prop:fdiv}, when $f=T(P,Q)$, $\chi^2_+(P\|Q)$ can be computed via the following expression:
$$\chi^2_+(P\|Q)=\int_0^{z_f} \big({\big|f'(x)\big|}^{-1}-1\big)_+^2\cdot \big|f'(x)\big| \diff x + F(0)\cdot(1-f(0))+\tau_F\cdot (1-z_f)$$
where $F(0) = \lim_{p\to 0^+} {F(t)} = 0, \tau_F=\lim_{p\to+\infty} \frac{F(t)}{t} = +\infty$. Since we assume $f$ is symmetric, $z_f= f(0)=1$. This also implies that $f^{-1}$ is the ordinary function inverse, i.e. $f(f(x))=x$. 
Let $y = f^{-1}(x) = f(x)$. Then $\diff y = f'(x) \diff x$. On the other hand, $x=f(y),\diff x = f'(y)\diff y$. $x=1$ corresponds to $y=0$ and $x=0$ corresponds to $y=1$, so
\begin{align*}
	\chi^2_+(P\|Q)
	&=\int_0^{1} \big({\big|f'(x)\big|}^{-1}-1\big)_+^2\cdot \big|f'(x)\big| \diff x\\
	&=\int_0^{1} \Big({\Big|\frac{\diff y}{\diff x}\Big|}^{-1}-1\Big)_+^2\cdot \big|f'(x)\big| \diff x\\
	&=-\int_1^{0} \big({\big|f'(y)\big|}-1\big)_+^2\diff y\\
	&=\int_0^{1} \big({\big|f'(y)\big|}-1\big)_+^2\diff y\\
	&=\chi^2_+(f).
\end{align*}
\end{proof}
We need some more calculation tools to prove \Cref{thm:mixtureSGD}.
\begin{lemma}\label{lem:functionals2}
	Let $f\in\T^S$ with $f(0)=1$ and $x^*$ be its unique fixed point. Then
	\begin{align*}
		\chi^2_+(f)&= \int_0^{x^*}(f'(x)+1)^2\diff x\\
		\kl(f) &= \int_0^{x^*}\big(|f'(x)|-1\big)\log |f'(x)|\diff x \\
		\kappa_2(f)&= \int_0^{x^*}\big(|f'(x)|+1\big)\big(\log |f'(x)|\big)^2\diff x \\
		\bar{\kappa}_3(f)&=\int_0^{x^*}\big|\log |f'(x)|+\kl(f)\big|^3+|f'(x)|\cdot\big|\log |f'(x)|-\kl(f)\big|^3\diff x\\
		\kappa_3(f)&=\int_0^{x^*}\big(|f'(x)|+1\big)\big(\log |f'(x)|\big)^3\diff x.
	\end{align*}
\end{lemma}

\begin{proof}[Proof of \Cref{lem:functionals2}]
	First we observe that $f'(x)\leqslant -1$ for $x\leqslant x^*$ and $f'(x)\geqslant -1$ for $x\geqslant x^*$. This means the integrand involved in $\chi^2_+$ is 0 in $[x^*,1]$ and hence proves the first identity.

	The rest of the proof is entirely based on a trick we used above.
	Let $y = f^{-1}(x) = f(x)$. Then $x=f(y),\diff x = f'(y)\diff y$. Since $x^*$ is the fixed point of $f$, $x=x^*$ corresponds to $y=x^*$. $x=1$ corresponds to $y=0$ and $x=0$ corresponds to $y=1$.
	\begin{align*}
		-\int_{x^*}^1\log |f'(x)|\diff x
		&=\int_{x^*}^1\log |f'(x)|^{-1}\diff x\\
		&=\int_{x^*}^0\log \Big|\frac{\diff x}{\diff y}\Big| \cdot f'(y)\diff y\\
		&=\int^{x^*}_0\log |f'(y)| \cdot |f'(y)|\diff y.
	\end{align*}
	So
	\begin{align*}
		\kl(f) &= -\int_0^{1}\log |f'(x)|\diff x \\
		&=-\int^{x^*}_0\log |f'(x)|\diff x-\int_{x^*}^1\log |f'(x)|\diff x\\
		&=-\int^{x^*}_0\log |f'(x)|\diff x+\int^{x^*}_0\log |f'(x)| \cdot |f'(x)|\diff x\\
		&= \int_0^{x^*}\big(|f'(x)|-1\big)\log |f'(x)|\diff x.
	\end{align*}
	The rest of identities can be proved in exactly the same way.
\end{proof}

\begin{lemma} \label{lem:CpMoments}
	Suppose $f\in\T^S$ and $f(0)=1$. $x^*$ is its unique fixed point. Let $g(x) = -f'(x)-1 = |f'(x)|-1$. Then
	\begin{align*}
		\kl(C_p(f)) &= p\int_0^{x^*}g(x)\log \big(1+pg(x)\big)\diff x \\
		\kappa_2(C_p(f))&= \int_0^{x^*}\big(2+pg(x)\big)\big[\log \big(1+pg(x)\big)\big]^2\diff x \\
		\kappa_3(C_p(f))&=\int_0^{x^*}\big(2+pg(x)\big)\big[\log \big(1+pg(x)\big)\big]^3\diff x.
	\end{align*}
\end{lemma}
\begin{proof}[Proof of \Cref{lem:CpMoments}]
	We prove for kl and the rest are similar. Let $x_p^*$ be the fixed point of $C_p(f)$. By \Cref{lem:functionals2},
	\[
		\kl(C_p(f)) = \int_0^{x^*_p}\big(|C_p(f)'(x)|-1\big)\log |C_p(f)'(x)|\diff x.
	\]
	From the expression of $C_p(f)$ \eqref{eq:Cp_expression} we know $\log |C_p(f)'(x)|=0$ in the interval $[x^*,x^*_p]$, and $C_p(f) = f_p = pf+(1-p)\Id$ in the interval $[0,x^*]$. So
	\begin{align*}
		\kl(C_p(f)) 
		&= \int_0^{x^*}\big(|f_p'(x)|-1\big)\log |f_p'(x)|\diff x.
	\end{align*}
	In the interval $[0,x^*]$, $g(x)=|f'(x)|-1\geqslant 0$. $f_p'(x) = pf'(x)+(1-p)(-1) = p(f'(x)+1)-1 = -pg(x)-1$, so $|f_p'(x)| = pg(x)+1$. When plugged in to the expression above, we have
	\[
	\kl(C_p(f)) = p\int_0^{x^*}g(x)\log \big(1+pg(x)\big)\diff x.
	\]
\end{proof}
\begin{proof}[Proof of \Cref{thm:mixtureSGD}]
	It suffices to compute the limits in \Cref{thm:CLT}, namely
	$$T\cdot \kl(C_p(f)), \,\,T\cdot \kappa_2(C_p(f)) \text{ and } T\cdot \kappa_3(C_p(f)).$$
	Since $T\sim p^{-2}$, we can consider $p^{-2}\kl(C_p(f))$ and so on.

	As in \Cref{lem:CpMoments}, let $x^*$ be the unique fixed point of $f$ and $g(x) = -f'(x)-1 = |f'(x)|-1$. Note that $g(x)\geqslant 0$ for $x\in[0,x^*]$. The assumption expressed in terms of $g$ is simply
	$$\int_0^1 g(x)^4\diff x<+\infty.$$
	In particular, it implies $g(x)^k$ are integrable in $[0,x^*]$ for $k=2,3,4$. In addition, $\chi^2_+(f) = \int_0^{x^*}g(x)^2\diff x$ by \Cref{lem:functionals2}.

	For the functional $\kl$, by \Cref{lem:CpMoments},
	\begin{align}
		\lim_{p\to0^+}\frac{1}{p^2}\,\kl(C_p(f))
		&= \lim_{p\to0^+}\int_0^{x^*}g(x)\cdot \frac{1}{p}\log \big(1+pg(x)\big)\diff x\tag{$*$}\label{eq:sgd1}\\
		&= \int_0^{x^*}g(x)\cdot \lim_{p\to0^+}\frac{1}{p}\log \big(1+pg(x)\big)\diff x\nonumber\\
		&=\int_0^{x^*}g(x)^2\diff x=\chi^2_+(f)\nonumber
	\end{align}
	Changing the order of the limit and the integral in \eqref{eq:sgd1} is approved by \DCT. To see this, notice that $\log(1+x)\leqslant x$.
	The integrand in \eqref{eq:sgd1} satisfies
	\begin{align*}
		0\leqslant g(x)\cdot \frac{1}{p}\log \big(1+pg(x)\big)\leqslant g(x)^2.
	\end{align*}
	We already argued that $g(x)^2$ is integrable, so it works as a dominating function and the limit is justified. When $p\sqrt{T}\to p_0$, we have
	\[T\cdot \kl(C_p(f))\to p_0^2\cdot\chi^2_+(f).\]
	So the constant $K$ in \Cref{thm:CLT} is $p_0^2\cdot\chi^2_+(f)$.

	For the functional $\kappa_2$ we have
	\begin{align*}
		\frac{1}{p^2}\kappa_2(C_p(f)) &= \int_0^{x^*}\big(2+pg(x)\big)\Big[\frac{1}{p}\log \big(1+pg(x)\big)\Big]^2\diff x.
	\end{align*}
	By a similar dominating function argument,
	\begin{align*}
		\lim_{p\to0^+}\frac{1}{p^2}\,\kappa_2(C_p(f)) = 2\int_0^{x^*}g(x)^2\diff x=2\chi^2_+(f).
	\end{align*}
	Adding in the limit $p\sqrt{T}\to p_0$, we know $s^2$ in \Cref{thm:CLT} is $2p_0^2\cdot\chi^2_+(f)$. Once again, we have $s^2 = 2K$.

	The same argument involving $g(x)^4$ applies to the functional $\kappa_3$ and yields
	$$\lim_{p\to0^+}\frac{1}{p^3}\,\kappa_3(C_p(f)) = 2\int_0^{x^*}g(x)^3\diff x.$$
	Note the different power in $p$ in the denominator. It means $\kappa_3(C_p(f)) = o(p^2)$ and hence $T\cdot \kappa_3(C_p(f))\to 0$ when $p\sqrt{T}\to p_0$.

	Hence all the limits in \Cref{thm:CLT} check and we have a $G_\mu$ limit where
	\[\mu = 2K/s = s = \sqrt{2p_0^2\cdot\chi^2_+(f)} = p_0\cdot\sqrt{2\chi^2_+(f)}.\]
	This completes the proof.
\end{proof}
\chirep*
\begin{proof}[Proof of \Cref{lem:chi2GDP}]
	We use \Cref{prop:chiplus} as the tool. Obviously $P=\N(\mu,1)$ and $Q=\N(0,1)$ satisfy the conditions there. So it suffices to compute $\chi^2_+(\N(\mu,1)\|\N(0,1))$. Recall that $\chi^2_+$ is the $F$-divergence with $F(t) = (t-1)_+^2$, so $\chi^2_+(P\|Q) = \E_Q\big[( \frac{P}{Q}-1)_+^2\big]$. Let $\varphi$ and $\Phi$ be the density function and cdf of the standard normal. We have
	\begin{align*}
		\chi^2_+(G_\mu) &= \chi^2_+(\N(\mu,1)\|\N(0,1))\\
		&= \E_{x\sim \N(0,1)}\Big[\Big(\frac{\varphi(x-\mu)}{\varphi(x)}-1\Big)_+^2\Big]\\
		&= \int_{\mu/2}^{+\infty}\Big(\frac{\varphi(x-\mu)}{\varphi(x)}-1\Big)^2\cdot \varphi(x)\diff x\\
		&= \int_{\mu/2}^{+\infty}\Big(\frac{\varphi(x-\mu)}{\varphi(x)}\Big)^2\cdot \varphi(x)\diff x -2 \int_{\mu/2}^{+\infty}\varphi(x-\mu)\diff x+\int_{\mu/2}^{+\infty}\varphi(x)\diff x\\
		&= \underbrace{\int_{\mu/2}^{+\infty}e^{2\mu x-\mu^2}\cdot \varphi(x)\diff x}_{I}-2(1-\Phi(\mu/2))+\Phi(-\mu/2)\\
		&=I+3\Phi(-\mu/2)-2.
	\end{align*}
	For the integral $I$,
	\begin{align*}
		I &= \int_{\mu/2}^{+\infty}e^{2\mu x-\mu^2}\cdot \varphi(x)\diff x\\
		&= \int_{\mu/2}^{+\infty}\frac{1}{\sqrt{2\pi}}\cdot e^{2\mu x-\mu^2-x^2/2}\diff x\\
		&= \int_{\mu/2}^{+\infty}\frac{1}{\sqrt{2\pi}}\cdot e^{-(x-2\mu)^2/2}\cdot e^{\mu^2}\diff x\\
		&= e^{\mu^2}\cdot P[\N(2\mu,1)\geqslant \mu/2]\\
		&= e^{\mu^2}\cdot\Phi(3\mu/2)
	\end{align*}
	This completes the proof.
\end{proof}
\SGDlimitrep*
\begin{proof}[Proof of \Cref{thm:SGDlimit}]
	Combining \Cref{thm:sgdcompo,thm:mixtureSGD} and \Cref{lem:chi2GDP}, it suffices to check $\int_0^1(f'(x)+1)^4\diff x<+\infty$ when $f(x) = G_a(x) = \Phi(\Phi^{-1}(1-x)-a)$. Let $y = \Phi^{-1}(1-x)$. We have $\varphi(y)\diff y = -\diff x$. Hence
	\[G_a'(x) = \varphi(y-a) \cdot \frac{\diff y}{\diff x} = -\frac{\varphi(y-a)}{\varphi(y)} = -\e^{ay-\frac{a^2}{2}}.\]
	The integral is
	\begin{align*}
		\int_0^1(G_a'(x)+1)^4\diff x
		&=\int_{-\infty}^{+\infty}(-\e^{ay-\frac{a^2}{2}}+1)^4\varphi(y)\diff y,
	\end{align*}
	which is just a linear combination of moment generating functions of the standard normal and hence finite.
\end{proof}
\functionalGmurep*
\begin{proof}[Proof of \Cref{lem:functionalGmu}]
	We will use \Cref{lem:CpMoments}. It's easy to show the fixed point of $G_\mu$ is $x^* = \Phi(-\mu/2)$. So
	\begin{align*}
		\kl\big(C_p(G_\mu)\big) &=
		p\int_0^{\Phi(-\mu/2)}\big(-G_\mu'(x)-1\big)\log \big(1+p(-G_\mu'(x)-1)\big)\diff x
	\end{align*}
	Using the same change of variable $y = \Phi^{-1}(1-x) = -\Phi^{-1}(x)$, we have
	\begin{align*}
		\kl\big(C_p(G_\mu)\big) &=
		p\int^{+\infty}_{\mu/2}\Big(\frac{\varphi(y-\mu)}{\varphi(y)}-1\Big)\log \Big(1+p\Big(\frac{\varphi(y-\mu)}{\varphi(y)}-1\Big)\Big)\varphi(y)\diff y\\
		&=p\int_{\mu/2}^{+\infty} Z(y)\cdot\big(\varphi(y-\mu)-\varphi(y)\big)\diff y.
	\end{align*}
	The rest can be proved similarly.
\end{proof}
\SGDBerryrep*
\begin{proof}[Proof of \Cref{thm:SGDBerry}]
	Follows from plugging in the expressions above into \Cref{thm:Berry}.
\end{proof}

\end{document}